\documentclass{sage}

\ifx\definitionMacros\undefined\else \fi
 \ifx\optionkeymacros\undefined\else \fi
 \def\registered{{\ooalign                                          
   {\hfil\kern+.05em\raise.12ex                                     
 \hbox{\tiny R}\hfil\crcr{\footnotesize\mathhexbox20D}}}}           
 \newcommand{\TrueMath}    [1]{\mbox{$#1$}}                  
 \def\half{\TrueMath{\leavevmode\kern.1em \raise.5ex
                     \hbox{\the\scriptfont0 1}
                     \kern-.1em / \kern-.15em\lower.25ex
                     \hbox{\the\scriptfont0 2}
           }         }
\def\set12{\newfont{\size12}{cmbx12}}

\renewcommand{\frac}[2]{\TrueMath{\TrueMath{#1}\over\TrueMath{#2}}}

\def\simless
{\mathop{\raise0.50ex\hbox{$<$}\kern-0.75em\lower0.50ex\hbox{$\sim$}}\nolimits}
\def\simgreat
{\mathop{\raise0.50ex\hbox{$>$}\kern-0.75em\lower0.50ex\hbox{$\sim$}}\nolimits}
\usepackage{mathtools}
\usepackage{amsmath}
\usepackage{amsthm}
\usepackage{amsfonts}
\usepackage{amsbsy}
\usepackage{amssymb}
\usepackage{latexsym}

\usepackage{wrapfig}
\usepackage{natbib,url}

\usepackage{multirow}
\usepackage{subfigure}
\usepackage{xfrac}
\usepackage{mathrsfs}
\usepackage{cases}
\usepackage{empheq}

\usepackage{algorithmic}
\usepackage{algorithm}
\usepackage{subeqnarray}

\usepackage{hyperref}
\usepackage{soul}
\usepackage{xcolor}

\usepackage{tikz}
\usepackage{blindtext}

\usepackage{comment}
\usepackage{minitoc}

\usepackage[font=scriptsize]{caption}

\usepackage{color}

\newlength\myindent
\newtheorem{theorem}{Theorem}
\newtheorem*{theorem*}{Theorem}
\newtheorem*{proposition*}{Proposition}
\newtheorem{proposition}{Proposition}

\newtheorem{definition}{Definition}
\newtheorem{remark}{Remark}
\newtheorem{assumption}{Assumption}

\newtheorem*{example*}{Example}
\newtheorem*{property*}{Property}

{\color{blue}}

 \usepackage{tikz}

 \usepackage{marvosym}
 \usepackage{amssymb}
 \usepackage{pifont}
 \usepackage{xspace}

 \def\*tr{^{\,*T}}

 \DeclareSymbolFont{rsfs}{U}{rsfs}{m}{n}
 \DeclareSymbolFontAlphabet{\mathrsfs}{rsfs}

 \newcounter{HCRLalgorno}

\makeatletter
 \newcommand{\customlabel}[2]{%
   \protected@write \@auxout {}{\string \newlabel {#1}{{#2}{}}}}
\makeatother

\journal{The International Journal of Robotics Research}
\volume{XX}
\issue{XX}
\copyrightline{$\copyright$ Copyright}
\firstpage{1}
\lastpage{39}
\doi{doi number}
\articletype{Article Type}
\pubyear{2019}

\def\alert#1{{\color{blue}{#1}}}

\usepackage{color}

\graphicspath{{./../../}}

\usepackage{times}
\usepackage{multicol}

\usepackage{soul}
\usepackage{amsfonts,amsthm}
\usepackage{graphicx}
\usepackage{xfrac}
\usepackage{mathtools}
\usepackage{mathrsfs}
\usepackage{empheq}

\begin{document}


\title{Reactive Task and Motion Planning for Robust Whole-Body Dynamic Locomotion in Constrained Environments}

\author{Ye Zhao
	    \thanks{Corresponding author;
                 E-mail addresses: ye.zhao@me.gatech.edu, \{yinan.li, j.liu\}@uwaterloo.ca, \{utopcu, lsentis\}@utexas.edu.} $^{1}$,
	    Yinan Li$^{2}$, Luis~Sentis$^{3}$, Ufuk Topcu$^{3,4}$, Jun Liu$^{2}$.
}

\address{$^1$George W. Woodruff School of Mechanical Engineering, Georgia Tech, USA.\\
$^2$ Department of Applied Mathematics, University of Waterloo, Canada.\\
$^3$Department of Aerospace Engineering and Engineering Mechanics, UT Austin, USA.\\
$^4$Institute for Computational Engineering and Sciences, UT Austin, USA.}

\maketitle

\vspace{-0.1in}

\begin{abstract}
Contact-based decision and planning methods are becoming increasingly important to endow higher levels of autonomy for legged robots. Formal synthesis methods derived from symbolic systems have great potential for reasoning about high-level locomotion decisions and achieving complex maneuvering behaviors with correctness guarantees. This study takes a first step toward formally devising an architecture composed of task planning and control of whole-body dynamic locomotion behaviors in constrained and dynamically changing environments. At the high level, we formulate a two-player temporal logic game between the multi-limb locomotion planner and its dynamic environment to synthesize a winning strategy that delivers symbolic locomotion actions. These locomotion actions satisfy the desired high-level task specifications expressed in a fragment of temporal logic. Those actions are sent to a robust finite transition system that synthesizes a locomotion controller that fulfills state reachability constraints. This controller is further executed via a low-level motion planner that generates feasible locomotion trajectories. We construct a set of dynamic locomotion models for legged robots to serve as a template library for handling diverse environmental events. We devise a replanning strategy that takes into consideration sudden environmental changes or large state disturbances to increase the robustness of the resulting locomotion behaviors. We formally prove the correctness of the layered locomotion framework guaranteeing a robust implementation by the motion planning layer. Simulations of reactive locomotion behaviors in diverse environments indicate that our framework has the potential to serve as a theoretical foundation for intelligent locomotion behaviors.
\end{abstract}


 \section{Introduction}
The goal of this paper is to devise a reactive task and motion planning framework for whole-body dynamic locomotion (WBDL) behaviors in constrained environments. We employ formal methods for synthesis of a symbolic task planner and design of reachability controllers to achieve legged locomotion behaviors that are reactive to the environment. Although widely used in mobile robot motion planning [\cite{wongpiromsarn2012receding, kloetzer2010automatic, Fu2016TAC}] and autonomous driving [\cite{campbell2010autonomous, xu2018correctness}], formal methods have not been previously used to reason about keyframe states of dynamic locomotion behaviors. To that end, we rely on dynamic locomotion abstractions that reduce the dimensionality of the reasoning process [\cite{zhao2017robust}]. These abstractions allow to sequentially compose locomotion modes by reasoning about the previously mentioned keyframe dynamic locomotion states and achieve advanced reactive behaviors that can respond to dynamic events in the environment as well as to disturbances, a hallmark of intelligent locomotion behaviors. The complex locomotion behaviors studied in this paper could not be achieved by using motion planners alone without a high-level decision-making process. Reasoning about keyframe dynamic locomotion states has several advantages allowing to 1) take advantage of the passive dynamics of legged robots, 2) directly compose behaviors in the phase space of the locomotion process, 3) achieve goal state reachability considering robustness margins, and 4) adjust locomotion behaviors in response to disturbances. 


Our technical approach relies on a suite of template-based locomotion modes that span a spectrum of desired whole-body dynamic locomotion behaviors. Sequentially composing these modes via the proposed reactive synthesis enables us to formally combine tasks such as multi-contact locomotion, swinging movements, and hopping motions, as shown in Fig.~\ref{fig:unifiedBehaviors}. Using simplified models to characterize locomotion dynamics has been widely pursued such as the use of the linear inverted pendulum model (LIPM) [\cite{kajita20013d}], the spring-loaded inverted pendulum model [\cite{piovan2015reachability}], the brachiation-like pendulum model [\cite{bertram1999point}], the multi-contact model [\cite{sentis2010compliant, caron2016multi}], and our recently proposed prismatic inverted pendulum model (PIPM) [\cite{zhao2012three}], to name a few. Usually, these models are separately considered in their own specific scenarios and lack a framework to seamlessly integrate them. Seminal locomotion results using template models [\cite{raibert1986legged, full1999templates, alexander1984gaits, de2015penn}] and sequential composition of these models \cite{burridge1999sequential} championed the advantages of using simplified models to uncover the fundamental locomotion principles related to the fine details of multi-body mechanism and dynamics. The work in [\cite{arslan2012reactive}] employs sequential composition to achieve reactive and robust planning against both model uncertainty and measurement noise without replanning.
Nevertheless, no high-level decision-making algorithms with formal guarantees have been investigated, although the mentality of hierarchical planning and control had been proposed in [\cite{full1999templates}]. 

In this study, we aim at bridging this gap by proposing formal symbolic-level decision-making theories to sequentially compose more challenging --  highly dynamic, versatile, non-periodic -- locomotion behaviors reactive to dynamic environmental events. In the vein of work addressing rough terrain locomotion [\cite{englsberger2015three, sreenath2013partially, zhao2016robust}], we address the variability of the terrains by allowing the robot to respond to sudden environmental events. The behaviors we synthesize are required to satisfy formal task specifications in a provably correct manner, which we guarantee by using formal methods with discrete abstractions of hybrid systems [\cite{alur2000discrete}]. To the best of our knowledge, our study is the first attempt to use formal methods applied to phase-space keyframe state during for dynamic locomotion behaviors.
%


The inherent hybrid dynamics of the locomotion process and our use of keyframe dynamic locomotion states facilitate the discrete planning synthesis. Instead of discretizing the robot's state space, we rely on a discretization of the phase space keyframe states for synthesizing symbolic-level decisions which are further sent to the underlying motion planner.
We focus on the integration between the symbolic-level discrete task planner and the continuous motion planner. This top-down planning approach significantly reduces the computational complexity compared to bottom-up approaches [\cite{liu2013synthesis,tabuada2009verification,belta2017formal,liu2016finite}].
The correctness of our top-down hierarchy is guaranteed via a correct-by-construction synthesis at the task planner level and a reachability control synthesis at the motion planner level.

The contributions of this paper are as follows. The first one is on devising symbolic reasoning methods that make decisions on keyframe states of the dynamic locomotion process in response to the dynamically changing environment. Our second contribution is on ensuring robust locomotion under bounded disturbances by reasoning about keyframe state reachability. The third contribution is on using game theory to compose complex dynamic locomotion behaviors sequentially. The final contribution is on reasoning about the correctness of the overall planning framework.

\begin{figure}[t]
 \centering
   \includegraphics[width=0.95\linewidth]{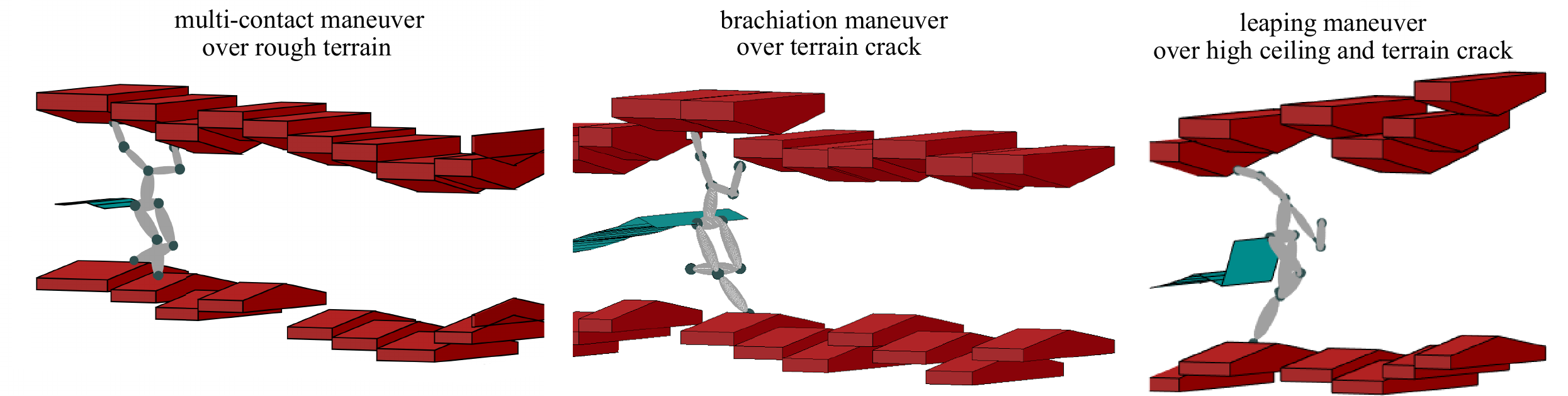}
 \caption{\captionsize Maneuvering in a constrained environment via multi-contact whole-body dynamic locomotion. Three locomotion modes are illustrated. The contact decisions are made according to the high-level symbolic task planner.}
\label{fig:unifiedBehaviors}
\end{figure}

This paper is organized as follows. Section~\ref{sec:problem-formulation} introduces various dynamic locomotion models and the problem formulation of switched systems, phase space planning, and temporal logic preliminaries. In
Section~\ref{sec:WBDLSpec}, we present the task specifications for whole-body dynamic locomotion and a reactive planner winning strategy via defining a two-player logic game. Section~\ref{sec:Abstraction} introduces a robust finite transition system for the hybrid locomotion process to reason about local robustness with respect to the bounded disturbances and proposes robustness margin sets using phase-space Riemmanian metrics. In Section~\ref{sec:correctness-reasoning}, we reason about the one-walking-step robust reachability and the correctness of the overall planning strategy. Simulation results of whole-body dynamic locomotion behaviors over changing environments
are shown in Section~\ref{section:results}. In Sections~\ref{section:discussions} and~\ref{sec:conclusion}, we discuss the results and the conclusions of this paper. The Appendix
presents supplementary mathematical formulations, algorithms, and propositions. A preliminary version of this paper was published in a conference proceeding [\cite{zhao2016high}]. Compared to that proceeding, this paper presents a new study on robust reachability control synthesis, incorporates additional locomotion modes and more diverse task specifications, proposes a replanning strategy, and implements more sophisticated behaviors with a diversity of environmental events.

 \section{Related Work}

Formal methods have been widely investigated for mobile navigation [\cite{kress2011correct, raman2015reactive, Jonathan15, Sadigh-RSS-16}]. The authors in [\cite{kloetzer2010automatic}] proposed an automated computational framework for decentralized communications and control of a team of mobile robots from global task specifications. This work suffers from high computational complexity and does not address reactive response to environmental changes. To alleviate the computational burden, the work in [\cite{wongpiromsarn2012receding}] proposed a receding-horizon based hierarchical framework that reduced the complex synthesis problem to a set of significantly smaller problems with a shorter horizon. An autonomous vehicle navigation process is simulated in the presence of exogenous disturbances. Provable correctness is an important property of temporal logic based control and planning approaches.
 The work of [\cite{kress2009temporal}] allows mobile robots to react to the environment in real time and guarantees the provable correctness of controllers. The approach proposed in [\cite{liu2013synthesis}] extended controller synthesis with guaranteed-correctness to nonlinear switched systems and designed a reactive mechanism in response to an adversarial environment at runtime. Given a high-level discrete controller encoding reactive task behaviors, the work in [\cite{decastro2015synthesis}] designed low-level controllers to guarantee the correctness of a high-level controller. More recently, the work of [\cite{duperret2020towards}] solves a formal discrete leaping navigation problem of legged robots to reach a goal set while in the interim reactively avoiding a set of obstacle states.
However, all of the work above is applied to 2D-world mobile robots or a single-leg hopper, which have simple dynamics unlike our focus on underactuated and hybrid legged robots. Although the recent works in [\cite{warnke2020towards, kulgod2020temporal, gu2021reactive}] explore the use of temporal-logic-based formal methods to solve bipedal robot navigation problems, the studied environments are well structured such as level ground or mild rough terrain with stairs. In addition, contact sequence planning for bipedal robots is straightforward due to the unique option of alternating two legs for contact. Formal-method-based planning for multi-limb robots, such as the one in this paper, requires to consider highly confined environment constraints and contact sequence planning which do not present in grounded mobile or bipedal walking robots.

\subsection{Formal methods for manipulation and locomotion}
Formal methods have also gained increasing attention in the mobile manipulation community via
task and motion planning (TAMP) methods [\cite{kaelbling2011hierarchical, srivastava2014combined,  dantam2018incremental, he2015towards, zhao2021sydebo}] or reactive synthesis methods [\cite{sharan2014formal, chinchali2012towards, he2017reactive}]. However, many existing TAMP approaches rely on sampling-based motion planners which ignore the underlying physical dynamics. To fill this gap, the recent work of [\cite{toussaintdifferentiable2018}] proposed a logic-geometric program
to incorporate manipulation dynamics into the task and motion planning process, where discrete logic rules are used to specify the mode sequence for dynamic manipulation tasks. However, this work lacked a reactive mechanism in response to environment actions and manipulated objects. More importantly, formal methods are yet to be used to reason about dynamic legged locomotion, or for more complex dynamic tasks for humanoid robots like the ones described in this paper. The authors in [\cite{antoniotti1995discrete}] determined goals for legged robots by using computational tree logic and synthesized controllers for locomotion. However, their work is restricted to static locomotion tasks which do not allow robots to walk dynamically or jump similarly to humans. An abstraction-based controller was proposed in [\cite{ames2015first}] for bipedal robots using virtual constraints, but this work focused on controller generation without addressing symbolic task reasoning. Recently, the work of [\cite{maniatopoulos2016reactive}] proposed an end-to-end approach to automatically synthesize temporal-logic-based plans on an Atlas humanoid robot. Reaction to low-level failures was formally incorporated by simply terminating the execution. However, the robot behaviors focus on manipulation and grasping tasks, instead of locomotion behaviors. The work of [\cite{sreenath2013partially}] proposed a two-layer hybrid controller for locomotion over varying-slope terrains with imprecise sensing. To account for terrain uncertainties, a high-level controller implements a partially observable Markov decision process to make sequential decisions for controller switching. 
Once again, this work does not address symbolic task reasoning for dynamic locomotion. In addition, this work is limited to walking on terrains with mild roughness while our focus is locomotion on highly rough terrain and constrained environments.

\subsection{Robustness reasoning of formal methods}
Robustness to disturbances and reactiveness to changing environments are major challenges in robotic systems.
Related work includes [\cite{fainekos2009robustness}] which studies the robust satisfaction of temporal logic specifications associated with continuous-time signals. Signal temporal logic (STL) [\cite{donze2010robust}] allows to reason about dense-time, real-valued signals, enabling for the evaluation of the extent to which the specifications are satisfied or violated. This property makes STL especially suitable to quantify robustness [\cite{farahani2015robust, deshmukh2015robust, sadraddini2015robust}]. The focus of all the work above is on the robust semantics of temporal logic, while our objective is to design robust locomotion planners where robustness margin sets are quantified as a goal in the reachability analysis under bounded disturbances. The work of [\cite{majumdar2011robust}] studied robust controller synthesis on discrete transition systems against disturbances and proposed a robust metric to ensure that the state deviation from the nominal system is bounded by the magnitude of the disturbance. The work of [\cite{topcu2012synthesizing}], on the other hand, investigated the amount of uncertainty that can be tolerated while the controller still satisfies the given specifications. Both of the two papers above, however, focused on robustness reasoning in a purely discrete model, whereas our proposed method reasons about robustness in a hybrid locomotion system and incorporates the underlying physical dynamics. Recently, the work in  [\cite{plaku2010motion,bhatia2010motion, he2015towards}] proposed a multi-layered synergistic framework such that the low-level sampling-based planner communicates with the high-level discrete planner through a middle coordinating layer. This coordinating layer allows the motion planner to ask the task planner for a new high-level plan when a failure occurs at the low level. This synergy between multiple planning layers enhances the robustness of the planning framework. As an alternative,
the work of [\cite{Dantam-RSS-16}] incrementally incorporated geometric information from the failure event of the motion planner into the task planner via the so-called incremental constraint updates. The robustness in the two lines of research above is reasoned from a replanning perspective. While our study employs a similar replanning strategy as theirs, our focus is on the formal synthesis of a task planner that can react to sudden event changes in the environment.
In this paper, we address the robustness as follows: (1) at the task planning level, we devise a reactive mechanism that chooses appropriate system actions according to environmental actions, and (2) at the motion planning level, we achieve robustness against bounded state disturbances by designing robust keyframe transitions for dynamic locomotion.

\begin{figure}[t]
 \centering
   \includegraphics[width=0.95\linewidth]{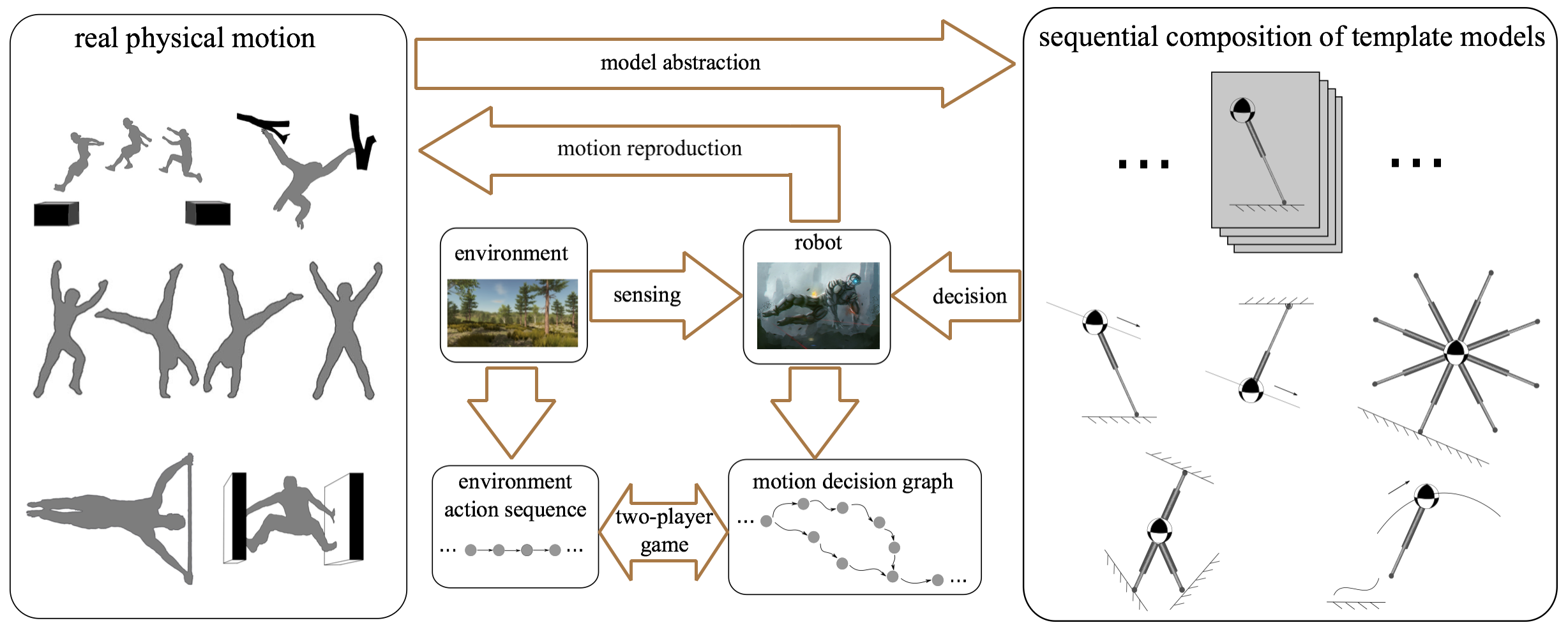}
 \caption{\captionsize Illustration of template-based locomotion behaviors dynamically interacting with complex environments. 
 Inspired from real-world human and animal motions, our study focuses on how to make model abstractions and high-level decisions for complex environments. A fundamental problem is how to use template models to characterize essential locomotion modes and sequentially compose these modes to achieve agile and robust locomotion. 
}
\label{fig:big-picture}
\end{figure}

\subsection{Multi-contact legged locomotion}
Multi-contact locomotion planning and control for humanoid robots have gained good traction as legged robots operate within complex environments more frequently in recent years [\cite{Sentis:10(TRO), chung2015contact, bouyarmane2011multi, hauser2014fast, posa21optimization}]. The work in [\cite{bretl2006motion}] studied multi-contact locomotion as a hybrid control problem while the work in [\cite{hauser2005non}] posed the multi-contact planning problem as a hierarchy that first reasons about contacts, and then interpolated these contacts with trajectories computed from a probabilistic planner. 
The study in [\cite{kudruss2015optimal}] formulated multi-contact centroidal momentum dynamics as an optimal control problem. However, all of the work above focused on either static or quasi-static mobility behaviors. Instead, our planning framework tackles highly dynamic behaviors, i.e., non-periodic multi-contact dynamic locomotion over rough and constrained environments. 
The work in [\cite{caron2015zmp}] employed contact wrench cones to geometrically construct dynamic supports in arbitrary virtual planes for multi-contact behaviors. This work did not employ a rich set of locomotion templates due to restrictive assumptions on the center of mass behavior.
Once again, all the work does not address symbolic reasoning of dynamic locomotion behaviors.

\section{Preliminaries and Problem Formulation}

\noindent \textbf{Problem Statement:} This study focuses on the reactive and robust synthesis of dynamic whole-body locomotion behaviors for robots equipped with arms and legs to maneuver in complex environments exposed to unexpected emergency events. We use a variety of reduced-order models characterizing the robot's center-of-mass dynamic behaviors. Robot actions are parameterized by discrete contact decisions (i.e., limb contact configurations) while environmental actions are composed of various features, such as stair height variations and emergency events, including the appearance of humans, terrain cracks, high ceilings, and narrow passages. A two-player game based on the linear temporal logic method is employed for the robot to be reactive to environmental events. We combine the reactive synthesis and reachability control to provide formal guarantees for locomotion in terms of correctness and robustness. While the synthesized actions and continuous control policies are designed off-line, we make them available as look-up tables for real-time online execution of reactive whole-body locomotion decisions and control commands. In this paper, we choose a specific set of environmental actions to demonstrate the versatility of our method in employing multiple limbs for locomotion and responding to a diversity of environmental changes and emergency events. Our method is flexible to incorporate more diverse environments, such as including the contact from lateral supporting walls or obstacles coming from different directions.

\label{sec:problem-formulation}
%
\subsection{Dynamic locomotion modes}
\label{subsec:low-level-imp}
We design a phase-space motion planner that consists of a palette of locomotion modes. To begin with, we introduce centroidal momentum dynamics in a general form.
Dynamics of mechanical systems can be represented by their rate of linear and angular momenta, which are affected by external wrenches (i.e., force/torque) exerted on the system. We characterize this class of dynamical systems via the  balance of moments around the system's centroid.
\begin{align}\label{eq:linearmomentum}
\boldsymbol{\dot{l}} &= m \boldsymbol{\ddot{p}}_{\rm com} = \sum_i^{N_c} \boldsymbol{f}_i + m \boldsymbol{g},\\[-0.5mm]\label{eq:angularmomentum}
\boldsymbol{\dot{k}} &= \sum_i^{N_c} (\boldsymbol{p}_i - \boldsymbol{p}_{\rm com}) \times \boldsymbol{f}_i + \boldsymbol{\tau}_i,
\end{align}
where $\boldsymbol{l} \in \mathbb{R}^3$ and $\boldsymbol{k} \in \mathbb{R}^3$ represent the centroidal linear and angular momenta, respectively. $\boldsymbol{f}_{i} \in \mathbb{R}^3$ is the $i^{\rm th}$ ground reaction force, $m$ is the total mass of the robot, $\boldsymbol{g} =(0, 0, -g)^T$ corresponds to the gravity field, $\boldsymbol{f}_{\rm com} = m \boldsymbol{\ddot{p}}_{\rm com} = m (\ddot{x}, \ddot{y}, \ddot{z})^T$ is the vector of center-of-mass inertial forces. Eq.~(\ref{eq:linearmomentum}) represents the rate of spatial linear momentum is equal to the total linear external forces. $\boldsymbol{p}_{i} = (p_{i,x}, p_{i,y}, p_{i,z})^T$ is the position of the $i^{\rm th}$ limb contact position. $\boldsymbol{\tau}_i \in \mathbb{R}^3$ is the $i^{\rm th}$ contact torque. Eq.~(\ref{eq:angularmomentum}) reveals that the rate of angular momentum is equal to the sum of the torques generated by contact wrenches at the CoM. 

Given this general model, certain assumptions are commonly imposed to make the problem tractable [\cite{audren2014model}]. In our case, six locomotion modes are proposed to produce various WBDL behaviors. 

\noindent\textbf{Mode (a): prismatic inverted pendulum model.} For single foot contact, Eq.~(\ref{eq:angularmomentum}) is simplified to $(\boldsymbol{p}_{\rm com} - \boldsymbol{p}_{\rm foot}) \times (\boldsymbol{f}_{\rm com} + m \, \boldsymbol{g}) = -\boldsymbol{\tau}_{\rm com}$.
Given a piece-wise linear CoM path surface to follow, the system dynamics are expressed as
\begin{align}\label{eq:PIPM}
\begin{pmatrix}
\ddot x\\
\ddot y
\end{pmatrix}
= \omega_{\rm PIPM}^2 
 \begin{pmatrix}
   x - x_{\rm foot} - \frac{\tau_y}{mg}\\[2mm]
  y - y_{\rm foot} - \frac{\tau_x}{mg}
 \end{pmatrix},
\end{align}
where $\ddot x$ and $\ddot y$ are linear CoM acclerations aligned with sagittal and lateral directions as defined in Eq.~(\ref{eq:linearmomentum}). The PIPM phase-space asymptotic slope [\cite{zhao2017robust}] is defined as
$\omega_{\rm PIPM} = \sqrt{g/z^{\rm apex}_{\rm PIPM}}, \; z^{\rm apex}_{\rm PIPM} = (a \cdot x_{\rm foot} + b \cdot x_{\rm foot} + c  - z_{\rm foot}),$
where $a$ and $b$ are the slopes for the piecewise linear CoM path surface $\psi_{\rm CoM}(x,y,z) = z - a x - b y - c = 0$. Thus, the dynamics in the vertical direction are represented by $\ddot z = a \ddot x + b \ddot y$ and not explicitly shown here. The control input is $\boldsymbol{u} = (x_{\rm foot}, y_{\rm foot}, \omega_{\rm PIPM}, \tau_x, \tau_y)^T$. For more details, please refer to the result in [\cite{zhao2017robust}].
%

\noindent\textbf{Mode (b): prismatic pendulum model}. When the terrain is cracked, the robot has to grasp the overhead support to swing over an unsafe region using brachiation. The system dynamics can be approximated as a prismatic pendulum model (PPM). For a single hand contact, we have
\begin{align}\label{eq:accel2}
\begin{pmatrix}
\ddot x\\
\ddot y
\end{pmatrix}
= - \omega_{\rm PPM}^2 
 \begin{pmatrix}
   x - x_{\rm hand} - \frac{\tau_y}{mg}\\[2mm]
  y - y_{\rm hand} - \frac{\tau_x}{mg}
 \end{pmatrix},
\end{align}
where similarly we can define $\omega_{\rm PPM} = \sqrt{g/z^{\rm apex}_{\rm PPM}}, \; z^{\rm apex}_{\rm PPM} = (z_{\rm hand} - a \cdot x_{\rm hand} - b \cdot x_{\rm hand} - c)$, given the same piece-wise linear CoM path surface $\psi_{\rm CoM}(x,y,z) = z - a x - b y -c = 0$ in Mode (a). Similarly, vertical direction dynamics are represented by $\ddot z = a \ddot x + b \ddot y$. A difference between modes (a) and (b) lies in that PPM dynamics are inherently stable since the CoM is always attracted to move towards the apex position while the PIPM dynamics are not. This study assumes the robot can firmly grasp the overhead support once receiving the upper limb contact command. Fine reasoning of the low-level grasping model and potential failure scenarios are out of the scope of this paper, though important, and will be studied in future work.

\begin{figure*}[t]
 \centering
   \includegraphics[width=0.95\linewidth]{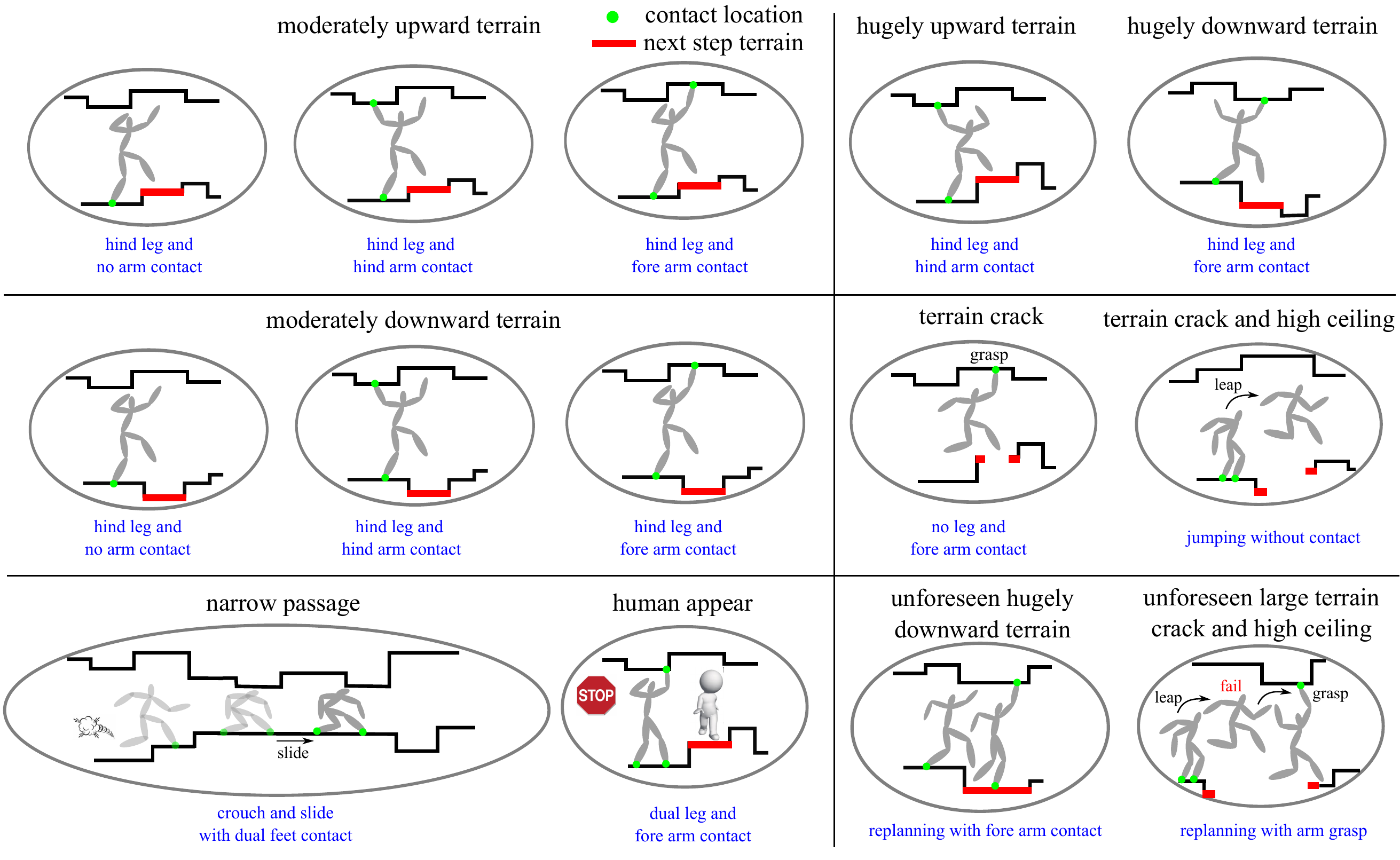}
 \caption{\captionsize Contact planning strategies for locomotion in rough terrains. We discretize the terrain height to decide what locomotion actions to take, and define them as environmental actions to set up a two player game decision problem. For instance, given a moderately upward or downward terrain, there can be multiple contact actions to deal with it. Events motivated by ordinary accidents in human daily lives, such as a crack on the terrain and the sudden appearance of a human, are treated as emergency events, and incorporated into the allowable environment. Detailed definitions of environment and system actions are provided in Section~\ref{subsec:actionSpec}.
}
\label{fig:ContactAction}
\end{figure*}

\noindent\textbf{Mode (c): stop-launch model.} When a human appears, the robot has to come to a stop, wait until human disappears, and start to move forward. The task in this mode consists on decelerating the CoM motion to zero and accelerating it from zero again. We name this model as a stop-launch model (SLM) with 
%
%
a constant CoM sagittal accelerations. 
The resulting phase-space trajectory is a parabolic manifold. 

\begin{figure}[t]
 \centering 
   \includegraphics[width=0.75\linewidth]{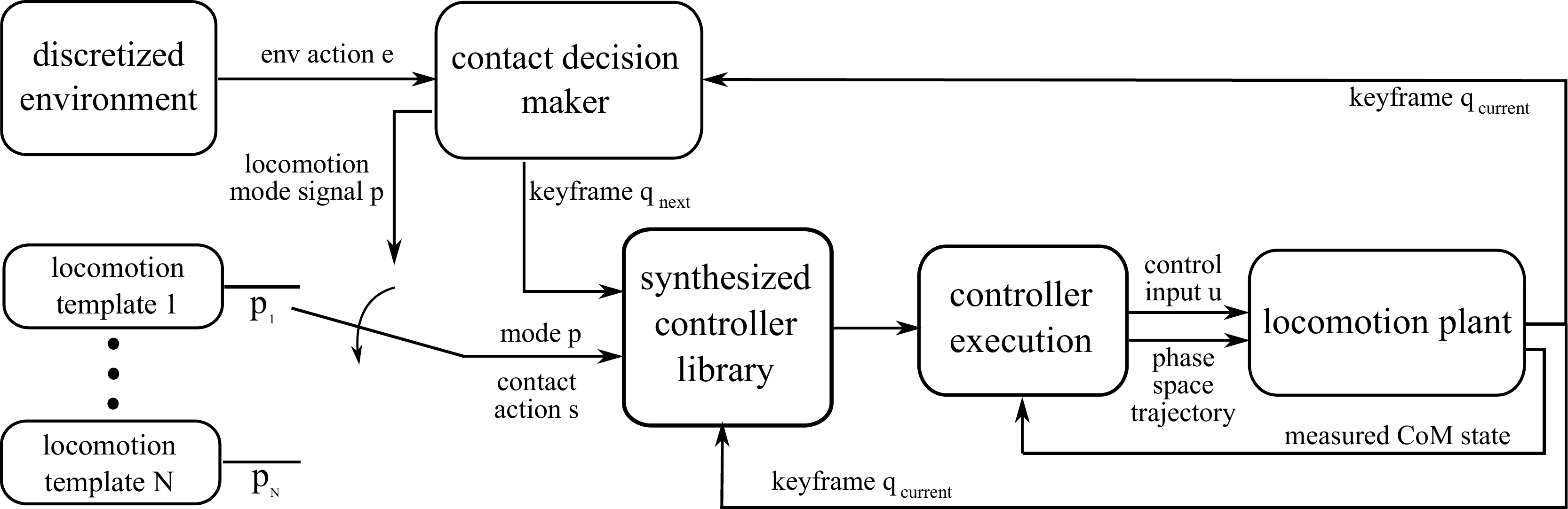}
 \caption{\captionsize Logic-based locomotion planner structure. A set of locomotion templates is devised for maneuvering in constrained dynamic environments. Each template is indexed by a locomotion mode signal $p$. The discrete environment actions are represented by the variable $e$ while control actions are represented by the variable $s$ describing limb contact actions. The discretized dynamic locomotion keyframes are represented by $q = (p_{\rm contact}, \dot{x}_{\rm apex})$. Based on an environmental action $e$ and a keyframe state $q$ at the current walking step, the contact decision maker decides the locomotion mode signal $p$ and the next keyframe locomotion state. More details on the usage of this decision process are discussed in Section~\ref{sec:Abstraction}.
}
\label{fig:SwitchingCtrl}
\end{figure}

\noindent\textbf{Mode (d): multi-contact model.}
In this mode, a multi-contact model (${\rm MCM}$) is proposed built upon the centroidal momentum dynamics. To make the dynamics tractable, we assume a known constant vertical acceleration $a_{z}$ in each step and neglect of the angular momentum $k_z$ around the $z$-axis [\cite{audren2014model}]. Therefore, we have a constant resultant vertical external force, i.e., $\sum_i^{N_c} f_{i,z} = m (\ddot z - g)$, where $N_c$ is the number of limb contacts. Since our model has point contacts, $\boldsymbol{\tau}_i = 0, \forall i \leq N_c$, and the dynamics are described by
\begin{align}\nonumber
\begin{pmatrix}
\ddot x\\[1mm]
\ddot y\\[1mm]
\ddot \varphi\\[1mm] 
\ddot \theta
\end{pmatrix}
 & =
\begin{pmatrix}
 \sum_i^{N_c} f_{i,x}/m\\[1mm]
  \sum_i^{N_c} f_{i,y}/m\\[1mm]
  - (\ddot z - g)\cdot y + z\cdot \sum_i^{N_c} f_{i,y}/m  - \sum_i^{N_c} p_{i,z} \cdot f_{i,x}/m  + \sum_i^{N_c} p_{i,z}\cdot f_{i,z}/m\\[1mm]
  (\ddot z - g)\cdot x - z \cdot \sum_i^{N_c} f_{i,x}/m + \sum_i^{N_c} p_{i,z} \cdot f_{i,x}/m - \sum_i^{N_c} p_{i,y}\cdot f_{i,y}/m
\end{pmatrix},
\end{align}
where $\varphi$ and $\theta$ are torso roll and pitch angles aligned with the CoM sagittal and lateral directions as derived from Eq.~(\ref{eq:angularmomentum}). The external force vector $(f_{i, x}, f_{i, y}, f_{i, z})$ represents the $i^{\rm th}$ contact force. The vertical position $z$ is a function of $x$ and $y$ defined \textit{a priori}.

\noindent\textbf{Mode (e): hopping model.}
This model applies when the locomotion model needs to jump over an unsafe region. In this case, the CoM dynamics follow a free-falling ballistic trajectory. We have $\ddot x  = \ddot y  = 0, \ddot z  = -g$. The trajectory is fully controlled by the initial condition, where a discontinuous jump in the CoM state can occur and be used to generate a desired linear momentum. For instance, when the robot jumps over a cracked terrain, it needs to push the ground as the foot lifts to generate a sufficiently large sagittal linear acceleration.

\noindent\textbf{Mode (f): sliding model.}
This model applies when the robot needs to slide through a constrained region. The CoM dynamics are subject to a constant friction force. Thus, $\ddot x$ is a constant negative value, and we assume $\ddot y  = 0, \ddot z  = 0$. The sagittal linear velocity decays at a constant rate.

Given the locomotion modes above, we define the set of locomotion modes as
\begin{align}\nonumber
\mathcal{P} \coloneqq \{p_{\rm PIPM}, p_{\rm MCM}, p_{\rm PPM}, p_{\rm SLM}, p_{\rm HM}, p_{\rm SM}\}.
\end{align}
All the locomotion modes above are illustrated in Fig.~\ref{fig:ContactAction}. Each mode has closed-form solutions for their phase-space tangent and cotangent manifolds as will be derived in Section~\ref{sec:Abstraction} and Appendix~\ref{appen:PSManifold}. The timing synchronization between the sagittal and lateral dynamics is guaranteed by a Newton-Raphson foot placement searching algorithm [\cite{zhao2017robust}].
Likewise, more complex tasks can be defined in the locomotion mode set $\mathcal{P}$. For instance, cartwheel, dense gaps, and spinkick behaviors as shown in [\cite{peng2018deepmimic}] are promising behaviors to be explored.

{Our phase-space planning process produces three-dimensional locomotion. However, the planning framework of this study focuses on forward walking using sagittal keyframes. Given high-level sagittal keyframes, the robot's lateral dynamic behavior is automatically computed by our motion planner. Turning behaviors can be incorporated in our framework by using the method that we introduced in}{[\cite{zhao2017robust}]}.

\subsection{Switched systems and phase-space planning}
\label{subsec:PSPlanning}
Given the continuous locomotion modes above, we formulate the locomotion planning problem as a switched system [\cite{liberzon2012switching}]. The dynamics of the whole-body dynamic locomotion (WBDL) process are defined as
\begin{align}\label{eq:switchedsystem}
\dot{\boldsymbol{\xi}}(\zeta) = f_p\big(\boldsymbol{\xi}(\zeta), \boldsymbol{u}(\zeta), d(\zeta)\big), \; p \in \mathcal{P},
\end{align}
where $\boldsymbol{\xi}(\zeta) \in \Xi \subseteq \mathbb{R}^{12}$ denotes the full system state vector at $\zeta \in \mathbb{R}_{\geq 0}$, i.e., the twelve dimensional center-of-mass position and angular state vector of the robot during the locomotion process\footnote{The state vector $\boldsymbol{\xi}$ is reused to represent the center-of-mass sagittal states $(x, \dot x)$ when we model specific locomotion modes in later sections.}. The phase progression variable $\zeta$, analogous to time, represents the current phase progression on a locomotion trajectory. The control input is denoted by $\boldsymbol{u}(\zeta) = (\boldsymbol{p}_{\rm contact}, \omega, \tau_x, \tau_y, \tau_z) \in \mathcal{U}$, where $\boldsymbol{p}_{\rm contact}$ represents a set of contact position vectors, where each contact position vector is three-dimensional;  $\omega$ represents the slope of the phase-space asymptote dependent on specific locomotion modes as defined in Section~\ref{subsec:low-level-imp}; and $(\tau_x, \tau_y, \tau_z)$ represents a three-dimensional torso torque vector. Each locomotion mode merely involves a subset of the full state and control vectors. In addition, $d \in \mathcal{D} \subseteq \mathbb{R}^d$ represents an external disturbance. The locomotion mode $p$ (i.e., the locomotion mode) indexes a specific locomotion mode belonging to the set $\mathcal{P}$ and $f_p(\cdot)$ denotes a vector field associated with the locomotion mode $p$. A logic-based switched system modeling the locomotion process is shown in Fig.~\ref{fig:SwitchingCtrl}.

Our phase-space planning is a three-dimensional hybrid bipedal locomotion planning framework based on robustly tracking a set of non-periodic keyframe states. This framework focuses on non-periodic gait generation for robust and agile locomotion over various challenging terrains and under external disturbances. The keyframe state in the phase-space is defined as

\begin{figure}[t]
 \centering
   \includegraphics[width=0.95\linewidth]{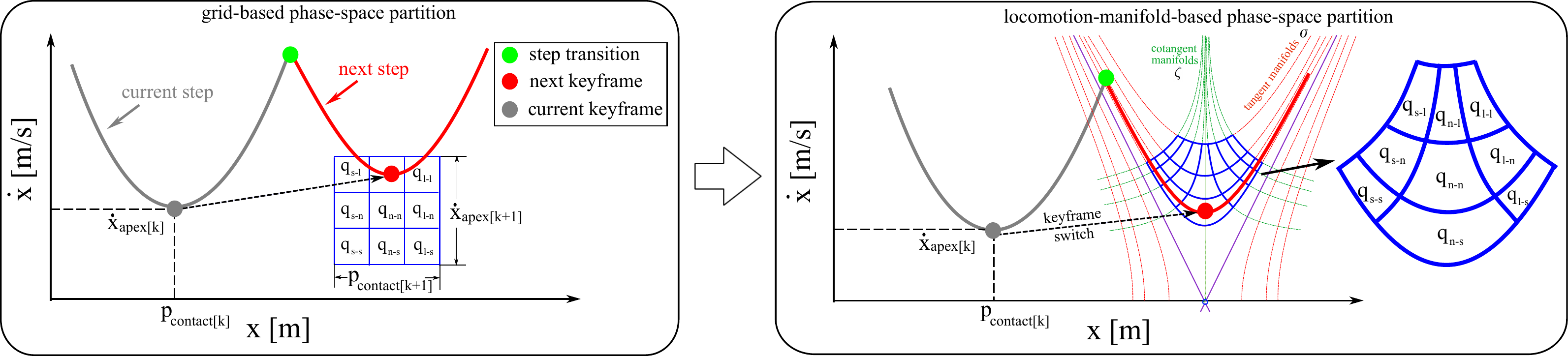}
 \caption{\captionsize Phase-space partition of locomotion manifolds for keyframe design. The left figure shows a grid-based partition while the right figure is a non-Euclidean partition that follows the phase-space locomotion manifolds. The latter partition is consistent with locomotion dynamics and we define it as the "locomotion-manifold-based partition". This partition will be used to achieve robust locomotion. We use different granularities for two orthogonal axes.}
\label{fig:Discretization}
\end{figure}

\begin{definition}[\textbf{Phase-space keyframe}]
A keyframe state in the phase-space of a locomotion system is a critical point on the locomotion manifold normally located either at the point of minimal or maximal velocity, or at an approximately central position of the phase-space manifold of one continuous walking step (see the gray and red dots in Fig.~\ref{fig:Discretization}).
\end{definition}

In general, this keyframe state refers to the apex state when the center-of-mass (CoM) velocity reaches the local minimal or maximum velocity in the CoM sagittal axis.\footnote{In the special case of the phase-space trajectory having a constant slope, we choose the CoM state locating at the central position of the phase-space trajectory as the keyframe state.} Given two consecutive keyframe states, the phase-space planner evolves continuously and computes the contact transitions of one walking step as defined below.
\begin{definition}[\textbf{Phase-space contact switch}]
A phase-space contact switch, i.e., a contact transition, is defined by the intersection of two adjacent phase-space trajectories (see green dot in Fig.~\ref{fig:Discretization}(a)).
\end{definition}
Our contact-triggered switching strategy is especially suitable for non-periodic locomotion, which is abstracted as a progression map $\Phi$ between keyframe states, that is, driving the robot's center-of-mass from one desired keyframe to the next one via the control input $\boldsymbol{u}$, i.e. $(p_{{\rm contact}_{k+1}}, \dot{x}_{{\rm apex}_{k+1}}) = \Phi(p_{{\rm contact}_k}, \dot{x}_{{\rm apex}_k}, \boldsymbol{u})$, where $p_{{\rm contact}_k}$ and $\dot{x}_{{\rm apex}_k}$ denote the $k^{\rm th}$-step CoM sagittal position and velocity at the contact apex, respectively. To accomplish whole-body dynamic locomotion behaviors, we will compose a sequence of locomotion modes with planned keyframes. This can be achieved by synthesizing a high-level task planner protocol which makes proper contact decisions like the ones shown 
in Fig.~\mbox{\ref{fig:ContactAction}} and determines the switching strategy of the low-level motion planner. 

%
\begin{definition}[\textbf{One walking step in the phase-space}]\label{def:one-walking-step}
One walking step (OWS) of the locomotion process is defined as two consecutive semi-step phase-space trajectories (see Fig.~\ref{fig:Discretization}(a)). The first semi-step trajectory starts at the first keyframe state (gray dot) and ends at the contact switch (green dot) while the second semi-step trajectory starts at the contact switch and ends at the second keyframe (red dot). 
\end{definition}
Instead of using generalized coordinates associated with the robot joints, our planning framework chooses to use the robot's center-of-mass state as the output space. This simplified coordinate choice is often used in the locomotion communities. Alternatives for dimensionality reduction include, for instance, differential flatness [\cite{liu2012synthesis}] and partial hybrid zero dynamics [\cite{ames2015first}]. 
 
The switched system dynamics in Eq.~(\ref{eq:switchedsystem}) can be represented by a tuple
\begin{align}\label{eq:SwitchTuple}
\mathcal{SS} = (\Xi, \Xi_0, \mathcal{U}, \mathcal{P}, f, AP, \mathcal{L})
\end{align}
where $\Xi_0 \subseteq \Xi$ is a set of initial conditions, $AP$ is a set of atomic propositions and $\mathcal{L}: \Xi\rightarrow 2^{AP}$ is a labeling function. Then a control strategy for $\mathcal{SS}$ is a partial function defined as
\begin{align}
\label{eq:controlstrategy}
\Omega_i(\boldsymbol{\xi}_0, \boldsymbol{\xi}_1, \ldots, \boldsymbol{\xi}_i) = \boldsymbol{u}_i \in \mathcal{U}^{[0, \Delta \zeta _i]}, \forall i = 0, 1, 2, ...
\end{align}
where $\boldsymbol{\xi}_0, \boldsymbol{\xi}_1, \ldots, \boldsymbol{\xi}_i$ is a finite sequence of sampled states evaluated at discrete phase progression instants $\zeta_0, \zeta_1, \ldots, \zeta_i$ satisfying $\zeta_{j+1} -  \zeta_{j} = \Delta \zeta_i, \forall 0 \leq j \leq i-1$, and $\mathcal{U}^{[0, \Delta \zeta _i]}$ denotes the set of control input signals from $[0, \Delta \zeta _i]$ to $\mathcal{U}$. It is assumed that  $\boldsymbol{u}_i$ is the constant control input with a phase progression duration $\Delta \zeta_i$.

\textbf{Contact switching planner synthesis problem:} Given a switched system $\mathcal{SS}$ in Eq.~(\ref{eq:SwitchTuple}) and a specification $\varphi$ expressible in the linear temporal logic (LTL) form, synthesize a contact planning strategy for the system that (i) only generates correct phase-space trajectories $\kappa = (\boldsymbol{\xi}, \rho, \eta, \mu)$ in the sense that $\kappa \models \varphi$ for all initial conditions in $\Xi_0$, (ii) generates a locomotion mode $\mu$ in response to the environment actions at runtime. $\varphi$ is realizable by $\mathcal{SS}$ if there exists such a switching strategy. $\rho, \eta$ and $\mu$ are the continuous counterparts of the discrete environment action $e$, contact action $s$, and locomotion mode $p$. More detailed definitions will be introduced in Section~{\ref{sec:Abstraction}}.

%
%
\subsection{Finite transition systems and LTL preliminaries}
We now define system, environment, and product finite transition systems and describe linear temporal logic (LTL) preliminaries. 
\begin{definition}[\textbf{Finite transition system of the robot system}]\label{def:SFTS}
A finite transition system of the robot system is a tuple,
\begin{equation}\label{eq:SFTS}
\mathcal{TS}_s \coloneqq (\mathcal{Q}, \mathcal{P}, \mathcal{S}, \mathcal{T}_s, \mathcal{I}_s, AP_s, \mathcal{\tilde{L}}_s),
\end{equation}
where $\mathcal{Q}$ is a finite set of states, $\mathcal{P}$ is a set of system modes as mentioned in Eq.~(\ref{eq:switchedsystem}), $\mathcal{S}$ is a finite set of controllable robot contact actions, $\mathcal{T}_s \subseteq \mathcal{Q} \xrightarrow{\mathcal{P} \times \mathcal{S}} \mathcal{Q}$ is a transition, $\mathcal{I}_s = \mathcal{Q}_0 \subseteq \mathcal{Q}$ is a set of initial states, $AP_s$ is a set of atomic propositions, $\mathcal{\tilde{L}}_s: \mathcal{Q} \rightarrow 2^{AP_s}$ is a labeling function mapping the state to an atomic proposition. $\mathcal{TS}_s$ is finite if $\mathcal{Q}, \mathcal{P}, \mathcal{S}$ and $AP_s$ are finite.
\end{definition}
%
%

\begin{figure}[t]
 \centering
   \includegraphics[width=0.45\linewidth]{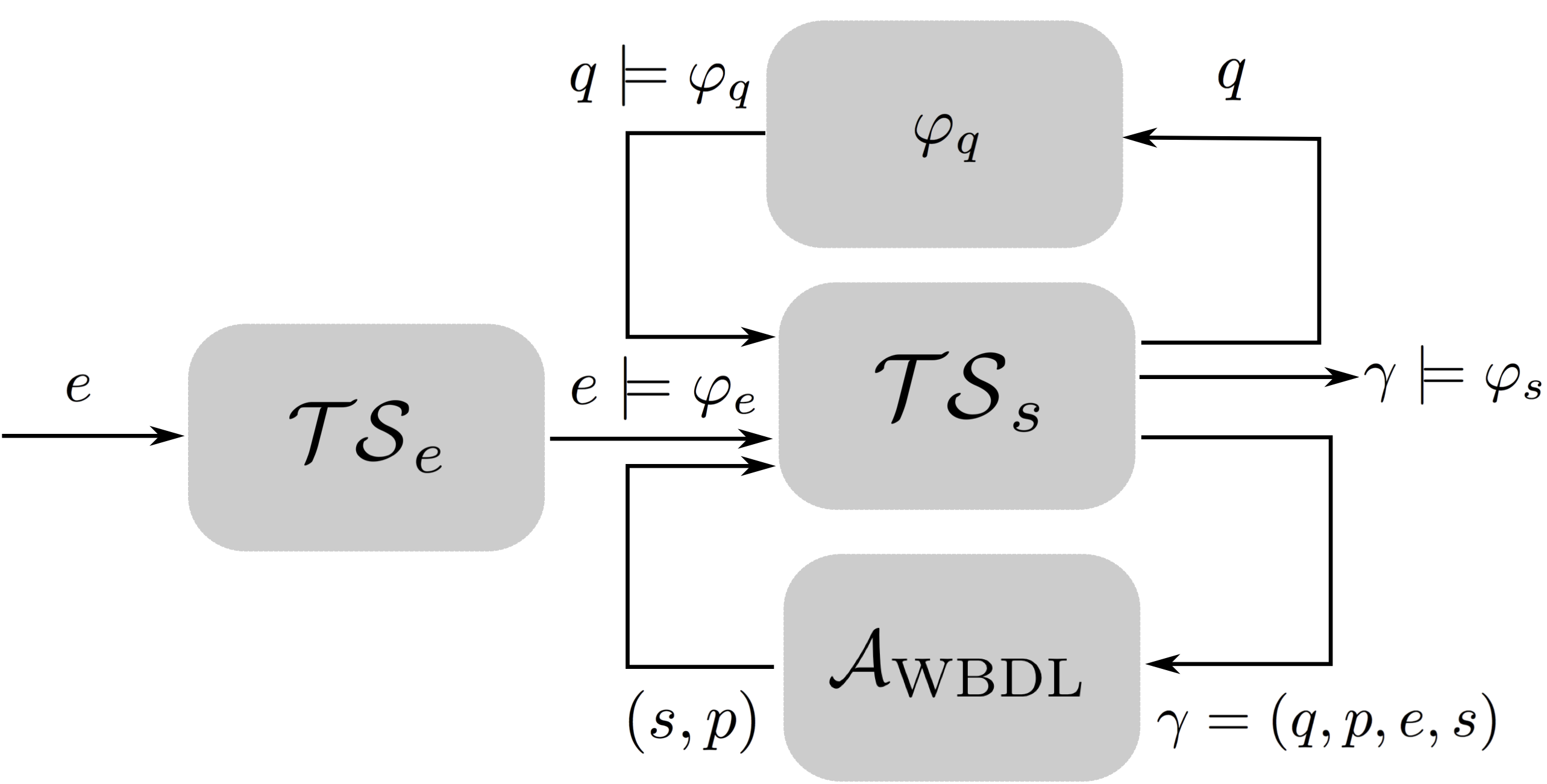}
 \caption{\captionsize The interconnected feedback diagram of system and environment finite transition systems $\mathcal{TS}_s$ and $\mathcal{TS}_e$, and a winning strategy $\mathcal{A}_{\rm WBDL}$ synthesized in Section~\ref{sec:high-level-planner}.
}
\label{fig:Diagram} 
\end{figure}

\begin{definition}[\textbf{Finite transition system of the environment}]\label{def:EFTS}
A finite transition system of the environment is a tuple, 
\begin{equation}\label{eq:EFTS}
\mathcal{TS}_e \coloneqq (\mathcal{E}, \mathcal{T}_e, \mathcal{I}_e, AP_e, \mathcal{\tilde{L}}_e),
\end{equation}
where $\mathcal{E}$ is a finite set of environmental states, $\mathcal{T}_e \subseteq \mathcal{E} \times \mathcal{E}$ is a transition, $\mathcal{I}_e = \mathcal{E}_0 \subseteq \mathcal{E}$ is a set of initial states, $AP_e$ is a set of atomic propositions, $\mathcal{\tilde{L}}_e: \mathcal{E} \rightarrow 2^{AP_e}$ is a labeling function mapping the state to an atomic proposition. $\mathcal{TS}_e$ is finite if $\mathcal{E}$ and $AP_e$ are finite.
\end{definition}
\begin{definition}[\textbf{Open finite product transition system}]\label{def:OFPTS}  
Given $\mathcal{TS}_s$ and $\mathcal{TS}_e$, we define an open finite product transition system (OFPTS) to describe the overall system behavior, including the robot and its environment as, 
\begin{equation}\label{eq:OFPTS}
\mathcal{TS}_{\rm prod} \coloneqq (\mathcal{Q}, \mathcal{P}, \mathcal{S}, \mathcal{E}, \mathcal{T}, \mathcal{I}, \overset{\sim}{AP},  \mathcal{\tilde{L}}),
\end{equation}
where $\mathcal{Q}, \mathcal{P}$ and $\mathcal{S}$ are defined as previously, $\mathcal{E}$ is a finite set of uncontrollable  environmental actions, $\mathcal{F} = \mathcal{Q} \times \mathcal{P} \times \mathcal{S} \times \mathcal{E}$, $\mathcal{T} \subseteq \mathcal{F} \rightarrow \mathcal{F}$ is a transition, $\mathcal{I} = \mathcal{F}_0 \subseteq \mathcal{F}$ is a set of initial states, $\overset{\sim}{AP}$ is a set of atomic propositions, $\mathcal{\tilde{L}}: \mathcal{Q} \rightarrow 2^{\overset{\sim}{AP}}$ is a labeling function mapping the state to an atomic proposition. $\mathcal{TS}_{\rm prod}$ is finite if $\mathcal{Q}, \mathcal{P}, \mathcal{E}, \mathcal{S}$ and $\overset{\sim}{AP}$ are finite.
\end{definition}
\noindent Note that, the environment states $\mathcal{E}$ in $\mathcal{TS}_e$ are treated as uncontrollable actions in $\mathcal{TS}_{\rm prod}$. This is why $\mathcal{TS}_{\rm prod}$ is called a ``open'' finite transition system [\cite{topcu2012synthesizing}]. Without loss of generality, it is assumed that for every pair $(q, e) \in \mathcal{Q} \times \mathcal{E}$, there exists at least one pair $(p, s)$ such that $(q, e) \xrightarrow{p, s} (q', e')$. The OFPTS considered in this study has non-deterministic transitions. 

\begin{definition}[\textbf{Execution and word of an OFPTS}]\label{def:Execution}
An execution $\boldsymbol{\gamma}$ of an OFPTS $\mathcal{TS}_{\rm prod}$ is an infinite path sequence $\boldsymbol{\gamma} = (q_0, p_0, e_0, s_0) (q_1, p_1, e_1, s_1) (q_2, p_2, e_2, s_2) \ldots$, with $\gamma_i = (q_i, p_i,e_i, s_i) \in \mathcal{Q} \times \mathcal{P} \times \mathcal{E} \times \mathcal{S}$ and $\gamma_i \xrightarrow{\mathcal{T}} \gamma_{i+1}$. The \textit{word} generated from $\boldsymbol{\gamma}$ is $w_{\gamma} = w_{\gamma}(0) w_{\gamma}(1) w_{\gamma}(2) \ldots$, with $w_{\gamma}(i) = \mathcal{\tilde{L}}(\gamma_i), \; \forall i \geq 0$. 
\end{definition}
\noindent The word $w_{\gamma}$ is said to satisfy a LTL formula $\varphi$, if and only if the execution $\boldsymbol{\gamma}$ satisfies $\varphi$. If all executions of $\mathcal{TS}_{\rm prod}$ satisfy $\varphi$, we say that $\mathcal{TS}_{\rm prod}$ satisfies $\varphi$, i.e., $\mathcal{TS}_{\rm prod} \models \varphi$. Please refer to Fig.~\ref{fig:Diagram} for an illustration of the finite transition systems. Linear temporal logic is an extension of propositional logic that incorporates temporal operators. Preliminaries of linear temporal logic are explained in Appendix~\ref{subsec:LTL-preliminary}.



\subsection{Discrete task planner synthesis formulation}
Given the preliminaries above, we formulate a discrete task planner synthesis problem and introduce a specific fragment of the temporal logic for the task specifications.

\noindent \textbf{Discrete task planner synthesis problem:} Given a product transition system $\mathcal{TS}_{\rm prod}$ and a LTL specification $\varphi$ following the assume-guarantee form [\cite{bloem2012synthesis}], 
\begin{align}\label{eq:gr1}
\varphi \coloneqq \big(\varphi_e \Rightarrow (\varphi_q \wedge \varphi_s) \big),
\end{align}
where $\varphi_e$ and $\varphi_q, \varphi_s$ are  propositions for the admissible environment actions, the keyframe states, and the correct overall system behavior, respectively; in particular, $\varphi_s$ incorporates the behaviors of locomotion mode $p$ and contact action $s$; we synthesize a contact planner switching strategy $\gamma$ that generates only correct executions $(q, p, e, s)$, i.e., $(q, p, e, s) \models \varphi$.

To make the computation tractable, we employ a fragment of LTL formulae with a favorable polynomial complexity, named the Generalized Reactivity (1) (GR (1)) formulae [\cite{bloem2012synthesis}]. This class of formulae is expressed as, for $v \in \{e, q, s\}$,
\begin{align}\label{eq:GR}
\varphi_v = \varphi^v_{\rm init} \bigwedge_{i \in I_{\rm safety}} \Box \varphi^v_{{\rm trans}, i} \bigwedge_{i \in I_{\rm goal}} \Box \Diamond \varphi^v_{{\rm goal}, i},
\end{align}
where $\varphi^v_{\rm init}$ are the propositional formulae defining initial conditions. $\varphi^v_{{\rm trans},i}$ refer to the transitional propositional formulae (i.e., safety conditions) incorporating the state at next step. $\varphi^v_{{\rm goal}, i}$ are the propositional formulae describing the goals to be reached infinitely often (i.e., liveness conditions). 

\begin{remark}
The GR(1) formula is an efficient fragment of LTL and reasons over a rich set of states and actions and makes the task planner synthesis process tractable. A motivation of using this automated synthesis is to lay the theoretical foundation of devising a correct-by-construction decision-maker for composing complex locomotion trajectories.
\end{remark}
%

 \section{Task Planning for Whole-body Dynamic Locomotion}
\label{sec:WBDLSpec}
In this section, we introduce the temporal logic specifications for locomotion in a possibly adversarial environment. We will specify a two-player game where the environment and keyframe state are the first player while the robot action is the second player. Our task specifications will capture two types of environmental events: (i) varying height terrains are treated as ordinary events; and (ii) sudden incidents, such as a person appearing on the robot's path or a crack on the terrain, are treated as adversarial actions, since if the robot does not respond properly they may cause an accident.

{The design of linear temporal logic (LTL) specifications relies on human designers who specify locomotion tasks and models of the environment. Our task specifications below are designed according to locomotion heuristics. In general, there are no unique ways to evaluate the efficacy of the LTL design process. For instance, we could decide to add an additional environmental specification to forbid repeating the same environmental actions \textsf{\small terrainCrack-normalCeiling} $e_{\rm tc{\text -}nc}$, expressed as $\Box (e_{\rm tc{\text -}nc} \Rightarrow \neg \bigcirc e_{\rm tc{\text -}nc})$. Using locomotion heuristics is an effective way to generate natural and safe locomotion behaviors. 
The heuristics that we have chosen employ human intuition, which often lead to natural and recognizable behaviors. In addition, another set of heuristics is used to guarantee safety.}

\subsection{Environment specifications}
\label{subsec:actionSpec}
As previously stated, we treat the environment as a player ``acting'' against the robot's locomotion process. We define an environmental action set, $\mathcal{E}$, as the composition of two subsets: a set for varying height terrain and a set for emergencies (i.e., the so-called sudden events), respectively.
\begin{align}\label{eq:envirActions}
\mathcal{E} \coloneqq \mathcal{E}_{\rm terrain} \cup \mathcal{E}_{\rm emergency} = \{e_{\rm md}, e_{\rm hd}, e_{\rm mu}, e_{\rm hu} \} \cup \{e_{\rm tc{\text -}nc}, e_{\rm tc{\text -}hc}, e_{\rm ha}, e_{\rm np}\},
\end{align}
where the elements in the set $\mathcal{E}_{\rm terrain}$ denote different height terrain actions, as illustrated in Fig.~\ref{fig:ContactAction}. For instance, $e_{\rm md}$ denotes \textsf{\small moderatelyDownward} terrain. The actions in $\mathcal{E}_{\rm emergency}$ represent sudden events, i.e. \textsf{\small terrainCrack-normalCeiling}, \textsf{\small terrainCrack-highCeiling}, \textsf{\small humanAppear}, and \textsf{\small narrowPassage}. The environmental action set specified above is generalizable to other environmental events while maintaining computational tractability.
%
Given the environmental actions above, we design the following specifications. 
First, the following sudden environmental actions are assumed to not occur at the initial instant:
\begin{align}\label{eq:env-initial-prop}
\varphi^e_{\rm init} = \neg e_{\rm tc{\text -}nc} \wedge \neg  e_{\rm tc{\text -}hc} \wedge \neg e_{\rm ha} \wedge \neg e_{\rm np}
\end{align}
Since only one environmental action can be \textsf{\small True} at any time, we enforce the following transitional proposition
\begin{align}\label{eq:environment-goal-spec}
\Box \Big( \big( e_{\rm md} \land \bigwedge_{e \in \mathcal{E} \backslash e_{\rm md}} (\neg e) \big) \bigvee \big(e_{\rm hd} \land \bigwedge_{e \in \mathcal{E} \backslash e_{\rm hd}} (\neg e) \big) \bigvee \ldots \bigvee \big(e_{\rm np} \land \bigwedge_{e \in \mathcal{E} \backslash e_{\rm np}} (\neg e) \big) \Big).
\end{align}
where the operator $\bigwedge_{e \in \mathcal{E} \backslash e_{\rm md}}$ is used to represent the conjunction of multiple environmental propositions $\neg e, \forall e \in \mathcal{E} \backslash e_{\rm md}$.
To enable the robot to maneuver through the dynamic environment, certain sudden environmental actions are forbidden to occur consecutively, as shown in the following transitional specifications:
\begin{itemize}
\item ($S_{e{\text -}1}$) If the current environmental action is \textsf{\small terrainCrack-highCeiling}, then the next environmental action can not be \textsf{\small terrainCrack-highCeiling}, \textsf{\small humanAppear}, nor \textsf{\small narrowPassage}.
\begin{align}
\Box \Big(e_{\rm sc\text{-}hc} \Rightarrow \neg (e_{\rm sc\text{-}hc}\wedge e_{\rm ha} \wedge e_{\rm np})\Big)
\end{align}
\item ($S_{e{\text -}2}$) If the current environmental action is \textsf{\small terrainCrack-normalCeiling}, then the next environmental action can not be \textsf{\small terrainCrack-highCeiling}, \textsf{\small humanAppear}, nor \textsf{\small narrowPassage}.
\begin{align}
\Box \Big(e_{\rm sc\text{-}nc} \Rightarrow \neg (e_{\rm sc\text{-}hc}\wedge e_{\rm ha} \wedge e_{\rm np})\Big)
\end{align}
\item ($S_{e{\text -}3}$) If the current environmental action is \textsf{\small narrowPassage}, then the next environmental action can not be \textsf{\small terrainCrack-normalCeiling} nor \textsf{\small terrainCrack-highCeiling}.
\begin{align}
\Box \Big(e_{\rm np} \Rightarrow \neg (e_{\rm sc\text{-}hc} \wedge e_{\rm sc\text{-}nc} )\Big)
\end{align}
\end{itemize}
To evaluate the effectiveness of our proposed approach handling all the allowable environmental actions, we enforce them to occur infinitely often via the goal proposition:
\begin{align}\label{eq:environment-goal}
\varphi^e_{\rm goal} = (\Box \Diamond e_{\rm md}) \wedge (\Box \Diamond e_{\rm hd}) \wedge \ldots \wedge (\Box \Diamond e_{\rm np})
\end{align}

To ensure the robot makes progress (i.e., continuously moves forward within the constrained environment), we define the following liveness condition:
\begin{align}
\varphi^{e}_{\rm liveness} \coloneqq \neg \Diamond \Box e_{\rm ha} \wedge \neg \Diamond \Box e_{\rm np} 
\end{align}
which is consistent with the goal proposition of Eq.~(\ref{eq:GR}). This specification establishes that the robot cannot eventually always encounter the conditions \textsf{\small humanAppear} or \textsf{\small narrowPassage}. In fact, this liveness condition should also include the environmental action \textsf{\small terrainCrack-highCeiling}, i.e., $\neg \Diamond \Box e_{\rm sc\text{-}hc}$, which is already guaranteed by $\Box (e_{\rm sc\text{-}hc} \Rightarrow \neg e_{\rm sc\text{-}hc})$ in specification ($S_{e{\text -}1}$).

\subsection{Robot specifications}
To maneuver in the environment using whole-body dynamic locomotion, we define the following robot actions 
\begin{align}\label{eq:systemActions}
\mathcal{S} \coloneqq \{s_{\rm li\text{-}aj}, \; \forall (i, j) \in \mathcal{S}_{\rm index}\},
\end{align}
where the indices `$\rm l$' and `$\rm a$' are short for \textsf{\small leg} and \textsf{\small arm}, respectively. $(i, j)\in \mathcal{S}_{\rm index}$ corresponds to the contact limb with $\mathcal{S}_{\rm index} = {\rm \{(h, n), (h, h), (h, f), (d, h), (d, f), (d, d), (d, n), (n, f), (n, n)\}}$, where the letters `$\rm h$', `$\rm f$', `$\rm d$' and `$\rm n$' represent \textsf{\small hind}, \textsf{\small fore}, \textsf{\small dual} and \textsf{\small no} contacts, respectively. For instance, $s_{\rm lh\text{-}af}$ specifies the \textsf{\small legHindArmFore} contact action in the sense that the robot's hind leg and the fore arm are in contact for that action while the other two limbs are not in contact. Notice that we don't specify left and right limbs explicitly as the \textsf{\small hind} and \textsf{\small fore} adjectives lead to unique assignments during the locomotion process. 

We enforce the robot not to take actions responding to emergency events of the environment $\mathcal{E}_{\rm emergency}$, i.e., $\varphi^s_{\rm init} = \neg s_{\rm ln\text{-}af} \wedge \neg s_{\rm ln\text{-}an} \wedge \neg (s_{\rm ld\text{-}ah} \vee s_{\rm ld\text{-}af})$, which are already guaranteed by the initial propositions defined for the environmental actions in Eq.~(\ref{eq:env-initial-prop}). Given a specific set of locomotion modes $\mathcal{P}$ as defined in Section~\ref{subsec:low-level-imp}, the robot transitional specifications $\varphi^s_{\rm trans}$ are defined as follows:
%
%
\begin{itemize}
%
\item ($S_{\rm robot{\text -}1}$) Robot actions in response to varying-height terrain $\mathcal{E}_{\rm terrain}$ are specified as
\begin{align}\nonumber
&\Box \Big( (e_{\rm md} \vee e_{\rm mu}) \Rightarrow (p_{\rm PIPM} \wedge s_{\rm lh\text{-}an}) \vee \big(p_{\rm MCM} \wedge (s_{\rm lh\text{-}ah} \vee s_{\rm lh\text{-}af})\big)\Big) \\\nonumber
\bigwedge & \Box (e_{\rm hu} \Rightarrow p_{\rm MCM} \wedge s_{\rm lh\text{-}ah}) \bigwedge \Box (e_{\rm hd} \Rightarrow p_{\rm MCM} \wedge s_{\rm lh\text{-}af}),
\end{align}
where \textsf{\small moderate} terrain variations allow for the use of more robot contact actions than in the case of \textsf{\small huge} terrain variations. For instance, if $e = e_{\rm hu}$, i.e. the terrain has an action \textsf{\small hugelyUpward}, the robot has only one action to choose from, consisting of using its hind arm for contact such that it can push forward its center of mass to overcome the huge terrain variation as shown in Fig.~\ref{fig:ContactAction}. 
\item ($S_{\rm robot{\text -}2}$) If the environmental action \textsf{\small terrainCrack\text{-}normalCeiling} occurs, i.e., a crack on the terrain appears and the ceiling above the robot has a normal height (assumed to be accessible by the robot), the robot will grab a supposedly existing handle on the overhead support using its forearm (i.e., $s_{\rm ln\text{-}af}$). On the other hand, when there is no crack on the terrain, we don't allow the use of that action:
\begin{align}\nonumber
\Box(e_{\rm tc{\text -}nc} \Rightarrow p_{\rm PPM} \wedge s_{\rm ln\text{-}af}) \bigwedge \Box( \neg e_{\rm tc{\text -}nc} \Rightarrow \neg p_{\rm PPM} \wedge \neg s_{\rm ln\text{-}af}).
\end{align}

\item ($S_{\rm robot{\text -}3}$) If the environmental action \textsf{\small humanAppear} occurs, i.e., a person appears in front of the robot, the robot comes to a stop using the \textsf{\small legDual} contacts and the arm contacts. On the other hand, when the person disappears, the robot should continue walking from where it stopped before:
\begin{align}\nonumber
\Box \big(e_{\rm ha} \Rightarrow p_{\rm SLM} \wedge (s_{\rm ld\text{-}ah} \vee s_{\rm ld\text{-}af} \vee s_{\rm ld\text{-}an}) \big)
\bigwedge \Box \big( \neg e_{\rm ha} \Rightarrow \neg p_{\rm SLM} \wedge \neg ( s_{\rm ld\text{-}ah} \vee s_{\rm ld\text{-}af} \vee \neg s_{\rm ld\text{-}an} )\big).
\end{align}

\item ($S_{\rm robot{\text -}4}$) If a narrow passage \textsf{\small narrowPassage} appears, the robot will slide on the ground using two feet and no arm contacts. On the other hand, if there is no narrow passage, the robot will not use the sliding mode
\begin{align}\nonumber
\Box(e_{\rm np} \Rightarrow p_{\rm SM} \wedge  s_{\rm ld\text{-}an}) \bigwedge \Box(\neg e_{\rm np} \Rightarrow \neg p_{\rm SM}).
\end{align}

\item ($S_{\rm robot{\text -}5}$) If the environmental action \textsf{\small terrainCrack\text{-}highCeiling} appears, i.e., a crack appears on the terrain and there is a high ceiling, the robot will have to leap over the cracked region using a hopping motion (i.e., $s_{\rm ln\text{-}an}$). On the other hand, when this environmental action does not occur, we do not allow to use that action:
\begin{align}\nonumber
\Box(e_{\rm tc{\text -}hc} \Rightarrow p_{\rm HM} \wedge s_{\rm ln\text{-}an}) \bigwedge \Box( \neg e_{\rm tc{\text -}hc} \Rightarrow \neg p_{\rm HM} \wedge \neg s_{\rm ln\text{-}an}).
\end{align}

%

\end{itemize}
As for the goal proposition of the robot, we require that all locomotion modes and contact actions will occur infinitely often to verify their correctness.

\begin{align}
\varphi^s_{\rm goal} = (\Box \Diamond p_{\rm PIPM}) \wedge (\Box \Diamond s_{\rm ld\text{-}ah}) \wedge (\Box \Diamond s_{\rm ld\text{-}af}) \wedge (\Box \Diamond s_{\rm ld\text{-}an})
\end{align}
where we do not list all the goal propositions of locomotion modes and contact actions. The reason is that the other goal propositions regarding contact actions and locomotion modes are implied by the goal propositions of the environment defined in Eq.~(\ref{eq:environment-goal}).


\subsection{Keyframe specifications}
\label{subsec:LTLDS}
Our phase-space motion planner relies on a keyframe state vector $q = \{p_{\rm contact}, \dot{x}_{\rm apex}\}$ as defined in Section~\ref{subsec:PSPlanning}. In the task planner, the keyframe state is designed to be non-deterministic. We define a discretized phase-space region to choose keyframe states for each walking step using a Riemannian geometry decomposition as shown in Fig.~\ref{fig:Discretization}(b). The keyframe states consist of ordinary and special types (see further below)
\begin{align}\label{eq:complete-keyframe-set}\nonumber
\mathcal{Q} \coloneqq \mathcal{Q}_{\rm ordinary} \cup \mathcal{Q}_{\rm special} = &\{q_{i\text{-}j\text{-}k},\; i \in \mathcal{I}_{\rm ordinary\text{-}behavior},\; \forall (j, k) \in \mathcal{I}_{\rm level}\times\mathcal{I}_{\rm level}\} \\ 
&\cup\{q_{i\text{-}j},\; i \in \mathcal{I}_{\rm special\text{-}behavior},\; \forall j \in \mathcal{I}_{\rm level}\}
\end{align}
where ordinary behaviors are $\mathcal{I}_{\rm ordinary\text{-}behavior} = \{\textsf{\small walk, brachiation}\}$ while special behaviors are $\mathcal{I}_{\rm special\text{-}behavior} = \{\textsf{\small stop, hop, slide}\}$. A apex velocity index $j$ and a step length index $k$ refer to the set $\mathcal{I}_{\rm level} = \{s, m, l\}$ whose elements are three different keyframe ``levels'': $s$ (\textsf{\small Small}), $m$ (\textsf{\small Medium}) and $l$ (\textsf{\small Large}).\footnote{More levels can be introduced at the expense of a combinatorial increase on the total number of the keyframe states.} For instance, $q_{\rm walk\text{-}s\text{-}l}$ represents \textsf{\small walkSmallVelocityLargeStep}, a walking keyframe with a small apex velocity, and a large step length. In our case, the ordinary locomotion behaviors (i.e., \textsf{\small walk} and \textsf{\small brachiation}) comprise 9 keyframe states, respectively while the special locomotion behaviors (i.e., \textsf{\small stop}, \textsf{\small hop} and \textsf{\small slide}) comprise 3 keyframe states, respectively.

Given the environmental actions in Section~\ref{subsec:actionSpec}, the specifications for keyframe states are designed as follows.
\begin{itemize}
\item ($S_{q{\text -}1}$) If the next environmental action is \textsf{\small moderatelyDownward} $e_{\rm md}$, the level for the next keyframe state $q$ remains constant or increases by one level either from step length or apex velocity:
\begin{align}\nonumber
&\Box \big((q_{\rm walk\text{-}s\text{-}s} \wedge \bigcirc e_{\rm md} ) \Rightarrow \bigcirc ( q_{\rm walk\text{-}s\text{-}s} \vee q_{\rm walk\text{-}s\text{-}m} \vee q_{\rm walk\text{-}m\text{-}s}) \big) \\\nonumber
\bigwedge &\Box \big((q_{\rm walk\text{-}s\text{-}m} \wedge \bigcirc e_{\rm md}) \Rightarrow \bigcirc ( q_{\rm walk\text{-}s\text{-}m} \vee q_{\rm walk\text{-}s\text{-}l} \vee q_{\rm walk\text{-}m\text{-}m}) \big) \\\nonumber
& \quad \quad \hdots \\\nonumber
 \bigwedge &\Box \big((q_{\rm walk\text{-}l\text{-}m} \wedge \bigcirc e_{\rm md} ) \Rightarrow \bigcirc ( q_{\rm walk\text{-}l\text{-}m} \vee q_{\rm walk\text{-}l\text{-}l}) \big) \bigwedge \Box \big((q_{\rm walk\text{-}m\text{-}l} \wedge \bigcirc e_{\rm md} ) \Rightarrow \bigcirc ( q_{\rm walk\text{-}m\text{-}l} \vee q_{\rm walk\text{-}l\text{-}l}) \big) \\\nonumber
\bigwedge &\Box \big((q_{\rm walk\text{-}l\text{-}l} \wedge \bigcirc e_{\rm md}) \Rightarrow \bigcirc q_{\rm walk\text{-}l\text{-}l} \big) \bigwedge \Box \Big(\big( (q_{\rm brachiation} \vee q_{\rm stop}) \wedge \bigcirc e_{\rm md}\big )   
\Rightarrow \bigcirc (q_{\rm walk\text{-}s\text{-}m} \vee q_{\rm walk\text{-}m\text{-}m} \vee q_{\rm walk\text{-}l\text{-}m}) \Big),
\end{align}
where, if $q = q_{\rm walk\text{-}s\text{-}s}$, $\bigcirc q$ can be $q_{\rm walk\text{-}s\text{-}s}$ (remaining constant), $q_{\rm walk\text{-}s\text{-}m}$ (step length increases one level) or $q_{\rm walk\text{-}m\text{-}s}$ (apex velocity increases one level). All the other keyframes in ordinary scenarios follow the same pattern. There are three special cases: (i) when $q = q_{\rm walk\text{-}l\text{-}m}$, there are only two choices for $\bigcirc q$, i.e., $q_{\rm walk\text{-}l\text{-}m}$ and $q_{\rm walk\text{-}l\text{-}l}$; (ii) the same situation applies to $q_{\rm walk\text{-}m\text{-}l}$; (iii) when $q = q_{\rm walk\text{-}l\text{-}l}$, the only choice is $\bigcirc (q = q_{\rm walk\text{-}l\text{-}l})$. In emergency cases, we assign $\bigcirc q$ by $q_{\rm walk\text{-}s\text{-}m}$, $q_{\rm walk\text{-}m\text{-}m}$ or $q_{\rm walk\text{-}l\text{-}m}$. 

\item ($S_{q{\text -}2}$) If the next environmental action is \textsf{\small hugelyDownward} $e_{\rm hd}$, the level for the next keyframe state increases by one or two units, either on the step length or on the apex velocity. The only exception is as follows: when the current keyframe is $q = q_{\rm l\text{-}l}$, then the next step is only allowed to choose the keyframe $q_{\rm l\text{-}l}$.
\begin{flalign}
\nonumber
&\Box \big((q_{\rm walk\text{-}s\text{-}s} \wedge \bigcirc e_{\rm hd} ) \Rightarrow \bigcirc ( q_{\rm walk\text{-}m\text{-}s} \vee q_{\rm walk\text{-}s\text{-}m} \vee q_{\rm walk\text{-}l\text{-}s} \vee q_{\rm walk\text{-}s\text{-}l}\vee q_{\rm walk\text{-}m\text{-}m}) \big) &\\\nonumber
\bigwedge &\Box \big((q_{\rm walk\text{-}s\text{-}m} \wedge \bigcirc e_{\rm hd} ) \Rightarrow \bigcirc ( q_{\rm walk\text{-}s\text{-}l} \vee q_{\rm walk\text{-}m\text{-}m} \vee q_{\rm walk\text{-}m\text{-}l}) \big) &\\\nonumber
& \quad \quad \hdots &\\\nonumber
\bigwedge &\Box \Big(\big((q_{\rm walk\text{-}l\text{-}m} \vee q_{\rm walk\text{-}m\text{-}l} \vee q_{\rm walk\text{-}l\text{-}l}) \wedge \bigcirc e_{\rm hd} \big) \Rightarrow \bigcirc q_{\rm walk\text{-}l\text{-}l} \Big) &\\\nonumber 
\bigwedge & \Box \Big(\big( (q_{\rm branchiation} \vee q_{\rm stop}) \wedge \bigcirc e_{\rm hd} \big ) \Rightarrow \bigcirc (q_{\rm walk\text{-}s\text{-}m} \vee q_{\rm walk\text{-}m\text{-}m} \vee q_{\rm walk\text{-}l\text{-}m}) \Big) &
\end{flalign}
where, if $q = q_{\rm walk\text{-}s\text{-}s}$, $\bigcirc q$ increases by (i) one unit level, i.e., $q_{\rm walk\text{-}s\text{-}m}$ and $q_{\rm walk\text{-}m\text{-}s}$, or (ii) two unit levels, i.e., $q_{\rm walk\text{-}m\text{-}m}, q_{\rm walk\text{-}l\text{-}s}$ and $q_{\rm walk\text{-}s\text{-}l}$. Special cases are $q_{\rm walk\text{-}l\text{-}m}, q_{\rm walk\text{-}m\text{-}l}$ and $q_{\rm walk\text{-}l\text{-}l}$ where $q_{\rm walk\text{-}l\text{-}l}$ is the only choice for the next walking step.

\item ($S_{q{\text -}3}$) If there is a crack on the terrain with a normal-height ceiling, i.e., $e_{\rm sc\text{-}nc}$, then the next keyframe state is $q_{\rm brachiation}$ relying on a different set of apex velocities and step lengths than for walking behaviors: 
\begin{align}\nonumber
\Box \big(\bigcirc e_{\rm sc\text{-}nc} \Rightarrow \bigcirc (q_{\rm brachiation\text{-}s} \vee q_{\rm brachiation\text{-}m} \vee q_{\rm brachiation\text{-}l})\big).
\end{align}
\item ($S_{q{\text -}4}$) If there is a crack on the terrain and there is a high ceiling, i.e., $e_{\rm sc\text{-}hc}$, then the keyframe state is $q_{\rm hop}$ relying on a specific apex velocity, regardless of the current $q$:
\begin{align}\nonumber
\Box \big(\bigcirc e_{\rm sc\text{-}hc} \Rightarrow \bigcirc (q_{\rm hop\text{-}s} \vee q_{\rm hop\text{-}m} \vee q_{\rm hop\text{-}l}) \big).
\end{align}

\item ($S_{q{\text -}5}$) If a human appears in front of the robot, i.e., $e_{\rm ha}$, then the next keyframe state is $q_{\rm stop}$ relying on a specific step length, regardless of the current $q$:
\begin{align}\nonumber
\Box \big(\bigcirc e_{\rm ha} \Rightarrow \bigcirc (q_{\rm stop\text{-}s} \vee q_{\rm stop\text{-}m} \vee q_{\rm stop\text{-}l}) \big).
\end{align}
\item ($S_{q{\text -}6}$) If there is a narrow passage, i.e., $e_{\rm np}$, then the next key frame state is $q_{\rm slide}$ relying on a specific apex velocity, regardless of the current $q$:
\begin{align}\nonumber
\Box \big(\bigcirc e_{\rm np} \Rightarrow \bigcirc (q_{\rm slide\text{-}s} \vee q_{\rm slide\text{-}m} \vee q_{\rm slide\text{-}l}) \big).
\end{align}
\end{itemize}
The remaining eight scenarios involving different environment and system action combinations are defined in a similar manner omitted here for brevity. The specifications in ($S_{q{\text -}1}$)-($S_{q{\text -}6}$) and all others belong to $\varphi^q_{\rm trans}$.

From a high-level perspective, the goal of our task planner is to enable the robot to continuously maneuver through constrained environments by repeatedly selecting contact actions among $\mathcal{S}$. To be consistent with the environmental goal specification in Eq.~(\ref{eq:environment-goal-spec}), we enforce the following liveness specification for the keyframe states.
\begin{align}
\varphi^q_{\rm goal} = \bigwedge_{q \in \mathcal{Q}}(\Box \Diamond q)
\end{align}


All the task specifications have been proposed such that $\varphi = \big((\varphi_q \wedge \varphi_e) \Rightarrow \varphi_s \big)$ holds.
%
%
%

\begin{figure}[t]
 \centering
   \includegraphics[width=0.45\linewidth]{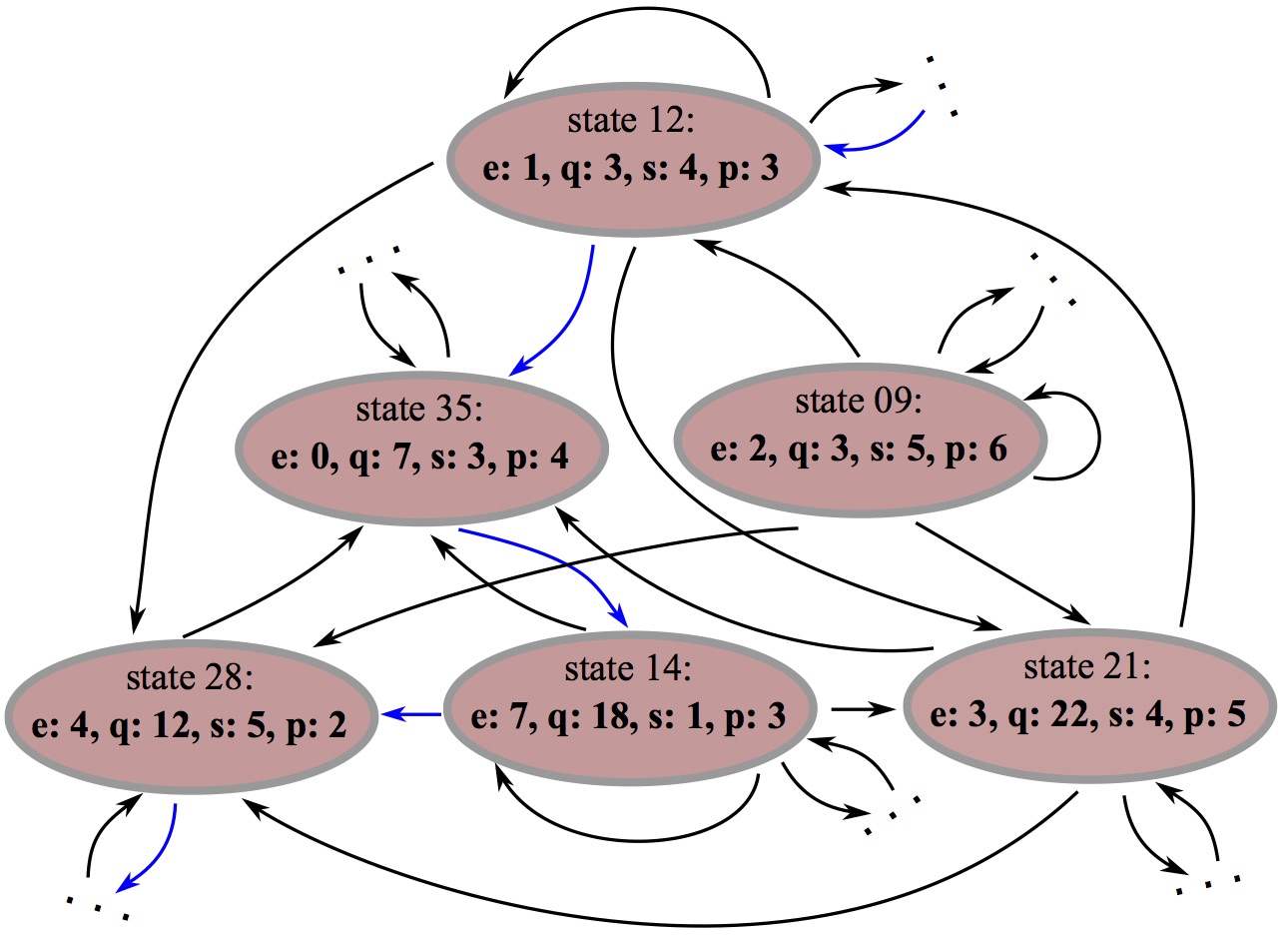}
 \caption{\captionsize A fragment of the synthesized automaton for the WBDL contact planner. Nondeterministic transitions are encoded in this automaton. The blue transitions represent a specific execution. For illustration, we index both the environmental action $\mathcal{E}$ in Eq.~(\ref{eq:envirActions}) and the system action $\mathcal{S}$ in Eq.~(\ref{eq:systemActions}) as $\{0, \ldots, 7\}$ and $\{0, \ldots, 8\}$ in order, respectively. The robot keyframe state $\mathcal{Q}$ is indexed as $\{0, \ldots, 26\}$ in order. For instance, when the automaton state is at number $12$, we encounter environmental action $e = 1$. The winning strategy assigns keyframe state $q = 3$, robot contact action $s = 4$ and locomotion switching mode $p = 3$. This state allows several nondeterministic transitions for the next walking step decision.}
\label{fig:fragmentAutomaton}
\end{figure}

\subsection{Synthesis of a high-level reactive task planner}
\label{sec:high-level-planner}
Here we formulate the high-level locomotion planning problem as a game between the robot and its possibly adversarial environment. Given the task specifications defined above, a reactive control protocol is synthesized such that the controlled legged robot behaviors satisfy all the designed specifications whatever admissible uncontrollable environment behaviors are. 
%
\begin{definition}[\textbf{Game played by the WBDL task planner}]\label{def:game}
A game for the whole-body dynamic locomotion task planner is a tuple 
\begin{align}\nonumber
\mathcal{G} \coloneqq \langle \mathcal{V}, \mathcal{X}, \mathcal{Y}, \theta_i, \theta_o, \Psi_i, \Psi_o, \phi_{\rm win} \rangle
\end{align}
with the following elements
\begin{itemize}
    \item $\mathcal{X} \coloneqq \mathcal{E}$ is a set of input variables for player 1; 
    \item $\mathcal{Y} \coloneqq \mathcal{Q} \times \mathcal{S} \times \mathcal{P}$ is a set of output variables for player 2; 
    \item $\mathcal{V} = \mathcal{X} \times \mathcal{Y}$ is a finite set of proposition state variables over finite domains in the game;
    \item $\theta_i$ and $\theta_o$ are atomic propositions characterizing initial states of the input and output variables, respectively;
    \item $\Psi_i(\mathcal{V}, \mathcal{X}')$ and $\Psi_o(\mathcal{V}, \mathcal{X}', \mathcal{Y}')$ are the transition relations for the input and output variables for next steps, respectively;
    \item $\phi_{\rm win}$ is the winning condition given by an LTL formula.
\end{itemize}

\end{definition}
\noindent A winning strategy for the task planner represented by the pair $(\mathcal{TS}_{\rm prod}, \varphi)$ is defined as a partial function $(\gamma_0 \gamma_1 \cdots \gamma_{i-1}, e_i) \mapsto (q_i, s_i, p_i)$, where a keyframe state $q_i$, a contact action $s_i$, and a switching mode $p_i$ are chosen according to the state sequence history and the current environmental action in order to satisfy the assume-guarantee form in Eq.~(\ref{eq:gr1}). All the specifications are satisfied whatever admissible yet uncontrollable environmental actions are.
\begin{proposition}[\textbf{Existence of a winning WBDL strategy}]\label{prop:winningstrategy}
A winning WBDL strategy $\mathcal{A}_{\rm WBDL}$ exists for the game $\mathcal{G}$ in Definition~\ref{def:game} if and only if $(\mathcal{TS}_{\rm prod}, \varphi)$ is \textit{realizable}. 
\end{proposition}

Fig.~\ref{fig:fragmentAutomaton} shows an automaton fragment of the WBDL contact planner $\mathcal{A}_{\rm WBDL}$. Self-transition exists in \textsf{\small moderatelyUpward} states (e.g., state $09$) and \textsf{\small moderatelyDownward} states (e.g., states $12$ and $14$) while \textsf{\small hugelyDownward} states (e.g., state $35$) do not have a self-transition according to proposition ($S_{es{\text -}1}$). There is no transition between states $09$ and $14$ due to infeasible keyframe state transition. States $21$ and $28$ in red nodes represent \textsf{\small humanAppear} and \textsf{\small terrainCrack} events, respectively.
\begin{remark}
Non-deterministic transitions exist in the synthesized automaton as follows:
(i) environmental actions are non-deterministic. (ii) given an environmental action, several non-deterministic keyframe states can be chosen. (iii) even when both an environmental action and a keyframe state are given, non-deterministic system contact actions exist for certain transitions. This non-deterministic transitions allow self-transitions. In this case, we can guarantee the robot to make progress (i.e., maneuvering forward) due to the properties of locomotion keyframe states.
\end{remark}


The keyframe specifications in this section purely reason about logic-level decisions and have no knowledge of underlying locomotion dynamics. However, the locomotion dynamics, especially those affected by external disturbances or model uncertainties, often result in the desired keyframe transitions being unrealizable. As such, we need to propose keyframe transitions with robustness margins and synthesize a reachability based controller to determine realizable keyframe transitions by the low-level locomotion dynamical system as proposed in the next section. 


 \section{Robust Reachability Control of Hybrid Locomotion Systems}
\label{sec:Abstraction}
When we model robot dynamics and estimate physical environments, uncertainty is ubiquitous due to sensor noise, model inaccuracy, external disturbance, sudden environmental changes, contact surface geometry uncertainty, and so on. As a result, commands from the symbolic task planner are potentially not achievable by the low-level motion planner. Additionally, a mismatch between the high-level discrete and low-level continuous planners is usually caused by the abstraction techniques applied on the underlying continuous systems.
To handle these difficulties, we define a robust finite transition system and compute its keyframe transitions via synthesizing reachability controllers for every single  walking step. In order to use phase-space locomotion manifolds to define robustness margin sets, a phase-space mapping needs to be defined between the Euclidean and Riemmanian spaces to evaluate whether a phase-space state is in the robustness margin or not.

\subsection{Phase-space Euclidean-to-Riemmanian mapping}
\label{subsection:mapping}
We first consider a specific locomotion process, e.g., the prismatic inverted pendulum model (PIPM) (see Section~\ref{subsec:low-level-imp} for more details) in order to establish a Euclidean-to-Riemmanian mapping in the phase space. Our previous study derives closed-form solutions of phase-space tangent and cotangent manifolds for this process [\cite{zhao2017robust}] as follows.
\begin{proposition}[\textbf{PIPM phase-space tangent manifold}]\label{prop:PSMTangent}
Given the PIPM of Eq. (\ref{eq:PIPM}) with initial conditions $(x_0, \dot{x}_0) = (x_{\rm foot}, \dot{x}_{\rm apex})$ and known foot placement $x_{\rm foot}$, the phase-space tangent manifold is characterized by the states $(x, \dot{x}, x_{\rm foot}, \dot{x}_{\rm apex})$ such that
\begin{align}\label{eq:simplifiedPSM}
\sigma(x, \dot{x}, x_{\rm foot}, \dot{x}_{\rm apex}) = \dfrac{\dot{x}^2_{\rm apex}}{\omega_{\rm PIPM}^2} \big(\dot{x}^2 - \dot{x}^2_{\rm apex} 
                                            - \omega_{\rm PIPM}^2(x - x_{\rm foot})^2\big),
\end{align}
where $\sigma = 0$ represents the nominal phase-space manifold. When $\sigma \neq 0$, it represents the Riemannian distance to the nominal phase-space manifold.
\end{proposition}
The tangent manifold can be used to measure deviations from the nominal locomotion trajectory in the phase-space. We use this manifold to quantify the width of a phase-space robustness margin.
\begin{figure}[t]
 \centering
   \includegraphics[width=0.75\linewidth]{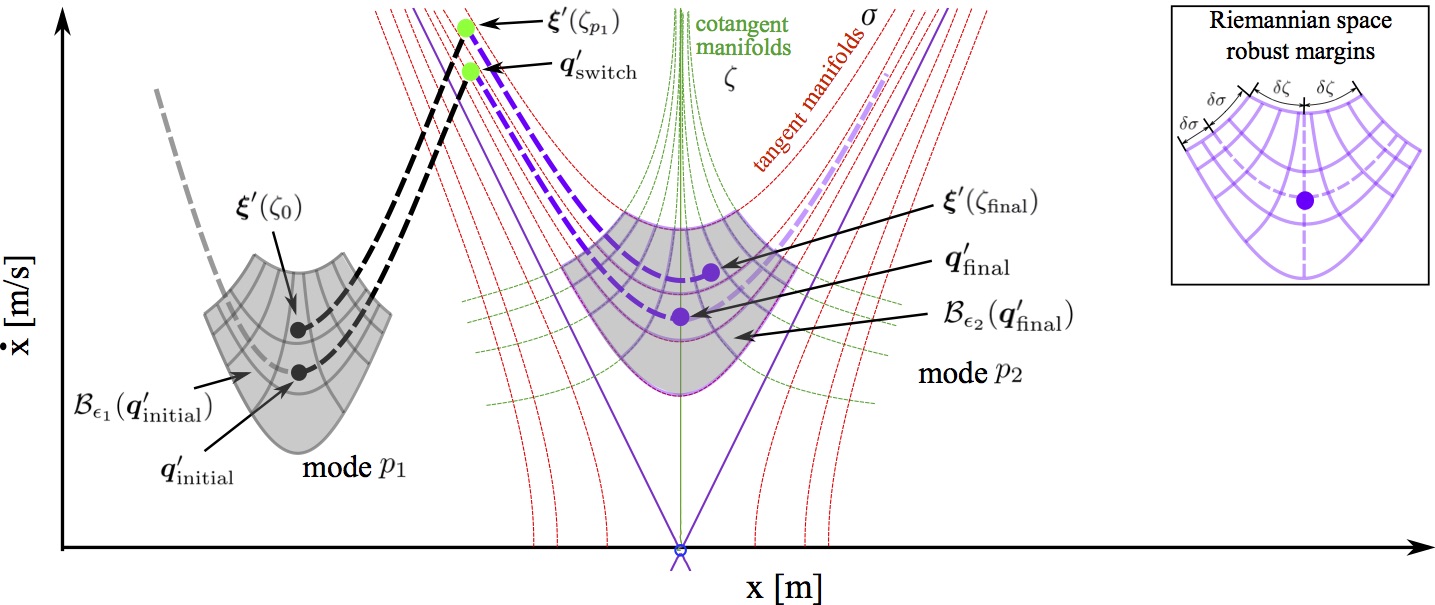}
 \caption{\captionsize Phase-space abstraction via locomotion-manifold-based partition. This figure shows a keyframe state transition process with robustness margins.
 Compared to the conventional grid-based partition in Euclidean phase space of Fig~\ref{fig:Discretization}(a), this partition complies with locomotion dynamics, further enabling us to define robustness margins based on closed-form locomotion manifolds.
}
\label{fig:phase-space-discretization}
\end{figure}
\begin{proposition}[\textbf{PIPM phase-space cotangent manifold}]\label{prop:PSMCotangent}
Let $\zeta_0$ be a nonnegative scaling value representing the initial phase of a cotangent manifold. Given the PIPM of Eq.~(\ref{eq:PIPM}) and a specific initial state $(x_0, \dot x_0)$ different from the keyframe $(x_{\rm foot}, \dot x_{\rm apex})$, the cotangent manifold is characterized by the states $(x, \dot{x}, x_0, \dot{x}_0)$ such that
\begin{align}\label{eq:zeta_manifold}
\zeta(x, \dot{x}, x_0, \dot{x}_0)  = \zeta_0(\dfrac{\dot{x}}{\dot{x}_0})^{\omega_{\rm PIPM}^2} \dfrac{x - x_{\rm foot}}{x_0 - x_{\rm foot}},
\end{align}
where $\zeta_0$ is chosen as the phase progression value at the keyframe state in this study.
\end{proposition}
This cotangent manifold represents the arc length along the tangent manifold $\sigma$ in Eq.~{(\ref{eq:simplifiedPSM})}. We use this cotangent manifold to quantify the length of a phase-space robustness margin. Detailed derivations of these two closed-form solutions above, i.e., $\sigma(x, \dot{x}, x_{\rm foot}, \dot{x}_{\rm apex}) = 0$ and $\zeta(x, \dot{x}, x_0, \dot{x}_0) = 0$, are provided in [\cite{zhao2017robust}]. A similar analysis can be performed for other locomotion models as described in Section~\ref{subsec:low-level-imp} (see the propositions in Appendix~\ref{appen:PSManifold} for other locomotion modes). Given these analytical solutions, we define a mapping between the Euclidean and Riemmanian spaces as
\begin{align}\label{eq:manifold-mapping}
\begin{pmatrix}
\zeta  \\
\sigma
\end{pmatrix}
= 
\mathcal{Z}_p (\boldsymbol{\xi})
=
\begin{pmatrix}
\mathcal{Z}_{p,\zeta}(x, \dot x)  \\
\mathcal{Z}_{p,\sigma}(x, \dot x)
\end{pmatrix}
\end{align}
where $\mathcal{Z}_p (\boldsymbol{\xi})$ is a nonlinear mapping of the CoM state $(x, \dot x)$ to the Riemannian space states and obtained by using the phase space manifold of the $p^{\rm th}$ locomotion mode. This mapping will be used for the robust finite transition system definition in order to quantify the location of the phase-space state in the Riemmanian space.

\subsection{Robust finite transition system for one walking step}
We now focus on a case of the one-walking-step locomotion process as defined in Def.~\ref{def:one-walking-step}. 
As illustrated in Fig.~\ref{fig:Discretization}(b), the discrete task planner uses a Riemannian discretization of the local state space, which is defined by an abstraction map $\mathcal{M}_{\rm Riem}: \Xi \rightarrow \mathcal{Q}$ such that for all $(\boldsymbol{\xi}, \boldsymbol{q}) \in \Xi \times \mathcal{Q}$,
\begin{align}\label{eq:keyframeMapping}
    |\boldsymbol{\xi} - \boldsymbol{q}| \preceq \nu \; \implies \; \boldsymbol{q} = \mathcal{M}_{\rm Riem}(\boldsymbol{\xi}),
\end{align}
where $\nu$ is the granularity of the discretization\footnote{In this section, we use a bold symbol $\boldsymbol{q}$ to represent a keyframe state since it is a multi-dimensional state vector. In the task planner, the keyframe state is represented by a non-bold symbol $q$ due to its pure discrete property.}. The operators $|\cdot|$ and $\preceq$ above represent vectorized absolute values and element-wise inequality, respectively.\footnote{For two given $n$-vectors $x$ and $y$, we have $|x| = (|x_1|, |x_2|, \ldots, |x_n|)$ and $|x| \preceq |y| \Leftrightarrow |x_i| \leq |y_i|, \forall i \in \{1, \ldots, n\}$.}

To guarantee that the motion planner yields feasible phase-space plans robust to disturbances, such as state measurement errors and disturbances in the dynamics, we introduce $\epsilon_1$ and $\epsilon_2$ as the bounds of initial and final robustness margins in the one-walking-step transition system. Namely, we not only consider the nominal initial and final keyframe states $\boldsymbol{q}_{\rm initial}$ and $\boldsymbol{q}_{\rm final}$ assigned by the task planner, but also neighbourhood keyframe cells overlapping the $\epsilon$-neighbourhood of nominal keyframe states. 
%
\begin{figure}[t]
 \centering
   \includegraphics[width=0.85\linewidth]{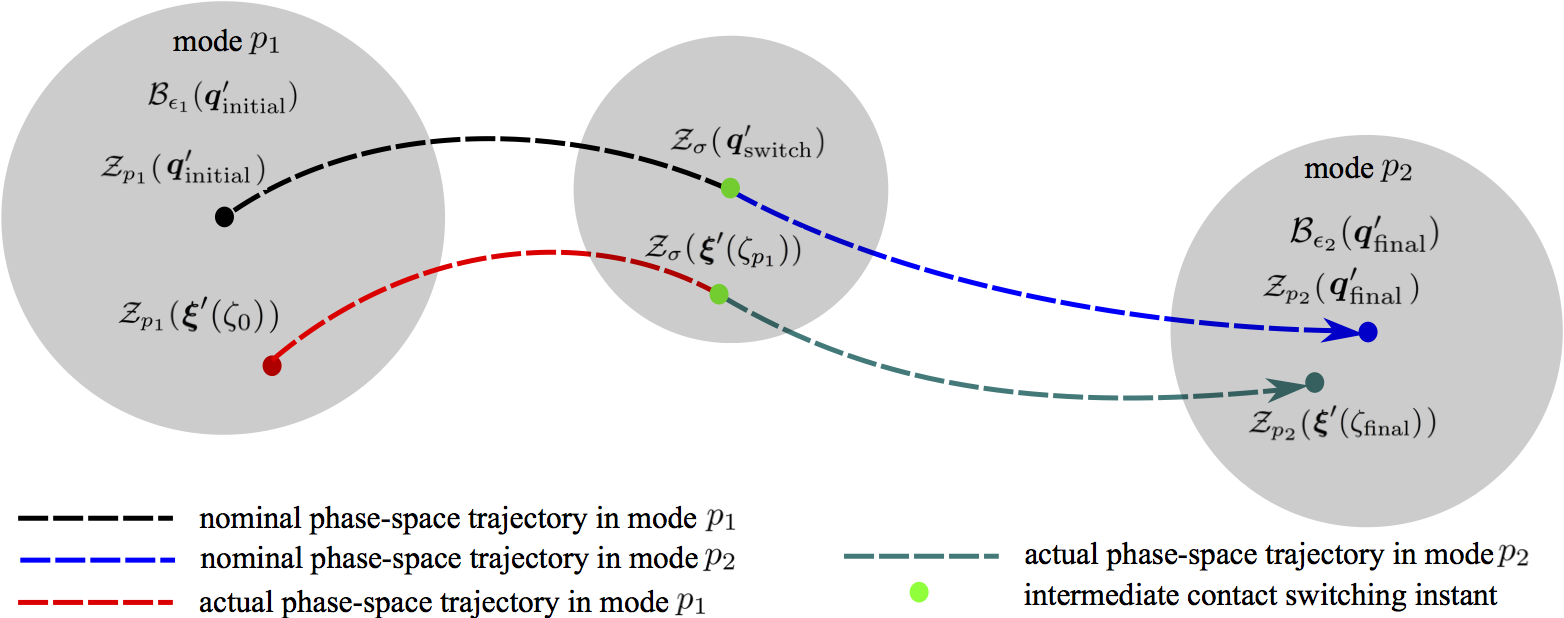}
 \caption{\captionsize Keyframe state reachability with robustness margins for one walking step. Due to measurement error or external disturbance, the initial state $\boldsymbol{\xi}'(\zeta_0)$ may deviate from the desired keyframe state $\boldsymbol{q}'_{\rm initial}$. A robustness region $\mathcal{B}_{\epsilon_1}(\boldsymbol{q}'_{\rm initial})$ is defined to bound the allowable state deviations. The actual and desired states evolve according to their respective system dynamics in locomotion modes $p_1$ and $p_2$, respectively. The state switches from mode $p_1$ to mode $p_2$ at the phase instant $\zeta_{p_1}$. The bound of distance between the nominal intermediate keyframe $\boldsymbol{q}'_{\rm switch}$ and $\boldsymbol{\xi}'(\zeta_{p_1})$ is shown in Eq.~(\ref{eq:intermediateCondition_1}). Finally, these two states reach $\boldsymbol{\xi}'(\zeta_{\rm final})$ and $\boldsymbol{q}'_{\rm final}$, respectively. To compute the keyframe transition, we require that   $\mathcal{Z}_{p_2}(\boldsymbol{\xi}'(\zeta_{\rm final}))$ should be $\epsilon_2$-close to $\mathcal{Z}_{p_2}(\boldsymbol{q}_{\rm final}')$, i.e., being in the margin $\mathcal{B}_{\epsilon_2}(\boldsymbol{q}'_{\rm final})$. More details of the definitions of robustness margin set are in Def.~\ref{def:robustSetSimp}. Note that, since the robustness margin set $\mathcal{B}$ is defined in the Riemannian space, the mapping $\mathcal{Z}_p$ is applied on the states $\boldsymbol{q}$ and $\boldsymbol{\xi}'$ in the figure symbols for consistency.}
\label{fig:RobustMargin}
\end{figure}

\begin{definition}[\textbf{Robustness margin sets}]\label{def:finalRobustSet}
The initial and final robustness margin sets around the nominal keyframe states $\boldsymbol{q}_{\rm initial}, \boldsymbol{q}_{\rm final} \in \mathcal{Q}$
are defined as
\begin{align}\nonumber
\tilde{\mathcal{B}}_{\epsilon_1}(\boldsymbol{q}_{\rm initial}) &\coloneqq \big\{ \boldsymbol{\xi}(\zeta_0) \; | \; |\mathcal{Z}_{p_1}(\boldsymbol{\xi}(\zeta_0)) - \mathcal{Z}_{p_1}(\boldsymbol{q}_{\rm initial})| \preceq \epsilon_1 \big\}. \\\nonumber
\tilde{\mathcal{B}}_{\epsilon_2}(\boldsymbol{q}_{\rm final}) &\coloneqq \big\{ \boldsymbol{\xi}(\zeta_{\rm final}) \; | \; |\mathcal{Z}_{p_2}(\boldsymbol{\xi}(\zeta_{\rm final})) - \mathcal{Z}_{p_2}(\boldsymbol{q}_{\rm final})| \preceq \epsilon_2 \big\}.
\end{align}
where $\epsilon_1, \epsilon_2 \in \mathbb{R}^2$ represent the bounds of $\tilde{\mathcal{B}}_{\epsilon_1}(\boldsymbol{q}_{\rm initial})$ and $\tilde{\mathcal{B}}_{\epsilon_2}(\boldsymbol{q}_{\rm final})$, respectively. $p_1$ and $p_2$ denote the locomotion modes before and after a contact switch, respectively.
\end{definition}

The robustness margins $\epsilon_1$ and $\epsilon_2$ in Def.~{\ref{def:finalRobustSet}} are defined in the Riemannian space. A mapping $\mathcal{Z}$ is applied on the Euclidean states $\boldsymbol{\xi}$ and $\boldsymbol{q}$ to convert them to the Riemannian space. We design $\epsilon_1\succ \nu$ and $\epsilon_2\succ\nu$ such that the robustness margins are larger than the discretized cell. To provide different robust margins, we allow for non-uniform sets, i.e., non-identical values for $(\epsilon_1, \epsilon_2)$. This non-uniform set design makes the size of the total number of allowable keyframe transitions more manageable.


Now we describe how to simplify the robustness margin sets based on the closed-form phase-space manifolds defined in Section~\ref{subsection:mapping}.
%
\begin{definition}[\textbf{Phase-space robustness margin sets}]\label{def:robustSetSimp}
Given closed-form locomotion phase-space manifolds from Propositions~\ref{prop:PSMTangent} and~\ref{prop:PSMCotangent}, the initial and final robustness margin sets are simplified to
\begin{gather}
\mathcal{B}_{\epsilon_1}(\boldsymbol{q}_{\rm initial}) \coloneqq \big\{ \boldsymbol{\xi}(\zeta, \sigma)\; | \; \zeta \in [\zeta_0 - \delta \zeta_{\epsilon_1}, \zeta_0 + \delta \zeta_{\epsilon_1}], \sigma \in [-\delta\sigma_{\epsilon_1}, \delta\sigma_{\epsilon_1}] \big\}, \\
\mathcal{B}_{\epsilon_2}(\boldsymbol{q}_{\rm final}) \coloneqq \big\{ \boldsymbol{\xi}(\zeta, \sigma) \;| \; \zeta \in [\zeta_{\rm final} - \delta \zeta_{\epsilon_2}, \zeta_{\rm final} + \delta \zeta_{\epsilon_2}], \sigma \in [-\delta \sigma_{\epsilon_2}, \delta \sigma_{\epsilon_2}] \big\},
\end{gather}
%
where $\boldsymbol{\xi} = (x, \dot{x})$, the initial state $\boldsymbol{q}_{\rm initial} = \boldsymbol{\xi}(\zeta_0, 0)$ and the final state $\boldsymbol{q}_{\rm final} = \boldsymbol{\xi}(\zeta_{\rm final}, 0)$, $\epsilon_1 = [\delta \zeta_{\epsilon_1}, \delta \sigma_{\epsilon_1}]$ and $\epsilon_2 =  [\delta \zeta_{\epsilon_2}, \delta \sigma_{\epsilon_2}]$ quantify the uncertainty bounds of $\mathcal{B}_{\epsilon_1}(\boldsymbol{q}_{\rm initial})$ and $\mathcal{B}_{\epsilon_2}(\boldsymbol{q}_{\rm final})$, respectively.
\end{definition}
These two pairs of bounds represent Riemannian distances in phase-space, as shown in the upper right miniature subfigure in Fig.~\ref{fig:phase-space-discretization}.
A locomotion-manifold-based partition is illustrated in Fig.~\ref{fig:Discretization}(b). The proposed robust finite transition system will use this partition to design the robustness margin around keyframe states as shown in Fig.~\ref{fig:phase-space-discretization}. A merit of our analysis is that this partition is consistent with the vector field of the locomotion dynamics. Additionally, this partition simplifies mathematical descriptions of robustness margin sets.
\begin{definition}[\textbf{Robust finite transition system for one walking step}]\label{def:robust-finite-transition}
Given two triples composed of nominal keyframe states, locomotion modes, and system contact actions $(\boldsymbol{q}_{\rm initial},p_1, s_1)$ and $(\boldsymbol{q}_{\rm final}, p_2, s_2)$\footnote{In the robust finite transition system layer, we use a bold symbol $\boldsymbol{q}$ to represent a keyframe state since it is a multi-dimensional state vector in phase-space. In the task planner, the keyframe state is represented by a non-bold $q$ due to the discrete domain reasoning.}, a finite transition subsystem $\mathcal{TS}_{\rm OWS}$ with robustness margins $\epsilon_1$ and $\epsilon_2$ for one walking step (OWS) is defined as a tuple
\begin{align}\label{eq:TS_subs}
\mathcal{TS}_{\rm OWS}\coloneqq (\mathcal{Q}_{\rm OWS}, \mathcal{I}_{\rm OWS}, \mathcal{P}_{\rm OWS}, \mathcal{S}_{\rm OWS}, \mathcal{A}_{\rm OWS}, \mathcal{T}_{\rm OWS})
\end{align}
with the following elements
\begin{itemize}
\item $\mathcal{Q}_{\rm OWS}$ is a set of keyframe states determined by the nominal keyframe pair $(\boldsymbol{q}_{\rm initial}, \boldsymbol{q}_{\rm final})$ and robustness margins $(\epsilon_1, \epsilon_2)$. $\mathcal{Q}_{\rm OWS} \subseteq \mathcal{Q}$, where $\mathcal{Q}$ is the set of all the allowable keyframe states defined in Eq.~(\ref{eq:complete-keyframe-set}). $\mathcal{Q}_{\rm OWS} = \mathcal{Q}_{p_1, {\rm OWS}} \cup \mathcal{Q}_{p_2, {\rm OWS}}$,
where $\mathcal{Q}_{p_1, {\rm OWS}}$ and $\mathcal{Q}_{p_2, {\rm OWS}}$ are defined as 
%
\begin{align}\label{eq:Q_p1_OWS}
    & \mathcal{Q}_{p_1, {\rm OWS}} = \{ \boldsymbol{q}'_{\rm initial} \in\mathcal{Q} \; \vert\;  \mathcal{M}_{\rm Riem}^{-1}(\boldsymbol{q}'_{\rm initial})\cap\mathcal{B}_{\epsilon_1}(\boldsymbol{q}_{\rm initial}) \neq \emptyset \}, \\\label{eq:Q_p2_OWS}
    & \mathcal{Q}_{p_2, {\rm OWS}} = \{ \boldsymbol{q}'_{\rm final} \in\mathcal{Q} \; \vert\;  \mathcal{M}_{\rm Riem}^{-1}(\boldsymbol{q}'_{\rm final})\cap\mathcal{B}_{\epsilon_2}(\boldsymbol{q}_{\rm final}) \neq \emptyset \}.
\end{align}
\item $\mathcal{I}_{\rm OWS} = \mathcal{Q}_{p_1, {\rm OWS}}$ is a set of initial states.
\item $\mathcal{P}_{\rm OWS}=\{p_1,p_2\}$ is a pair of locomotion modes for one walking step.
\item $\mathcal{S}_{\rm OWS}=\{s_1,s_2\}$ is a pair of contact actions for one walking step.
\item $\mathcal{A}_{\rm OWS}=\{\mathcal{A}_{p_1},\mathcal{A}_{p_2}\}$ with $\mathcal{A}_{p_1}=\bigcup_{\zeta_0\leq\zeta \leq \zeta_{p_1}} \mathcal{U}_{p_1}^{[\zeta_0,\zeta]}$ and $\mathcal{A}_{p_2}=\bigcup_{\zeta_{p_1}\leq\zeta \leq \zeta_{\rm final}} \mathcal{U}_{p_2}^{[\zeta_{p_1}, \zeta]}$,\footnote{It is worth noting that $\mathcal{A}_{p_1}$ and $\mathcal{A}_{p_2}$ represent a sequence of discrete control inputs, respectively, since this abstraction is designed for two consecutive walking steps involving a sequence of inter-sampling time steps.} where $\zeta_{p_1}$ represents the phase instant when the contact switch occurs.

\item $\left( (\boldsymbol{q}'_{\rm initial}, p_1, s_1),\boldsymbol{a},(\boldsymbol{q}'_{\rm final}, p_2, s_2)\right)\in\mathcal{T}_{\rm OWS}$ (i.e., $(\boldsymbol{q}'_{\rm initial}, p_1, s_1) \xrightarrow{\boldsymbol{a}} (\boldsymbol{q}'_{\rm final}, p_2, s_2)$) for $\boldsymbol{q}'_{\rm initial}\in\mathcal{Q}_{p_1, {\rm OWS}},\boldsymbol{q}'_{\rm final}\in\mathcal{Q}_{p_2, {\rm OWS}}$ if there exists a control sequence $\boldsymbol{a} = \{\boldsymbol{u}_{p_1}(\zeta_0), \ldots, \boldsymbol{u}_{p_1}(\zeta_{p_1}), \ldots, \boldsymbol{u}_{p_2}(\zeta_{p_2}), \ldots, \boldsymbol{u}_{p_2}(\zeta_{\rm final}) \} \in \mathcal{A}_{p_1} \cup \mathcal{A}_{p_2}$ for all bounded external disturbances $d: [\zeta_0, \zeta_{\rm final}] \rightarrow D_{\rm OWS} \subseteq \mathbb{R}^d$ such that the resulting solution $\boldsymbol{\xi}': [\zeta_0, \zeta_{\rm final}] \rightarrow \mathbb{R}^n$ follows,
\begin{align}\label{eq:phase1_1}
\dot{\boldsymbol{\xi}'}(\zeta) = f_{p_1}\big(\boldsymbol{\xi}'(\zeta),  \boldsymbol{u}_{p_1}(\zeta), d(\zeta)\big), \forall \zeta \in [\zeta_0, \zeta_{p_1}],\\\label{eq:phase2_1}
\dot{\boldsymbol{\xi}'}(\zeta) = f_{p_2}\big(\boldsymbol{\xi}'(\zeta),  \boldsymbol{u}_{p_2}(\zeta), d(\zeta)\big), \forall \zeta \in [\zeta_{p_1}, \zeta_{\rm final}],
\end{align}
satisfying
\begin{align}\label{eq:initial-bounds}
&|\mathcal{Z}_{p_1}(\boldsymbol{\xi}(\zeta_0)) - \mathcal{Z}_{p_1}(\boldsymbol{q}'_{\rm initial})| \preceq \epsilon_1, \boldsymbol{q}'_{\rm initial}\in\mathcal{M}_{\rm Riem}(\boldsymbol{\xi}(\zeta_0)),\\\label{eq:final-bounds}
&| \mathcal{Z}_{p_2}(\boldsymbol{\xi}(\zeta_{\rm final})) - \mathcal{Z}_{p_2}(\boldsymbol{q}'_{\rm final})| \preceq \epsilon_2, \boldsymbol{q}'_{\rm final}\in\mathcal{M}_{\rm Riem}(\boldsymbol{\xi}(\zeta_{\rm final})),\\\label{eq:intermediateCondition_1}
&|\mathcal{Z}_{p_1,\sigma}(\boldsymbol{\xi}'(\zeta_{p_1})) - \mathcal{Z}_{p_1,\sigma}(\boldsymbol{q}'_{\rm switch})| \leq \delta \sigma_{\epsilon_1}, |\mathcal{Z}_{p_2,\sigma}(\boldsymbol{\xi}'(\zeta_{p_1})) - \mathcal{Z}_{p_2,\sigma}(\boldsymbol{q}'_{\rm switch})| \leq \delta \sigma_{\epsilon_2},
\end{align}
where
$\mathcal{Z}(\cdot)$ is the two-dimensional Euclidean-to-Riemannian mapping introduced in Section~\ref{subsection:mapping}. The system vector fields $f_{p_1}$ and $f_{p_2}$ are jointly determined by the locomotion mode set $\mathcal{P}_{\rm OWS}$ and the contact action set $\mathcal{S}_{\rm OWS}$, respectively.





\end{itemize}
\end{definition}

The mapping function $\mathcal{Z}$ has two dimensions in the phase-space: tangent $\sigma$ and cotangent $\zeta$ manifolds as defined in Eq.~(\ref{eq:manifold-mapping}). Initial and final bound conditions are represented by Eqs.~(\ref{eq:initial-bounds}) and ~(\ref{eq:final-bounds}), respectively. Eq.~(\ref{eq:intermediateCondition_1}) essentially defines an intermediate set where the mode switch takes place, and determines the bound for the switching instant $\zeta_{p_1}$. The inequalities in Eqs.~(\ref{eq:initial-bounds}) and~(\ref{eq:final-bounds}) are element-wise.

A conceptual illustration of this transition computation is shown in Fig.~\ref{fig:RobustMargin}. Using the robustness margins, we construct the transition set $\mathcal{T}_{\rm OWS}$ of the robust finite transition subsystem $\mathcal{TS}_{\rm OWS}$, as defined in Def.~\ref{def:robust-finite-transition}, by adding all the feasible transitions $\{(\boldsymbol{q}'_{\rm initial}, p_1, s_1) \xrightarrow{\boldsymbol{a}} (\boldsymbol{q}'_{\rm final}, p_2, s_2)\}$ where $\mathcal{Z}_{p_2}(\boldsymbol{\xi}'(\zeta_{\rm final}))$ is within the $\epsilon_2$-distance to the targeted state $\mathcal{Z}_{p_2}(\boldsymbol{q}'_{\rm final})$.

The construction of $\mathcal{Q}_{\rm OWS}$ is shown in Fig.~\ref{fig:DiscreteStateSet}. The keyframe transitions in the robust finite transition subsystem $\mathcal{TS}_{\rm OWS}$
can be computed as follows: for all $ \boldsymbol{q}'_{\rm initial} \in \mathcal{Q}_{p_1, {\rm OWS}}$ and $\boldsymbol{a} \in \mathcal{A}_{\rm OWS}$, if there exists $\boldsymbol{q}'_{\rm final} \in \mathcal{Q}_{p_2, {\rm OWS}} \cap \mathcal{B}_{\epsilon_2}(\boldsymbol{\xi})$, then we add $\{\boldsymbol{q}'_{\rm initial} \xrightarrow{\boldsymbol{a}} \boldsymbol{q}'_{\rm final}\}$ to the transition set $\mathcal{T}_{\rm OWS}$.

\begin{figure}[t]
 \centering
   \includegraphics[width=0.65\linewidth]{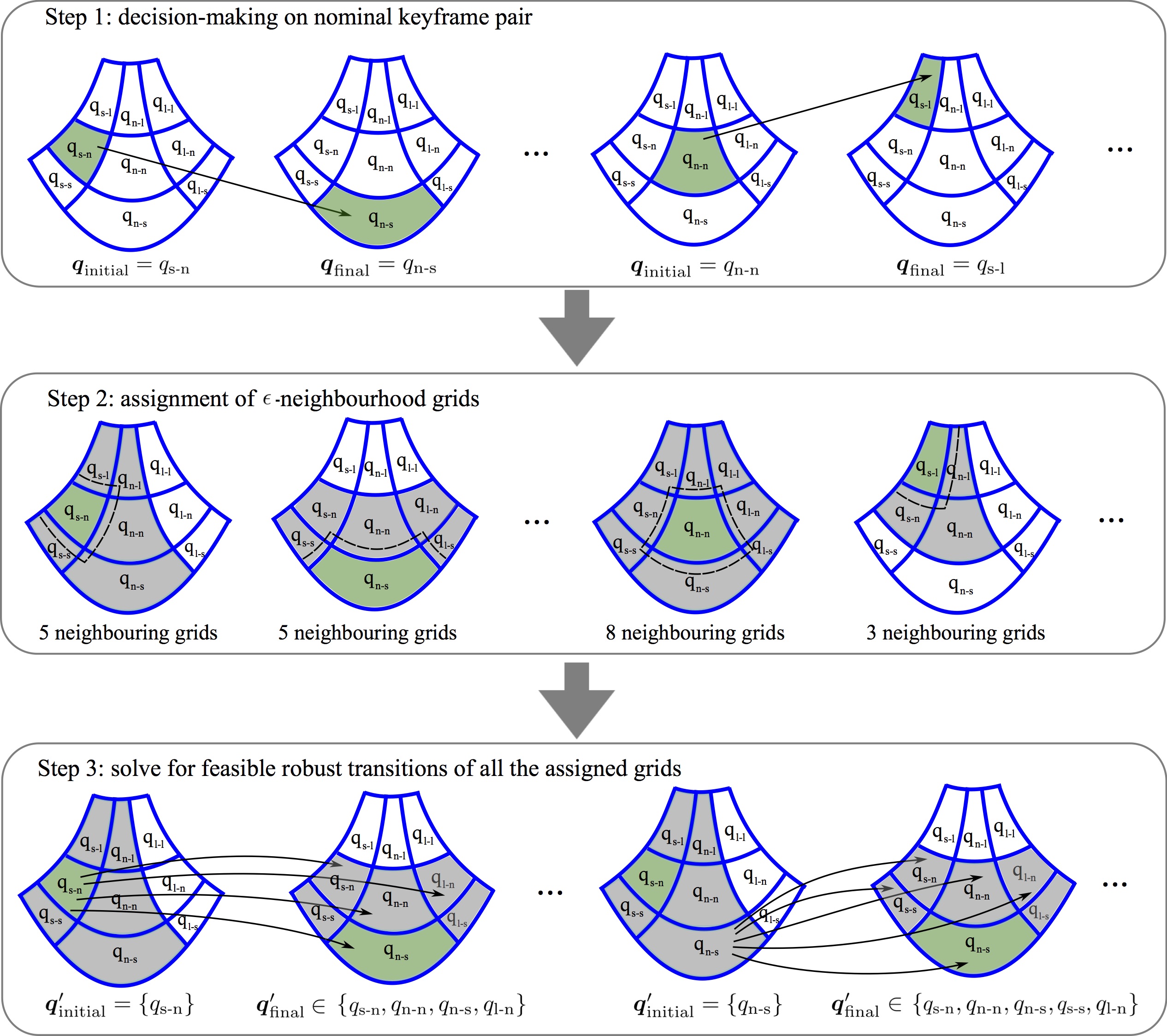}
 \caption{\captionsize Sequential procedure of designing feasible robust keyframe transitions. Step 1 determines the nominal keyframe state pair from the symbolic task planner. In Step 2, we design discrete state set $\mathcal{Q}_{\rm OWS}$ in the robust finite transition system. Four cases are shown for illustration. The green cell represents the nominal keyframe state determined from the task planner while its surrounding gray cells represent other allowable discrete states in $\mathcal{Q}_{\rm OWS}$. Finally, all feasible robust transitions are determined in Step 3.
}
\label{fig:DiscreteStateSet}
\end{figure}

\begin{algorithm}[t]
\caption{One-walking-step robust finite transition subsystem $\mathcal{TS}_{\rm OWS}$ with robustness margins $(\epsilon_1, \epsilon_2)$}\label{al:abstraction}
\begin{algorithmic}[1]
\STATE \textbf{Input:} $\epsilon_1, \epsilon_2, \boldsymbol{q}_{\rm initial}, \boldsymbol{q}_{\rm final}, p_1, p_2, s_1, s_2$
\STATE Define $\mathcal{A}_{p_1}, \mathcal{A}_{p_2}$ as finite subsets of $\bigcup_{\zeta_0\leq\zeta \leq \zeta_{p_1}} \mathcal{U}_{p_1}^{[\zeta_0,\zeta]}$
and $\bigcup_{\zeta_{p_1}\leq\zeta \leq \zeta_{\rm final}} \mathcal{U}_{p_2}^{[\zeta_{p_1}, \zeta].}$
\STATE Define the discrete state of keyframe $\mathcal{Q}_{\rm OWS}$ according to Eqs.~(\ref{eq:Q_p1_OWS}) and (\ref{eq:Q_p2_OWS}) of Def.~\ref{def:robust-finite-transition} and Fig.~\ref{fig:DiscreteStateSet}.
\STATE Initialize transition set $\mathcal{T}_{\rm OWS} \gets (\mathcal{Q}_{p_1, {\rm OWS}} \times \{p_1\}) \times \mathcal{A}_{\rm OWS} \times (\mathcal{Q}_{p_2, {\rm OWS}} \times \{p_2\})$ \COMMENT{initialize all possible transitions}
\FOR {$\boldsymbol{q}'_{\rm initial} \in \mathcal{Q}_{p_1, {\rm OWS}}$}
\FOR{$\boldsymbol{q}'_{\rm final} \in \mathcal{Q}_{p_2, {\rm OWS}}$}
\STATE Construct an inter-sampling finite abstraction $\mathcal{TS}_{\rm OWS, INT}$
\STATE \textsf{\small isReachable} $\leftarrow$ ReachabilityControl($\mathcal{B}_{\epsilon_1}(\boldsymbol{q}'_{\rm initial}), \mathcal{B}_{\epsilon_2}(\boldsymbol{q}'_{\rm final}), \mathcal{TS}_{\rm OWS, INT}$)
\IF{\textsf{\small isReachable} == \textbf{false}} 
\STATE $\mathcal{T}_{\rm OWS} = \mathcal{T}_{\rm OWS} \backslash \{(\boldsymbol{q}'_{\rm initial}, p_1) \xrightarrow{\boldsymbol{a}} (\boldsymbol{q}'_{\rm final}, p_2)\}$ \hspace{1.8in} \COMMENT{delete unqualified transitions}
\ENDIF
\ENDFOR
\ENDFOR
\STATE \textbf{return} $\mathcal{TS}_{\rm OWS} = (\mathcal{Q}_{\rm OWS}, \mathcal{I}_{\rm OWS}, \mathcal{P}_{\rm OWS}, \mathcal{A}_{\rm OWS}, \mathcal{T}_{\rm OWS})$
\end{algorithmic}
\end{algorithm}

Algorithm~\ref{al:abstraction} details the construction of the robust finite transition subsystem above. The high-level task planner specifies the inputs of the algorithm, i.e., two pairs of keyframe states, locomotion modes, and contact actions $(\boldsymbol{q}_{\rm initial}, p_1, s_1)$ and $(\boldsymbol{q}_{\rm final},p_2, s_2)$ with robustness margins $\epsilon_1$ and $\epsilon_2$, respectively, and by Def.~\ref{def:robust-finite-transition}, determines the set of finite states $\mathcal{Q}_{\rm OWS}$. This is the top-down component of our approach. 
The bottom-up component is the reachability control synthesis introduced in the next subsection.
Algorithm~\ref{al:abstraction} integrates the top-down and bottom-up components. 


The proposed robust finite transition system (RFTS) differs from the abstraction approaches in [\cite{liu2014abstraction, liu2016finite, tabuada2009verification,belta2017formal}] with respect to the following points: (i) The most salient difference is that our planning approach is a hierarchy consisting of both top-down and bottom-up components. The RFTS is an interface taking the desired command from the high-level symbolic task planner (i.e., the top-down component) and use this command to synthesize a reachability controller in the low-level motion planner (i.e., the bottom-up component). The approach in [\cite{liu2014abstraction, liu2016finite}] is an abstraction of the underlying continuous dynamical system and represents a bottom-up approach.
(ii) By using the proposed hierarhical structure, we are able to solve a more challenging problem with whole-body dynamic locomotion in a constrained environment, instead of simple examples such as 2D mobile robot or vehicle.
(iii) Our RFTS reasons about the robustness to bounded state disturbances at not only the inter-sampling level, but also the locomotion keyframe level capturing the essential locomotion dynamics.
(iv) We incorporate hybrid dynamics into our RFTS design, which is constructed for the one walking step. Overall, our planning framework sequentially composes multiple locomotion modes. (v) Instead of a grid-based partition, we use a locomotion-manifold-based partition to characterize the robustness margins of the keyframe states in the phase-space.

\subsection{One-walking-step reachability control synthesis}
\label{subsec:controlSynthesis}
To determine the transitions satisfying the conditions in Eqs.~(\ref{eq:initial-bounds})-(\ref{eq:intermediateCondition_1}) 
of Def.~\ref{def:robust-finite-transition}, we employ abstraction-based control synthesis developed for general dynamical systems. The idea of this approach is to automatically and rigorously compute the set of states that can be controlled to realize a given specification and generate feedback controllers for those states. Generally, abstraction-based control synthesis consists of three steps: 1) Construct a finite transition system (also called a finite abstraction) that over-approximates the dynamics of the original continuous system. 2) Design control algorithms based on the finite transition system with respect to the given specification. This step not only verifies whether the given specification is realizable by the low level robot dynamics, but also synthesizes a controller for the abstraction if realizable. 3) Interpolate the synthesized controller to be executed in the original continuous system.

We consider a one-walking-step locomotion subsystem defined on a local state space determined by two keyframe states.

\begin{definition}[\textbf{One-walking-step locomotion subsystem}]\label{def:SS_OWS}
Given the switched system tuple in Eq.~(\ref{eq:SwitchTuple}), a one-walking-step (OWS) locomotion subsystem from a given keyframe state $\boldsymbol{q}'_{\rm initial}$ with a robustness margin $\epsilon_1$ in the $(p_1^{\rm th}, s_1^{\rm th})$ mode to another keyframe state $\boldsymbol{q}'_{\rm final}$ with a robustness margin $\epsilon_2$ in the $(p_2^{\rm th}, s_2^{\rm th})$ mode is formulated as:
\begin{align}\label{eq:OWSsubsystem}
\mathcal{SS}_{\rm OWS} = (\Xi_{\rm OWS}, \Xi_{{\rm OWS}, 0}, \mathcal{U}_{\rm OWS}, \mathcal{P}_{\rm OWS}, f_{\rm OWS})
\end{align}
where the state space of the subsystem $\Xi_{\rm OWS} \subseteq\Xi$ is a local area determined by the two keyframe states $\boldsymbol{q}'_{\rm initial}$ and $\boldsymbol{q}'_{\rm final}$; $\Xi_{{\rm OWS}, 0}=\mathcal{B}_{\epsilon_1}(\boldsymbol{q}_{\rm initial})$ is the set of initial continuous states, and $\Xi_{\rm OWS}=\Xi_{p_1,{\rm OWS}}\cup\Xi_{p_2,{\rm OWS}}$ with $\Xi_{p_1,{\rm OWS}}$ and $\Xi_{p_2,{\rm OWS}}$ representing the local state space of the locomotion modes $p_1$ and $p_2$, respectively; $\mathcal{U}_{\rm OWS}$ is the allowable control input set for one walking step; $\mathcal{P}_{\rm OWS} = \{p_1, p_2\}$ represents a locomotion mode set composed of two consecutive walking steps; $f_{\rm OWS}=\{f_{p_1},f_{p_2}\}$ is the set of vector fields determined by $(p_1^{\rm th}, s_1^{\rm th})$ and ($p_2^{\rm th}, s_2^{\rm th})$ command pairs, respectively. The mode transition instant $\zeta_{p_1}$ is determined by Eq.~(\ref{eq:intermediateCondition_1}). 
\end{definition}

\begin{remark}
The state spaces $\Xi_{p_1,{\rm OWS}}$ and $\Xi_{p_2,{\rm OWS}}$ overlap so that a contact switch can happen during one walking step. This overlap should fully cover the intersection of two robust tubes defined in Eqs.~(\ref{eq:intermediateCondition_1}). A straightforward option is to make both $\Xi_{p_1,{\rm OWS}}$ and $\Xi_{p_2,{\rm OWS}}$ identical to the state space fully covering one walking step.
\end{remark}



The control actions $\boldsymbol{a}$ defined in Def.~\ref{def:robust-finite-transition} are a sequence of control signals for one walking step. We discrete the control space and maintain a constant control signal for each time step. In the following, we propose a finite abstraction of the one-walking-step locomotion subsystem $\mathcal{SS}_{\rm OWS}$. This abstraction is based on a predefined time step $\delta \zeta$ for the construction of control signals, a bounded disturbance $d$, and a finer Euclidean space discretization (i.e., an abstraction map $\mathcal{M}_{\rm OWS, Euc}$) rather than the one used in the task planner.

\begin{figure}[t]
 \centering
   \includegraphics[width=0.75\linewidth]{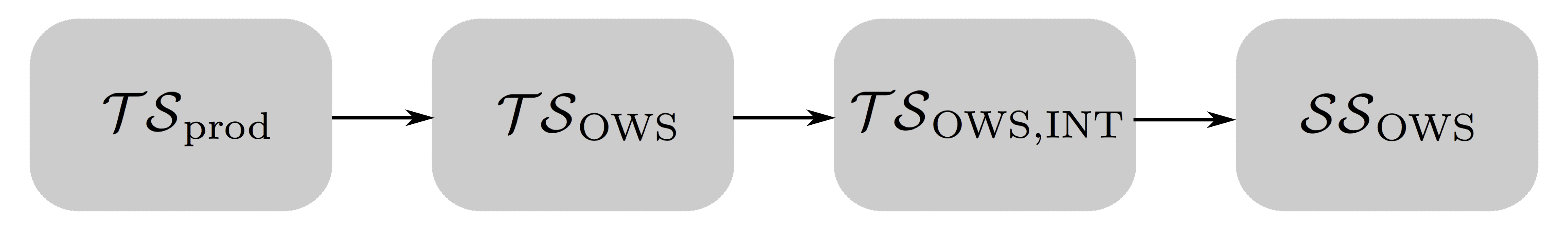}
 \caption{\captionsize Top-down hierarchy of layered abstractions: $\mathcal{TS}_{\rm prod}$ represents the finite product transition system of the high-level task planner in Section~\ref{sec:problem-formulation}; $\mathcal{TS}_{\rm OWS}$ denotes the robust finite transition system for one walking step in Def.~\ref{def:robust-finite-transition}; $\mathcal{TS}_{\rm OWS, INT}$ denotes the inter-sampling finite abstraction of one walking step in Def.~\ref{def:abstlowlevel}; $\mathcal{SS}_{\rm OWS}$ represents the continuous one-walking-step locomotion subsystem in Def.~\ref{def:SS_OWS}. }
\label{fig:FTS_sequence} 
\end{figure}

\begin{definition}[\textbf{Inter-sampling finite abstraction of one walking step}]\label{def:abstlowlevel}
Given a one-walking-step locomotion subsystem $\mathcal{SS}_{\rm OWS}$, an abstraction map $\mathcal{M}_{\rm OWS, Euc}: \Xi_{\rm OWS} \rightarrow \mathcal{Q}_{\rm OWS, INT}$, and a time step $\delta \zeta$, a finite transition system 
\begin{align*}
    \mathcal{TS}_{\rm OWS, INT}=(\mathcal{Q}_{\rm OWS, INT},\mathcal{Q}_{{\rm OWS, INT},0}, \mathcal{A}_{\rm OWS, INT},\mathcal{T}_{\rm OWS, INT})
\end{align*}
is defined as an inter-sampling finite abstraction of $\mathcal{SS}_{\rm OWS}$, denoted by $\mathcal{SS}_{\rm OWS}\preceq_{(\delta \zeta, d,  \mathcal{M}_{\rm OWS, Euc})}\mathcal{TS}_{\rm OWS, INT}$ if the following conditions hold

\begin{itemize}
    \item $\mathcal{Q}_{\rm OWS, INT}=\mathcal{Q}_{p_1,{\rm OWS, INT}}\cup\mathcal{Q}_{p_2,{\rm OWS, INT}}=\bigcup_{\boldsymbol{\xi}\in\Xi_{p_1,{\rm OWS}}}\mathcal{M}_{\rm OWS, Euc}(\boldsymbol{\xi})\cup\bigcup_{\boldsymbol{\xi}\in\Xi_{p_2,{\rm OWS}}}\mathcal{M}_{\rm OWS, Euc}(\boldsymbol{\xi})$ is a finite set of discrete states; an initial set of discrete states is defined as $\mathcal{Q}_{{\rm OWS, INT},0}=\bigcup_{\boldsymbol{\xi}\in\Xi_{\rm OWS,0}}\mathcal{M}_{\rm OWS, Euc}(\boldsymbol{\xi})$.
    \item $\mathcal{A}_{\rm OWS, INT} = \{ \mathcal{U}_{p_1}, \mathcal{U}_{p_2}\}$ is the set of control values, where $\mathcal{U}_{p_1}$ and $\mathcal{U}_{p_2}$ are the allowable control input ranges in the $p_1^{\rm th}$ and $p_2^{\rm th}$ locomotion modes, respectively.
    \item $(\boldsymbol{q},\boldsymbol{a}_{\rm INT},\boldsymbol{q}')\in\mathcal{T}_{\rm OWS, INT}$ for $\boldsymbol{q},\boldsymbol{q}'\in\mathcal{Q}_{p_1,{\rm OWS, INT}}$ (or $\boldsymbol{q},\boldsymbol{q}'\in\mathcal{Q}_{p_2,{\rm OWS, INT}}$, respectively), if there exists $\boldsymbol{a}_{\rm INT}\in\mathcal{A}_{\rm OWS, INT}$ being constant with one time-step duration $\delta \zeta$ and some external disturbance $d:[0,\delta \zeta]\to D_{\rm OWS}$ such that the resulting solution $\boldsymbol{\xi}:[0,\delta \zeta]\to\Xi_{p_1,{\rm OWS, INT}}$ (or $\boldsymbol{\xi}:[0,\delta \zeta]\to\Xi_{p_2,{\rm OWS, INT}}$, respectively) satisfies $\boldsymbol{\xi}(0)=\boldsymbol{\xi}_0$, $\boldsymbol{\xi}_0\in\mathcal{M}_{\rm OWS, Euc}^{-1}(\boldsymbol{q})$, $\boldsymbol{\xi}(\delta \zeta)\in\mathcal{M}_{\rm OWS, Euc}^{-1}(\boldsymbol{q}')$ and the system dynamics in Eq.~ (\ref{eq:phase1_1}) (or Eq.~(\ref{eq:phase2_1}), respectively).
\end{itemize}
\end{definition}


The abstraction map $\mathcal{M}_{\rm OWS, Euc}: \Xi_{\rm OWS, INT} \rightarrow \mathcal{Q}_{\rm OWS, INT}$ maps a continuous state in $\Xi_{\rm OWS, INT}$ into a discrete state in the set $\mathcal{Q}_{\rm OWS, INT}$. Equivalently, $\Xi_{\rm OWS, INT} \coloneqq \bigcup_{\boldsymbol{q}\in \mathcal{Q}_{\rm OWS, INT}}\mathcal{M}^{-1}_{\rm OWS, Euc}(\boldsymbol{q})$. A typical implementation of such a map $\mathcal{M}_{\rm OWS, Euc}$ is a uniform partition with a specific granularity. The condition in the last item  of Def.~\ref{def:abstlowlevel} indicates that $\mathcal{TS}_{\rm OWS, INT}$ is an over-approximation of $\mathcal{SS}_{\rm OWS}$. That is, all the transitions will be included as long as a transition is possible by using the locomotion dynamics under bounded disturbances. For instance, let us examine two consecutive inter-sampling discrete states $\boldsymbol{q}$ and $\boldsymbol{q}'$. We add a transition $(\boldsymbol{q},\boldsymbol{a}_{\rm INT},\boldsymbol{q}')$ if
\begin{align}\label{eq:reachset_intersection}
    \mathcal{M}_{\rm OWS, Euc}^{-1}(\boldsymbol{q}')\cap\mathcal{R}_{\delta \zeta}(\mathcal{M}_{\rm OWS, Euc}^{-1}(\boldsymbol{q}),\boldsymbol{a}_{\rm INT})\neq\emptyset,
\end{align}
where $\mathcal{R}_{\delta \zeta}(\cdot,\cdot)$ is defined as $\mathcal{R}_{\delta \zeta}(\Xi_0,\boldsymbol{u}):=\{\boldsymbol{\xi}(\delta \zeta)\,|\;\dot{\boldsymbol{\xi}}(\tau_r)= f_{p_i}\big(\boldsymbol{\xi}(\tau_r),\boldsymbol{u}(\tau_r), d(\tau_r)\big), \tau_r \in[0,\delta \zeta], \boldsymbol{\xi}(0)\in \Xi_0, i=1,2\}$, representing the reachable set of $\Xi_0\subseteq\Xi_{\rm OWS, INT}$ after a time step $\delta \zeta$ under the constant control input $\boldsymbol{u}$. For nonlinear dynamics, it is difficult to compute the exact reachable set $\mathcal{R}_{\delta \zeta}(\Xi_0,\boldsymbol{u})$. To circumvent this hurdle, we compute an over-approximation of the exact reachable set $\mathcal{R}_{\delta \zeta}(\Xi_0,\boldsymbol{u})$, denoted as $\widehat{\mathcal{R}}_{\delta \zeta}(\Xi_0,\boldsymbol{u})$. This over-approximation is obtained via employing interval-valued functions (refer to [\cite{jaulin2001applied}] for the details) of the discretized low-level dynamics. As a counterpart of real-valued functions, such an interval-valued function is evaluated over intervals and obeys interval arithmetic. As such, all the reachable states after a time step $\delta \zeta$ from any state in $\Xi_0$ are captured in the output of an interval-valued function. By refining the set $\Xi_0$ into smaller intervals, we can approximate the reachable set $\mathcal{R}_{\delta \zeta}(\Xi_0,\boldsymbol{u})$ with an arbitrary precision [\cite{liu2017robust}].


Next, we will discuss in detail how to compute the over-approximation $\widehat{\mathcal{R}}_{\delta \zeta}(\cdot,\cdot)$. 
\begin{assumption}[\textbf{Disturbance additivity and boundedness}]\label{assum:disturbanceAddBound}
We assume that the right-hand side of the disturbed switched system in Eq.~(\ref{eq:switchedsystem}) can be divided into a nominal part and a disturbance part: 
\begin{align}\label{eq:fullmodel}
f_p\big(\boldsymbol{\xi}(\zeta), \boldsymbol{u}(\zeta), d(\zeta)\big) = f_p\big(\boldsymbol{\xi}(\zeta), \boldsymbol{u}(\zeta)\big) + g_p\big(\boldsymbol{\xi}(\zeta), \boldsymbol{u}(\zeta)\big), \; p \in \mathcal{P}, \zeta \geq 0,
\end{align}
and the disturbance part is element-wise upper bounded by 
\begin{align}
\big|g_p\big(\boldsymbol{\xi}, \boldsymbol{u}\big)\big| \preceq r, \;\forall p \in \mathcal{P},\;\boldsymbol{\xi} \in \Xi,\;  \boldsymbol{u} \in \mathcal{U}, 
\end{align}
with the bound vector $r \in \mathbb{R}^n$. 
\end{assumption}

\begin{algorithm}[t]
\caption{Reachability control synthesis}
\label{al:reachcontrol}
\begin{algorithmic}[1]
    \STATE \textbf{procedure} ReachabilityControl(initial set $I$, target set $G$, inter-sampling finite abstraction $\mathcal{TS}_{\rm OWS, INT}$)
    \STATE assign a queue ${\rm Que}\leftarrow G$, $G\subseteq\mathcal{Q}$ \hspace{3.2in}\COMMENT{define a FIFO queue}
    \STATE initialize a winning set $\mathcal{WIN}\leftarrow G$
    \STATE $\mathcal{K}\leftarrow \boldsymbol{0}^{N\times M}$ \hspace{2.6in} \COMMENT{$N=|\mathcal{Q}_{\rm OWS, INT}|_{\rm numel}$ and $M=|\mathcal{A}_{\rm OWS, INT}|_{\rm numel}$}
    \WHILE{${\rm Que}\neq\emptyset$}
    \STATE $\boldsymbol{q}'\leftarrow {\rm Que.pop}()$
    \FORALL{$\boldsymbol{q}\in\mathcal{Q}_{\rm OWS, INT}$ and $a\in\mathcal{A}_{\rm OWS, INT}$ such that $\boldsymbol{q}\xrightarrow{a}\boldsymbol{q}'$}
    \IF{$\boldsymbol{q}''\in \mathcal{WIN}$ for all $\boldsymbol{q}''$ such that $\boldsymbol{q}\xrightarrow{a}\boldsymbol{q}''$}
    \IF{$\boldsymbol{q}\notin \mathcal{WIN}$}
    \STATE ${\rm Que}\leftarrow \boldsymbol{q}$
    \STATE $\mathcal{WIN}\leftarrow \boldsymbol{q}$
    \STATE $\mathcal{K}(\boldsymbol{q},a)\leftarrow 1$
    \ENDIF
    \ENDIF
    \ENDFOR
    \ENDWHILE
    \IF{$\mathcal{WIN}\cap I\neq\emptyset$}
    \STATE \textsf{\small isReachable} $\leftarrow \textbf{true}$
    \ELSE
    \STATE \textsf{\small isReachable}$\leftarrow \textbf{false}$
    \ENDIF
    \RETURN \textsf{\small isReachable}, $\mathcal{WIN}$, $\mathcal{K}$
\end{algorithmic}
\end{algorithm}

Given a locomotion mode, we denote by $\delta\boldsymbol{\xi}(\zeta)$ the difference of two trajectories $\boldsymbol{\xi}_1$ and $\boldsymbol{\xi}_2$ at the same instant $\zeta$. These two trajectories start from their initial states $\boldsymbol{\xi}_{1,0}$ and $\boldsymbol{\xi}_{2,0}$, respectively. With the Lipschitz condition and Assumption~\ref{assum:disturbanceAddBound}, we have
\begin{align*}
    |\delta\dot{\boldsymbol{\xi}}(\zeta)|\preceq L|\delta\boldsymbol{\xi}(\zeta)|+r 
    \;\; \implies \;\; |\delta\boldsymbol{\xi}(\zeta)|\preceq  |\boldsymbol{\xi}_{1,0}- \boldsymbol{\xi}_{2,0}|e^{L\zeta}+  L^{-1}(e^{L\zeta}-\boldsymbol{I}_{n\times n})r,\;\forall \zeta\in[0,\delta \zeta],
\end{align*}
which implies that under a disturbance bounded by the vector $r$, all the possible states after a time step $\delta \zeta$ stay within a ball centered at the nominal trajectory state with a radius vector $r_{\delta \zeta}=|\boldsymbol{\xi}_{1,0}- \boldsymbol{\xi}_{2,0}|e^{L\zeta} + L^{-1}(e^{L\delta \zeta}-\boldsymbol{I}_{n\times n})r$. Hence, the reachable set $\mathcal{R}_{\delta \zeta}(\mathcal{M}_{\rm OWS, Euc}^{-1}(\boldsymbol{q}),\boldsymbol{a}_{\rm INT})$ in Eq.~(\ref{eq:reachset_intersection}) can be over-approximated by the estimated reachable set of the nominal system $\dot{\boldsymbol{\xi}}(\zeta)= f_{p_i}\big(\boldsymbol{\xi}(\zeta),\boldsymbol{u}(\zeta))$ enlarged by $r_{\delta \zeta}$.

Given the abstraction defined in Def.~\ref{def:abstlowlevel}, we synthesize a reachability controller for the inter-sampling finite abstraction $\mathcal{TS}_{\rm OWS, INT}$ of a one-walking-step subsystem $\mathcal{SS}_{\rm OWS}$ as shown in Algorithm~\ref{al:reachcontrol}. This algorithm takes as inputs an initial set $I$, a target set $G$, and a finite abstraction $\mathcal{TS}_{\rm OWS, INT}$. Backward dynamics propagation is used to determine the realizability of the reachability controller. This algorithm returns a boolean value \textsf{\small isReachable} indicating the realizability of the target set $G$. If this target set is realizable, it outputs two additional sets: (i) a winning set $\mathcal{WIN}$ defined as all the states from which the reachability goal is satisfied under bounded state disturbances; and (ii) a boolean matrix $\mathcal{K}$ indexing the control strategy $\Omega: \mathcal{Q}\to 2^{\mathcal{A}}$. Otherwise, $\mathcal{WIN}$ and $\Omega$ are returned as empty sets. Note that, the operator $|A|_{\rm numel}$ on Line 4 of Algorithm~\ref{al:reachcontrol} represents the total number of elements in the set $A$. Given a library of synthesized controllers in Algorithm~\ref{al:reachcontrol}, an execution of the complete reachability controller based on the robust finite transition system is shown in Algorithm~\ref{al:execution_reachcontrol} in the Appendix.

A merit of the proposed hierarchical structure is to decompose the overall high-dimensional contact-rich planning problem into tractable sub-problems with smaller state dimensions, circumventing prohibitive computational complexity. In particular, the symbolic task planner takes charge of the high-level decisions being reactive to the environment actions. The middle-level robust finite transition system reasons about the robustness of a local phase space region around the nominal keyframe state w.r.t bounded state disturbances. The low-level phase-space planner executes the continuous locomotion dynamics. This hierarchy is analogous to the receding horizon control approach in [\cite{wongpiromsarn2012receding}], where the complex high-dimensional planning problem is decomposed into a set of solvable sub-problems. This strategy facilitates efficient decision making during dynamic interactions with uncertain environments.


\begin{remark}
$\mathcal{TS}_{\rm prod}$ and $\mathcal{TS}_{\rm OWS}$ establish a hierarchical relationship for task decomposition. $\mathcal{TS}_{\rm prod}$ is a high-level decision maker of a nominal keyframe state while $\mathcal{TS}_{\rm OWS}$ reasons about the robustness of the local phase-space region around the nominal keyframe state determined from $\mathcal{TS}_{\rm prod}$. Overall, $\mathcal{TS}_{\rm prod}$ and $\mathcal{TS}_{\rm OWS}$ form a top-down hierarchy (see Fig.~\ref{fig:FTS_sequence}) to simultaneously achieve ``global'' phase-space decision making and ``local'' robustness reasoning.
\end{remark}
In next section, we will prove the robust reachability of the synthesized controller, i.e., the reachability goal is realizable for $\mathcal{SS}_{\rm OWS}$ if it is realizable for $\mathcal{TS}_{\rm OWS}$. With such a guarantee, the robust finite transition system $\mathcal{TS}_{\rm OWS}$ interfaces the high-level task planner commands with the low-level hybrid motion planner.


 \section{Correctness of The Reactive Task and Motion Planner}
\label{sec:correctness-reasoning}
Correctness guarantees of the whole-body dynamic locomotion (WBDL) planner play a key role in the successful execution of robust legged locomotion interacting with dynamic environments. The objective of this section is to prove such a correctness. In particular, the correctness of our planning framework is interpreted as successful implementations of the high-level task planner on the low-level motion planner under bounded disturbances. Our locomotion planner is a hierarchy composed of a task planning layer with reactive synthesis and a robust motion planner layer synthesized by a robust finite transition system. A high-level task planner, i.e., a WBDL winning strategy $\mathcal{A}_{\rm WBDL}$, is synthesized via a two player game $\mathcal{G}$.
The two-player game $\mathcal{G}$ is solved between the robot and its environment to make a decision $(p, s, \boldsymbol{q})$ representing the locomotion mode $p$, the contact action $s$, and the keyframe state $\boldsymbol{q}$, respectively. This locomotion decision determines a nominal phase-space motion plan and is sent to the robust finite transition system $\mathcal{TS}_{\rm OWS}$ such that the high-level decision is achieved in a robust manner by the low-level continuous motion planner.


\subsection{One-walking-step robust reachability}
\label{subsection:one-walking-step-reachability}


%
To guarantee the robust finite transition system to be realizable by its underlying continuous system, we need to prove that the conditions in Eqs.~(\ref{eq:initial-bounds})-(\ref{eq:intermediateCondition_1}) of Def.~\ref{def:robust-finite-transition} also hold for continuous states. Since the robust finite transition system is based on the keyframes of one walking step, we name the keyframe reachability problem as ``one-walking-step robust reachability''. We model the bounded disturbance causing initial state deviations, model uncertainties, and external perturbations during the evolution of the locomotion trajectory. The term ``robust reachability'' refers to the reachability of the goal robustness margin set centered around the final keyframe from the initial robustness margin set.

\begin{theorem}[\textbf{One-walking-step robust reachability}]\label{theorem:OWSReachability}
Consider a one-walking-step locomotion subsystem $\mathcal{SS}_{\rm OWS}$ with two pairs of decisions $(\boldsymbol{q}_{\rm initial},p_1, s_1)$ and $(\boldsymbol{q}_{\rm final}, p_2, s_2)$ and its inter-sampling finite abstraction $\mathcal{TS}_{\rm OWS, INT}$. Assume that $\mathcal{SS}_{\rm OWS}\preceq_{(\delta \zeta, d,  \mathcal{M}_{\rm OWS, Euc})}\mathcal{TS}_{\rm OWS, INT}$ as defined in Def.~\ref{def:abstlowlevel}. If it is realizable for $\mathcal{TS}_{\rm OWS, INT}$, this walking step is realizable for $\mathcal{SS}_{\rm OWS}$, i.e., the robustness margin set $\mathcal{B}_{\epsilon_2}(\boldsymbol{q}_{\rm final})$ of the final keyframe $\boldsymbol{q}_{\rm final}$ is reachable from $\mathcal{B}_{\epsilon_1}(\boldsymbol{q}_{\rm initial})$.
\end{theorem}
\begin{proof}
Suppose that the walking step from $(\boldsymbol{q}_{\rm initial},p_1, s_1)$ to $(\boldsymbol{q}_{\rm final}, p_2, s_2)$ is realizable for $\mathcal{TS}_{\rm OWS, INT}$, i.e., the winning set $\mathcal{WIN}_{\rm TS}\neq\emptyset$ and there exists a control strategy $\Omega_{\rm TS}:\mathcal{Q}\to 2^{\mathcal{A}}$ for $\mathcal{TS}_{\rm OWS, INT}$. Let $\boldsymbol{q}\in\mathcal{B}_{\epsilon_1}(\boldsymbol{q}_{\rm initial})$ and $\boldsymbol{q}\in\mathcal{WIN}_{\rm TS}$. Under bounded disturbances, all the possible state sequences starting from $\boldsymbol{q}$, generated by the locomotion dynamics and the control strategy synthesized by Algorithm~\ref{al:reachcontrol}, will finally reach the target set $\mathcal{B}_{\epsilon_2}(\boldsymbol{q}_{\rm final})$. Let $\{\boldsymbol{q}_i\}_0^l$ ($l\in \mathbb{Z}$) be one of such sequences generated under a control sequence $\{\boldsymbol{a}_{i,{\rm INT}}\}_0^{l-1}$ and a sequence of disturbances $\{\boldsymbol{d}_i\}_0^{l-1}$ such that $\boldsymbol{q}_0=\boldsymbol{q}$ and $\boldsymbol{q}_l\in\mathcal{B}_{\epsilon_2}(\boldsymbol{q}_{\rm final})$.

We construct a control strategy $\Omega_{\rm OWS}:\Xi_{\rm OWS}\to 2^{\mathcal{A}}$ for the one-walking-step locomotion system $\mathcal{SS}_{\rm OWS}$ by
\begin{align*}
    \boldsymbol{a}_{\rm INT}\in\Omega_{\rm OWS}(\boldsymbol{\xi}),\;\text{if}\;\boldsymbol{a}_{\rm INT}\in\Omega_{\rm TS}(\mathcal{M}_{\rm OWS, Euc}(\boldsymbol{\xi})).
\end{align*}
Given the transitions of $\mathcal{TS}_{\rm OWS, INT}$ assigned by Eq.~(\ref{eq:reachset_intersection}), for $\forall \boldsymbol{\xi}(0)\in\mathcal{M}_{\rm OWS, Euc}^{-1}(\boldsymbol{q})$, there always exist a state $\boldsymbol{q}'$ and a control input $\boldsymbol{a}_{\rm INT}$ such that $\boldsymbol{\xi}(\delta \zeta)\in \mathcal{M}_{\rm OWS, Euc}^{-1}(\boldsymbol{q}')$. Thus,
for any $\boldsymbol{\xi}(0)\in\mathcal{M}_{\rm OWS, Euc}^{-1}(\boldsymbol{q})$, the resulting solution with the same control sequence $\{\boldsymbol{a}_{i,{\rm INT}}\}_0^{l-1}$ and disturbance $\{\boldsymbol{d}_i\}_0^{l-1}$
generated by $\Omega_{\rm OWS}$ will be guaranteed to reach the target set $\mathcal{B}_{\epsilon_2}(\boldsymbol{q}_{\rm final})$. This implies that at time $t_k = k\cdot \delta \zeta, \forall k \leq l$, $\boldsymbol{\xi}(t_k)\in\mathcal{WIN}_{\rm SS}$ (i.e., the winning set of $\mathcal{SS_{\rm OWS}}$). 
Therefore, $\mathcal{WIN}_{\rm SS}\neq\emptyset$ and $\Omega_{\rm OWS}$ is such a controller that can realize the one walking step. This completes the proof.
\end{proof}

\subsection{Correctness of the hierarchical WBDL planner}
\label{subsec:correctness-WBDL}

Given the one-walking-step robust reachability of Theorem~\ref{theorem:OWSReachability}, we now prove the correctness of the top-down planning hierarchy, i.e., $\mathcal{TS}_{\rm prod} \rightarrow \mathcal{TS}_{\rm OWS} \rightarrow \mathcal{TS}_{\rm OWS, INT} \rightarrow \mathcal{SS}_{\rm OWS}$ as shown in Fig.~\ref{fig:FTS_sequence}.
The correctness is defined in a robust sense, i.e., the actual keyframe states of the phase-space trajectory always stays within the robustness margins of the nominal keyframe states determined by the high-level task planner.

\begin{definition}
Assume a low-level locomotion trajectory $\kappa = (\boldsymbol{\xi}, \rho, \eta, \mu)$ and a high-level decision sequence $\gamma = (q, e, s, p)$ as defined in Definition~\ref{def:Execution}. The low-level trajectory $\kappa$ is a continuous implementation of the high-level execution $\gamma$, if there exists a sequence of non-overlapping phase intervals $\mathcal{H} = H_1 \cup H_2\cup H_3 \ldots$ and $\cup^{\infty}_{i=1} H_i = \mathbb{R}^+$ such that $\forall \zeta \in H_k, \forall k \geq 1$, the following mappings hold
\begin{align}\nonumber
\boldsymbol{\xi}(\zeta_{\rm left\text{-}bnd}) \in \mathcal{M}_{\rm Riem}^{-1}(q_k), \boldsymbol{\xi}(\zeta_{\rm right\text{-}bnd}) \in \mathcal{M}_{\rm Riem}^{-1}(q_{k+1}), \rho(\zeta) = e_k, \eta(\zeta) = s_k, \mu(\zeta) = p_k,
\end{align}
where $\mathcal{M}_{\rm Riem}$ is the Riemannian space abstraction defined in Eq.~(\ref{eq:keyframeMapping}), which maps the continuous state $\boldsymbol{\xi}$ region centered at the keyframe state into the discrete keyframe $q$ [\cite{liu2013synthesis}].
$\zeta_{\rm left\text{-}bnd}$ and $\zeta_{\rm right\text{-}bnd}$ are the left and right boundary value of the interval $H_k$.
\end{definition}

\noindent By this definition and stutter-equivalence [\cite{baier2008principles}], we can conclude $\kappa \models \varphi$ if and only if $\gamma \models \varphi$, where $\varphi$ is the task specifications in the symbolic task planner. For our phase-space planning, the interval $H_k$ represents the phase duration of the $k^{\rm th}$ walking step. This guarantees that the left boundary point of $H_k$ approaches to infinity as $k\rightarrow\infty$, and thus the continuous implementation guarantees the Zeno behavior to be ruled out. For detailed explanations, reader can refer to [\cite{liu2013synthesis}] and the reference therein. Given these preliminaries, we have the following correctness theorem:
\begin{theorem}[\textbf{Correctness of the WBDL task and motion planner}]
\noindent\textit{Given a robust finite transition system $\mathcal{TS}_{\rm OWS}$, a winning WBDL strategy synthesized from the two-player game is guaranteed to be implementable by the underlying low-level phase-space motion planner in a provable correct manner.
}
\end{theorem}
%
\begin{proof}
By Proposition~\ref{prop:winningstrategy}, a winning WBDL strategy $\mathcal{A}_{\rm WBDL}$ synthesized from the WBDL task planner game solves the discrete locomotion planning problem on $\mathcal{TS}_{\rm prod}$. This synthesis is correct-by-construction thanks to the properties of GR(1) formulae. According to $\mathcal{A}_{\rm WBDL}$, the system action $s_{k+1}$, switching mode $p_{k+1}$, and keyframe $q_{k+1}$ at the next $(k+1)^{\rm th}$ walking step are derived from next environment actions $e_{k+1}$ and current keyframe state $q_k$. To verify the correct implementation of a high-level decision sequence $\gamma$, we use the switching strategy semantics: given an initial state $\boldsymbol{\xi}(\zeta_0)$ and an initial environment action $\rho(\zeta_0) = e_0$, we assign $\eta(\zeta_0) = s_0$ and $\mu(\zeta_0) = p_0$ according to $\mathcal{A}_{\rm WBDL}$, where the step index $k = 0$. By using the control library synthesized from the robust finite transition system $\mathcal{TS}_{\rm OWS}$ with decision tuples of two consecutive walking steps (i.e., ($p_0, s_0, q_0$) and ($p_{1}, s_{1}, q_{1}$)) , we select a specific reachability controller synthesized by $\mathcal{TS}_{\rm OWS, INT}$ to achieve a robust keyframe transition at the next walking step. This is guaranteed by the one-walking-step robust reachability in Theorem~\ref{theorem:OWSReachability}, where the realizability of $\mathcal{TS}_{\rm OWS, INT}$ implies the realizability of the underlying continuous system $\mathcal{SS}_{\rm OWS}$. By executing this reachability controller, the continuous dynamics $\boldsymbol{\xi}(\zeta)$ evolve by following the dynamics of a specific locomotion mode $\dot{\boldsymbol{\xi}}(\zeta) = f_{p_0}(\boldsymbol{\xi}(\zeta), \boldsymbol{u}(\zeta), d(\zeta))$ under the bounded disturbance $d$.
Once we detect a new environment action $e_2$ before the locomotion contact switch, a new decision tuple ($p_{2}, s_{2}, q_{2}$) is generated immediately based on $\mathcal{A}_{\rm WBDL}$. 
Given this new decision tuple, the same procedure is repeated as above for the future $k^{\rm th}$ walking step where $k \in \mathbb{Z}_{\geq 2}$. Therefore, it is proved that the low-level trajectory correctly implements the high-level decision sequence.
\end{proof}

\begin{figure}[t]
 \centering
   \includegraphics[width=0.75\linewidth]{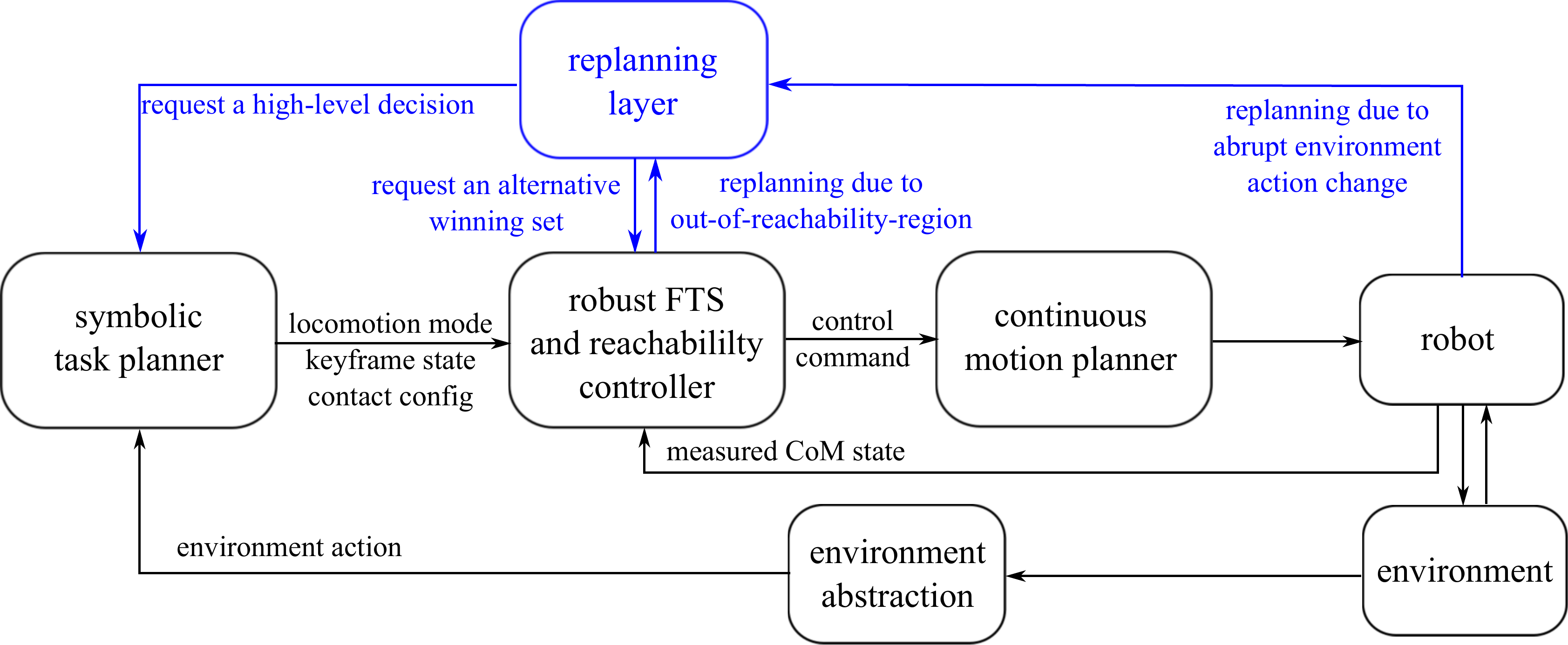}
 \caption{\captionsize Hierarchical structure of robust WBDL task and motion planner. This planner has three cascaded planning layers: high-level task planner, middle-level keyframe-based robust finite transition system, and low-level continuous motion planner. A replanning process (see blue lines) is triggered when the state is out of the reachability region (i.e., the winning set) or a sudden environmental change occurs.
}
\label{fig:planning_diagram}
\end{figure}

\subsection{Replanning strategy and robustness}
\label{subsec:planningAndRobustness}
It is worth noting that the proposed correctness holds under a set of assumptions on allowable environmental and system actions and disturbance boundedness. However, sometimes the real-world disturbance can violate the bounded disturbance assumption and perturb the state to be out of the local reachability region (i.e., the winning set). To handle this situation, we establish a replanning strategy to request a new high-level task planner command. Ideally, the union set of all local winning sets is expected to cover the entire state space of interest. However, the existence of a such a winning set union often can not be guaranteed. Thus, it is difficult to generalize formal correctness of one winning set to that of the union set of all winning sets. What we strive to is to maximize the phase-space coverage by the union set of all winning sets. From a practical implementation viewpoint, synthesizing a large number of winning sets enables our planner to cover a sufficiently large phase space such that it is always likely to find a feasible winning set when large disturbances occur.

In other words, there is no ground truth of ``formal correctness'' for real robotic systems. Even though we have a provably correct planner and implement it on a real robot in a correct way, the actual planner may not be formally ``correct'' due to many potential hardware issues. For instance, unmodelled actuator dynamics can easily break the correctness guarantee of task specifications at the high level. Thus, it makes more practical sense to target a formally correct approach that generates a palette of robust controllers, the winning sets of which jointly cover a sufficiently large phase-space of interest (if not a global state space). Our results indicate that a properly designed controller switching mechanism among these locomotion winning sets enables an effective replanning strategy such that a set of contiguous phase-space initial robustness margin sets can be controlled to reach a set of contiguous goal robustness margin sets under bounded disturbances.

Overall, the proposed planning framework reasons about robustness at the following three levels.
\begin{itemize}
    \item The robust finite transition system $\mathcal{TS}_{\rm OWS}$ explicitly incorporates neighbouring keyframe states via the robustness margin around the nominal keyframe state to handle initial state uncertainty in each walking step.
    \item If the disturbance is larger than the boundary value modeled in the controller synthesis, the state may be disturbed out of the winning set. In this case, an alternative winning set will be searched in the control library of allowable keyframe transitions determined by $\mathcal{TS}_{\rm OWS}$.
    If no alternative winning set is feasible, a replanning signal will be sent to the high-level task planner. The task planner will use the synthesized automaton $\mathcal{A}_{\rm WBDL}$ to assign a new locomotion decision $(p, s, \boldsymbol{q})$ and send it to the motion planner layer for replanning the next walking step.
    \item When an environment event changes suddenly, a replanning signal will be sent to the high-level task planner. Note that, this replanning process can only be executed before the next step transition. Otherwise, the contact of the next walking step already occurs. Fig.~\ref{fig:planning_sequence} in the Appendix shows a timing sequence of the replanning process. More details of this replanning process are shown in Algorithm~\ref{al:execution_reachcontrol} in the Appendix.
\end{itemize}

\begin{algorithm}[t]
\caption{Execution of reactive task and motion planner}\label{al:execution_highlevel}
\begin{algorithmic}[1]
\STATE \textbf{procedure} ExecuteReactivePlanner(nextEnvironmentAction $e_{\rm next}$, currentKeyframe $q_{\rm current}$, currentLocomotionMode $p_{\rm current}$, automaton $\mathcal{A}_{\rm WBDL}$)
\STATE $e_{\rm next} \in \mathcal{E} $ 
\STATE $q_{\rm next} \gets$ getNextKeyframe($e_{\rm next}, q_{\rm current}$) \hspace{3.0in} \COMMENT{look up $\mathcal{A}_{\rm WBDL}$}
\STATE $s_{\rm next} \gets$ getNextContactAction($e_{\rm next}, q_{\rm next}$) 
\STATE $p_{\rm next} \gets$ getNextLocomotionMode($e_{\rm next}, q_{\rm next}$) 
\STATE $(\boldsymbol{\xi}_{{\rm trans}}, t_{\rm trans}) \gets$ searchContactTransition($q_{\rm current}, q_{\rm next}$)
\STATE $(\boldsymbol{y}_{\rm limb, next}) \gets$ searchLateralLimbLocation($q_{\rm current}, q_{\rm next}$)
\STATE $(\boldsymbol{\xi}, \boldsymbol{\chi}) \gets$ ExecuteOWSReachabilityControl(keyframeState $q_{\rm current}$, $q_{\rm next}$, locomotionMode $p_{\rm current}$, $p_{\rm next}$, contactAction $s_{\rm current}, s_{\rm next}$, initialCoMState $\boldsymbol{\xi}_{\rm init}$, transitionTime $t_{\rm trans}$, nextEnvironmentAction $e_{\rm next}$)
\STATE $(e, q, p, s)_{\rm current} \gets (e, q, p, s)_{\rm next}$ \hspace{3in} \COMMENT{update task planner states}
\STATE \textbf{repeat} from Line 2
\STATE  \textbf{end procedure}
\end{algorithmic}
\end{algorithm}
%
%
 \section{Results}
\label{section:results}
\begin{figure*}[t]
 \centering
   \includegraphics[width=0.8\linewidth]{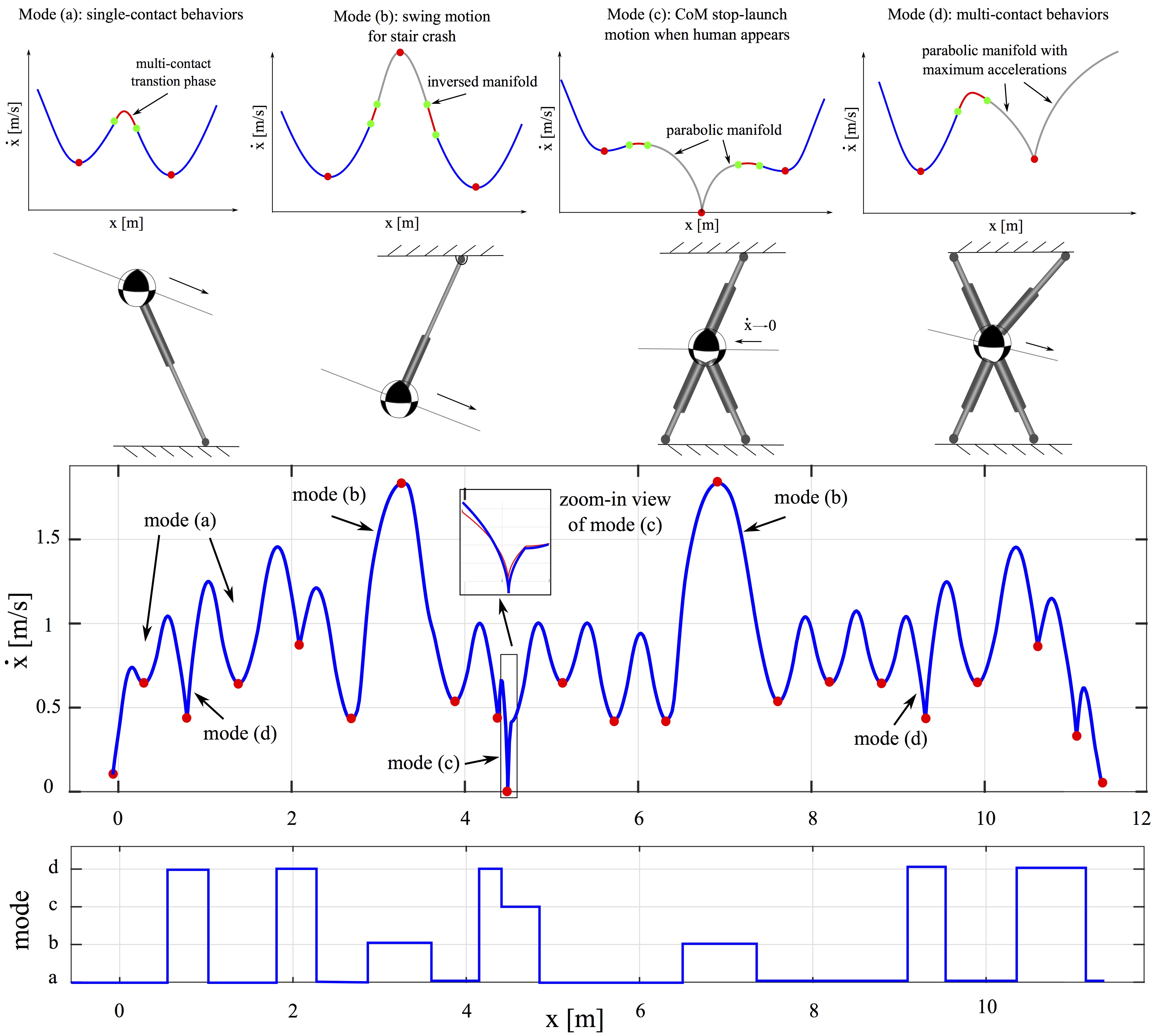}
 \caption{\captionsize Sequential composition of the sagittal CoM phase-space trajectories and mode switchings for a 20-step WBDL maneuver. The top four figures illustrate phase-space manifolds of the four locomotion modes. The mode switching is governed by the synthesized high-level contact planner. Among these steps, two terrain crack and one human appearance events are taken into account. A short multi-contact phase is designed between every two consecutive modes for a smooth transition (see the short red trajectory between two green dots).}
\label{fig:Specializedplanner}
\end{figure*}
\begin{figure}[t]
 \centering
   \includegraphics[width=.55\linewidth]{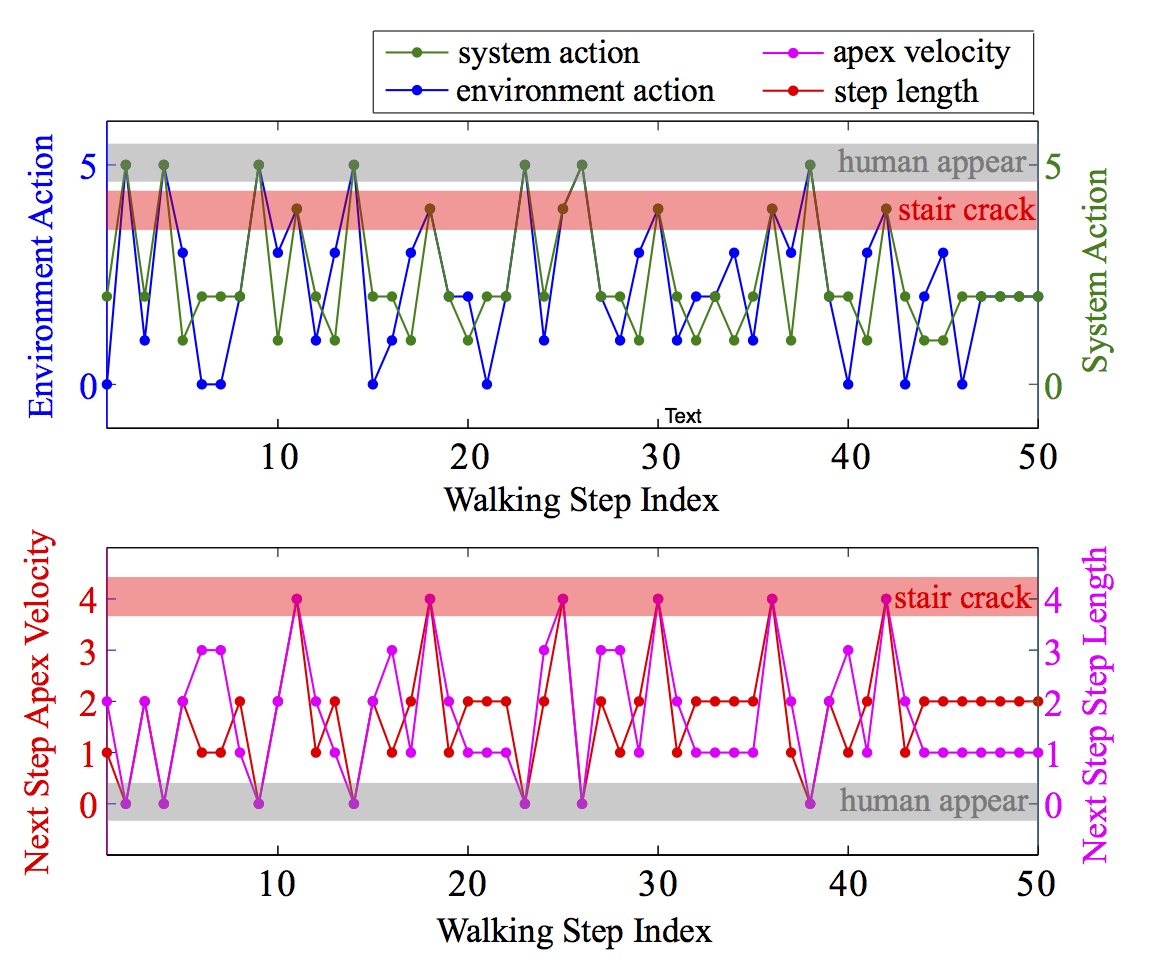}
 \caption{\captionsize Environment events, system actions and keyframe states of 50 walking steps according to the synthesized automaton. Actions and states are indexed by numbers. Emergency events, i.e, human appearance and terrain crack, are highlighted in the shaded regions. In the bottom subfigure, the numbers $0$ to $4$ on the vertical axis correspond to $\{0, 0.4, 0.6, 0.8, 1.7\}$ m/s for next step apex velocity and $\{0.15, 0.5, 0.6, 0.7, 0.6\}$ m for next step step length. }
\label{fig:ActionAndState}
\end{figure}
\begin{figure}[!th]
 \centering
   \includegraphics[width=0.9\linewidth]{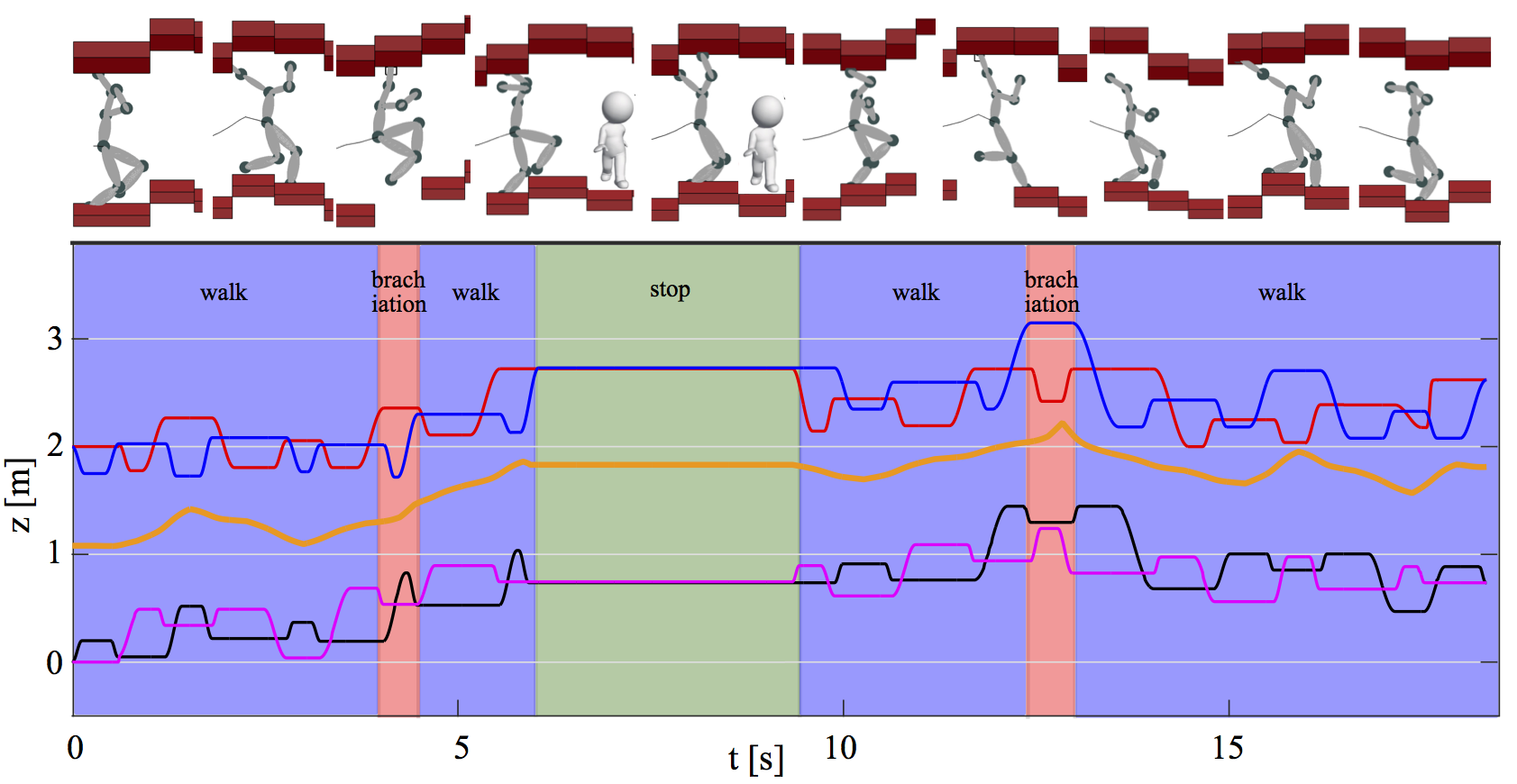}
 \caption{\captionsize Snapshots of the WBDL motions in respond to two environmental emergencies. The snapshots show a sequence of locomotion behaviors including a brachiation motion over the cracked terrain and a stopping motion when a human appears. The figure at the bottom shows the CoM vertical position trajectory (orange thick line), hand and feet trajectories (thin interlaced lines).}
\label{fig:motionSnapshot}
\end{figure}
We demonstrate whole-body dynamic locomotion (WBDL) results by sequentially composing the low-level locomotion modes via the symbolic task planner. In particular, we analyze in detail the robustness performance of the reachability control with respect to several key parameters. The Temporal Logic Planning (TuLiP) toolbox, a python-based embedded control software [\cite{wongpiromsarn2011tulip}], is used to synthesize the symbolic task planner. The gr1c\footnote{\url{http://scottman.net/2012/gr1c}} tool, involving the CU Decision Diagram Package, is used by TuLiP as the underlying synthesis solver.
The synthesized planner is correct by construction, i.e., satisfying all the proposed specifications. The LTL synthesis procedure is offline and will take around 20 minutes to generate an automaton on a MacBook with a 2.9 GHz Intel Core i9 processor and 32 GB of memory. Once the automaton is generated, the task planning process will be executed at run-time. To guarantee the successful implementations of the low-level motion plans by the high-level task planner, we perform a robust reachability analysis of the keyframe states by using the so-called robustly complete control synthesis (ROCS)\footnote{\url{https://git.uwaterloo.ca/hybrid-systems-lab/rocs}} tool [\cite{Li2018rocs}], which currently supports abstraction-based control synthesis using both uniform and non-uniform discretizations. ROCS also generates feedback control strategies, which are used to design the control library of our locomotion tasks.

\begin{figure}[t]
 \centering
   \includegraphics[width=0.7\linewidth]{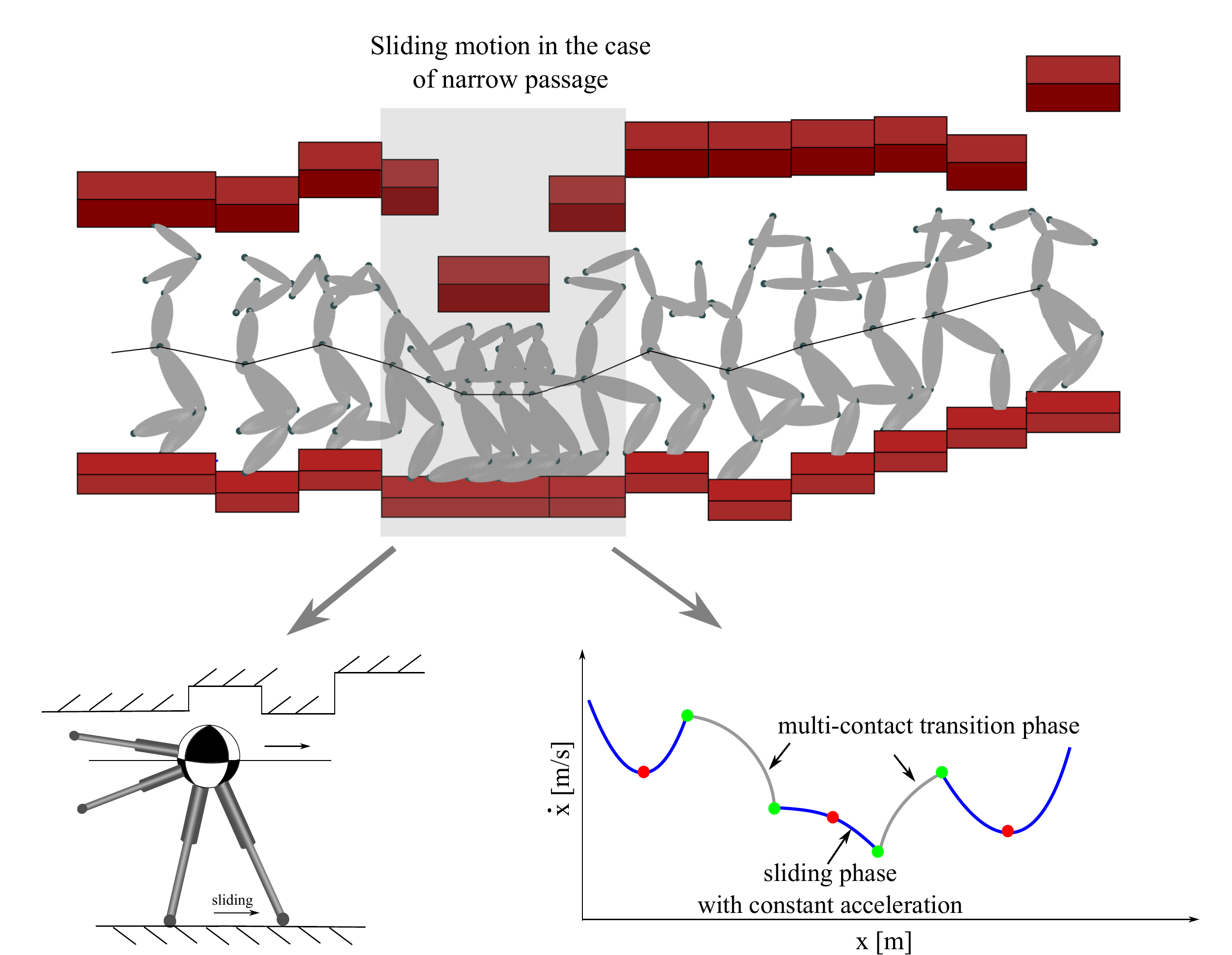}
 \caption{\captionsize Ground sliding motion when a narrow passage appears in the scene. The robot crouches and slides on the ground through the low-ceiling area with a constant negative acceleration. This ground sliding motion is preceded and succeeded by a multi-contact transition phase as shown in the phase-space subfigure at the lower right corner.}
\label{fig:slidingMotion}
\end{figure}

\begin{figure}[t]
 \centering
   \includegraphics[width=0.9\linewidth]{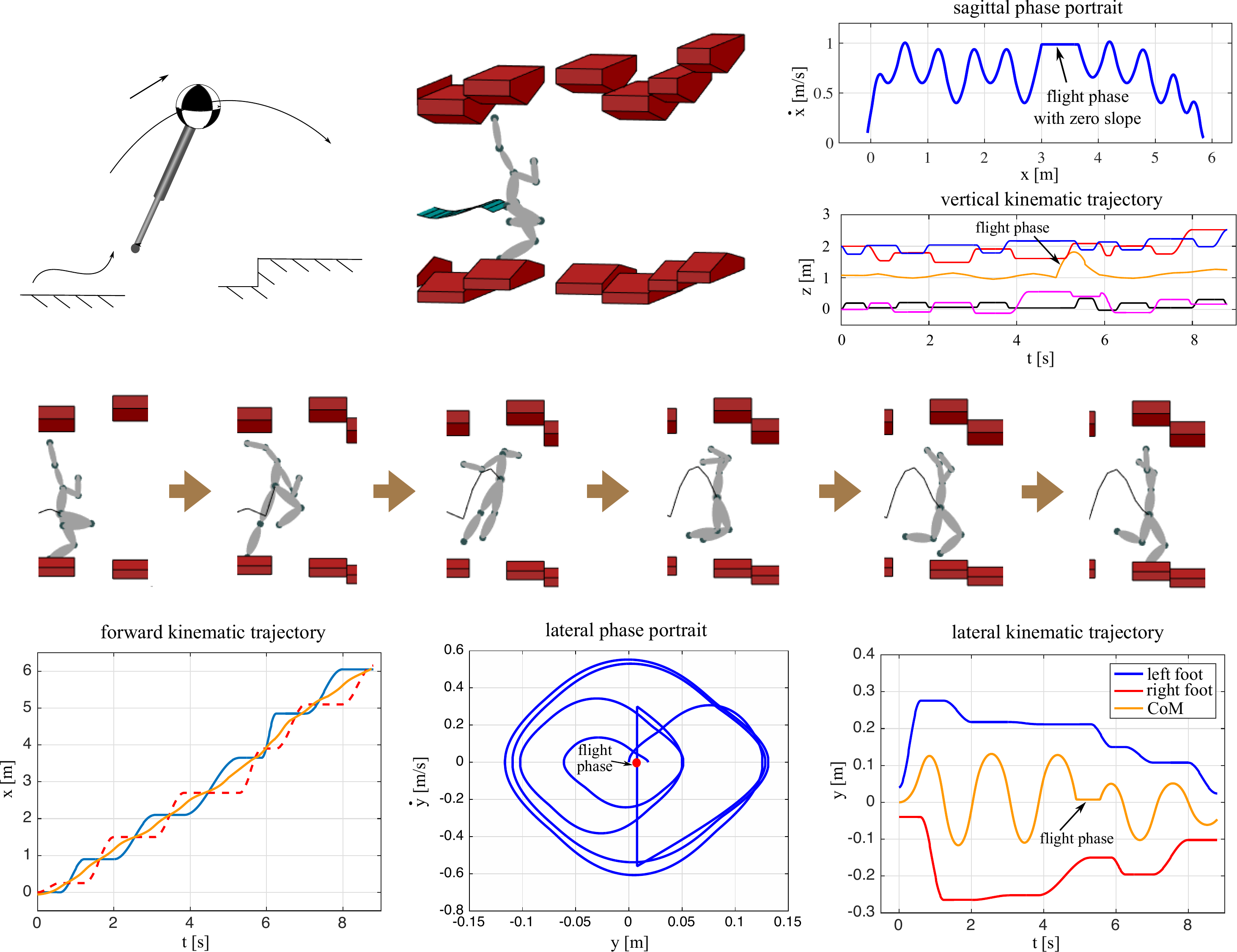}
 \caption{\captionsize Hopping motion when a crack in the terrain and a high ceiling occur simultaneously. In this case, no overhead support is available for grasping so the robot has to jump over the cracked terrain.}
\label{fig:hopping-motion}
\end{figure}
\subsection{Case I: locomotion with stopping and brachiation behaviors}
We first demonstrate a locomotion scenario involving environmental actions such as the appearance of a human and the terrain being crack in the scene. The synthesized discrete task planner is represented by a finite state automaton with $27$ states and $148$ transitions.
The two-dimensional keyframe state $q$ is composed of the apex velocity and step length. For either dimension, \textsf{\small Small, Medium} and \textsf{\small Large} labels are assigned to $\{1, 2, 3\}$ in order while \textsf{\small Stop} and \textsf{\small Swing} labels are assigned to $\{0, 4\}$. 
Fig.~\ref{fig:Specializedplanner} illustrates the sequentially composed center-of-mass (CoM) sagittal phase-space trajectory of a 20-step walking process.
Fig.~\ref{fig:ActionAndState} illustrates the discrete environment and system contact actions, and the corresponding keyframe states. 
At the low level, four locomotion modes (i.e., PIPM, PPM, MCM, SLM) are alternated according to the high-level decisions\footnote{For simplification, our model implements a simlified version of the above model with assumptions of a constant CoM acceleration and zero values of the states $\varphi$ and $\theta$.}. These high-level decisions are sent to the low-level motion planner one walking step ahead, i.e., a one-walking-step horizon. By inspecting the discrete sequences, we can verify that all the system contact actions and keyframe states respond to the environmental actions correctly, i.e., all the task specifications are satisfied. 
Fig.~\ref{fig:motionSnapshot} illustrates dynamic motion snapshots and continuous kinematic trajectories of the vertical CoM, foot, and hand positions. An accompanying video about the WBDL behaviors is available at \url{https://youtu.be/BdxYCmhRIMg}. 

\subsection{Case II: locomotion with ground sliding and hopping behaviors}
When the robot maneuvers through a narrow passage, an ordinary locomotion mode (e.g., \textsf{\small walk} and \textsf{\small brachiation}) will not work anymore due to the confined height. As such, a natural solution is to use a ground sliding mode: the robot crouches and slides with two feet through this constrained space as shown in Fig.~\ref{fig:slidingMotion}. The two arms are placed at a low position to avoid the contact with the ceiling. As shown in the bottom right subfigure of Fig.~\ref{fig:slidingMotion}, there are two multi-contact transition phases before and after the sliding phase (see the gray trajectory segments). We assume a constant negative CoM acceleration during the sliding phase, and thus the phase space trajectory of the sliding phase is a parabola. The low-level locomotion model corresponds to the mode (f) in Section~\ref{subsec:low-level-imp}.

When a crack in the terrain and a high ceiling occur simultaneously, the robot can not grasp the overhead support any longer. To maneuver forward successfully, the robot has to leap over the unsafe region as shown in Fig.~\ref{fig:hopping-motion}. Thus, a hopping phase will be executed with no contact with the environment. A constant CoM sagittal velocity shows up in the sagittal phase portrait while a parabola appears in the vertical position trajectory of Fig.~\ref{fig:hopping-motion}(b). The keyframe state of this hopping motion is chosen to be the center of the horizontal line segment in the sagittal phase-space. The lateral velocity is set to zero to avoid a lateral drift. 
The state will stay at the red dot in Fig.~\ref{fig:hopping-motion}(d) during the hopping motion. Since the robot locomotes forward, the lateral phase portrait in Fig.~\ref{fig:hopping-motion}(d) behaves like a limit cycle but non-periodically due to the rough terrain. The low-level locomotion model corresponds to the mode (e) in Section~\ref{subsec:low-level-imp}. Via introducing specific locomotion modes, our planner is capable of handling emergency-motivated scenarios.

\subsection{Case III: locomotion replanning strategy}
When the robot is already leaping in the air and detects a sudden change from an ordinary terrain to a cracked terrain, it has to replan its contact action and locomotion mode to accommodate this sudden change on the fly. The robot will execute a replanning process to ask the high-level planner for a new decision, i.e., grasp the ceiling support and swing the robot's body over the cracked region as shown in Fig.~\ref{fig:replanning}. Otherwise, the robot will fail to locomote. The top subfigure of Fig.~\ref{fig:replanning} shows a decision sequence including the replanning process. The three rows represent environment actions, locomotion modes, and keyframe states, respectively. The second column with three dotted blocks is the decision before the replanning process and not executed yet until the next walking step. The third column represents the replanned decision which is executed in response to the sudden environmental action change. This replanning process in the phase-space is illustrated in Fig.~\ref{fig:planning_sequence}. Let us consider a more challenging terrain scenario (i.e., \textsf{\small terrainCrack\text{-}highCeiling} as described in Section~\ref{sec:WBDLSpec}), where there is a cracked terrain with a high ceiling. Then the robot can not replan by grasping the overhead support anymore, and the locomotion process will fail inevitably. To avoid this situation, the \textsf{\small terrainCrack\text{-}highCeiling} environment action is not allowed to occur consecutively in our environment specification $S_{\rm e\text{-}1}$.

Besides the above replanning strategy in response to environment changes, there is another replanning strategy embedded in the reachability control library. When the robot state is perturbed to be out of the reachability region, i.e., the winning set currently being executed, the robot can not reach the robustness margin of the targeted keyframe. Then the robot calls for a new reachability controller in the library that covers the current perturbed state and uses this replanned controller to reach a new keyframe goal. Overall, the replanning process determines which new controller to call in the control library.

\begin{figure}[t]
 \centering
   \includegraphics[width=0.95\linewidth]{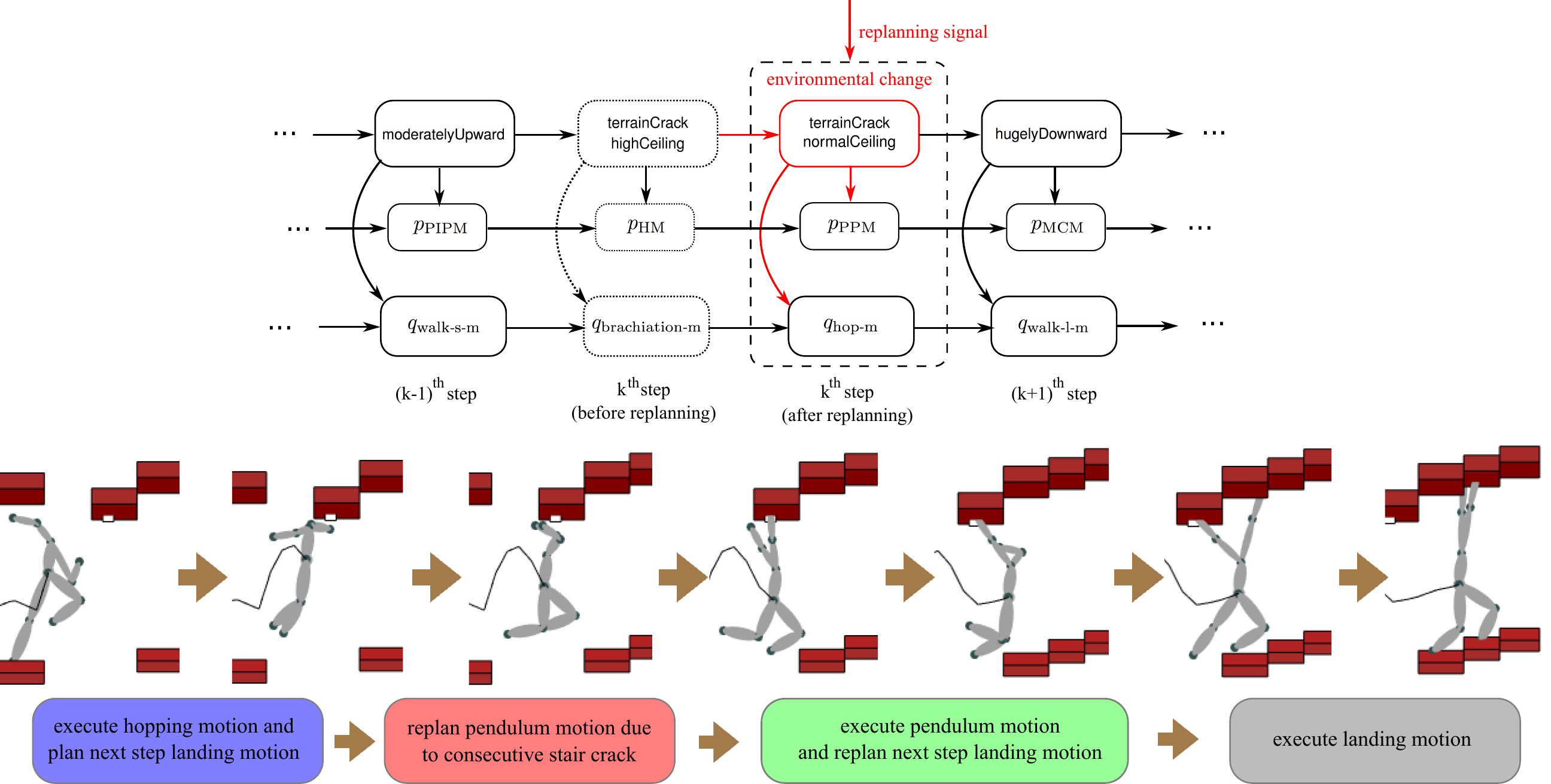}
 \caption{\captionsize Replanning in response to a sudden environment event change. Suppose that during the flight phase, the robot finds out that the next terrain is cracked. Accordingly, the robot triggers its replanning process by changing the locomotion mode to the prismatic pendulum model, i.e., grasping the overhead support, swinging the body over the second terrain crack region, and then landing on the non-cracked terrain. The top subfigure shows the decision sequence in the symbolic task planner. The bottom subfigure shows the snapshot sequence of the locomotion process.}
\label{fig:replanning}
\end{figure}


\subsection{Case IV: Validation of the robust reachability controller}
This case study evaluates the performance of the synthesized controller given a keyframe robust reachability goal. Let us first consider the prismatic inverted pendulum model (PIPM). Assume that we have a sagittal CoM state vector $\boldsymbol{\xi} = (x, v_x)^T$ and two consecutive locomotion modes denoted by ${\rm PIPM}_1$ and ${\rm PIPM}_2$, respectively. The PIPM dynamics in Eq.~(\ref{eq:PIPM}) are reformulated as the CoM dynamics below 
\begin{align}\label{eq:PIPM-simplified}
\begin{pmatrix}
\dot{x}(\zeta)\\
\dot{v}_x(\zeta)
\end{pmatrix}
= 
 \begin{pmatrix}
   v_x(\zeta)\\[2mm]
  \omega_{\rm PIPM}^2 (x(\zeta) - x_{\rm foot})
 \end{pmatrix},
\end{align}
where we assume zero torso angular momentum $(\tau_x, \tau_y) = \boldsymbol{0}$ and a predefined foot placement position $x_{\rm foot}$ for simplicity. The continuous control input $\omega_{\rm PIPM} \in [\omega_{\rm nominal} - \delta \omega, \omega_{\rm nominal} + \delta \omega]$, where $\delta \omega$ is a predefined bound. Note that the parameters $x_{\rm foot, 1}, x_{\rm foot, 2}$, and $\omega_{\rm nominal}$ are determined by the high-level symbolic task planner.
%

Let us define two nominal keyframe states $\boldsymbol{q}_{\rm initial} = \big(\dot{x}(\zeta_0), v_x(\zeta_0)\big) = (0 \text{m}, 0.5 \text{m/s})$ and $\boldsymbol{q}_{\rm final} = \big(\dot{x}(\zeta_{\rm final}), v_x(\zeta_{\rm final})\big) = (0.5 \text{m}, 0.6 \text{m/s})$\footnote{In this paper, $(a, b)$ represents a vector of two values $a$ and $b$ while $[a, b]$ represents an interval bounded by $a$ and $b$.} determined from the high-level planner. The goal is to solve the PIPM closed-loop phase-space trajectories starting from the initial robustness margin set $\mathcal{B}_{\epsilon_1}(\boldsymbol{q}_{\rm initial})$ and reaching the final robustness margin set $\mathcal{B}_{\epsilon_2}(\boldsymbol{q}_{\rm final})$ as defined in Def.~\ref{def:robustSetSimp}.
To this end, we synthesize a controller to determine the realizability of a keyframe transition in one walking step for $\mathcal{TS}_{\rm OWS}$ and generate a control strategy $\Omega: \Xi_{\rm OWS}\to 2^{\mathcal{A}_{\rm OWS}}$ if it is realizable.
The intermediate robustness margin set $\mathcal{B}_{\rm inter}$ between the locomotion modes ${\rm PIPM}_1$ and ${\rm PIPM}_2$ is defined by Eq.~(\ref{eq:intermediateCondition_1}).
%
%
Locomotion mode switching is only allowed when the state is within $\mathcal{B}_{\rm inter}$. Overall, the controller synthesis of one walking step is composed of three steps: first, the CoM trajectory starts from $\mathcal{B}_{\epsilon_1}(\boldsymbol{q}_{\rm initial})$ and moves towards $\mathcal{B}_{\rm inter}$; second, the state reaches $\mathcal{B}_{\rm inter}$ and switches the locomotion mode; third, the CoM state reaches $\mathcal{B}_{\epsilon_2}(\boldsymbol{q}_{\rm final})$. To reach $\mathcal{B}_{\epsilon_2}(\boldsymbol{q}_{\rm final})$, the conditions in Eqs.~(\ref{eq:initial-bounds})-(\ref{eq:intermediateCondition_1}) need to be satisfied by propagating the PIPM dynamics forward under bounded state disturbances and a bounded control input $\omega_{\rm PIPM}$. 
For this example, we assign the initial and final robustness margins of $\mathcal{B}_{\epsilon_1}(\boldsymbol{q}_{\rm initial})$ and $\mathcal{B}_{\epsilon_2}(\boldsymbol{q}_{\rm final})$ as $\delta \zeta_{\epsilon_1} = 0.05, \delta \sigma_{\epsilon_1} = 0.002$ and $\delta \zeta_{\epsilon_2} = 0.05, \delta \sigma_{\epsilon_2} = 0.006$, respectively.

As to the underlying continuous locomotion subsystem $\mathcal{SS}_{\rm OWS}$, we assign the one-walking-step state space $\Xi_{\rm OWS}=\bigcup_{p\in\mathcal{P}_{\rm OWS}}\Xi_{p}$ where $\mathcal{P}_{\rm OWS} = \{{\rm PIPM}_1, {\rm PIPM}_2\}$, $\Xi_p=[-0.1\text{m},0.7\text{m}]\times[0.1\text{m/s},1.2\text{m/s}]$ for all $p\in\mathcal{P}_{\rm OWS}$, the initial state set $\Xi_{\rm OWS,0}=\{\boldsymbol{\xi}:|\mathcal{Z}_{p_1, \sigma}(\boldsymbol{\xi})| \leq \delta \sigma_{\rm bound, init}\}$, the control space $\mathcal{U}_{\rm OWS}=[2,4]$. To construct an inter-sampling finite abstraction $\mathcal{TS}_{\rm OWS, INT}$, we uniformly discretize the state space $\Xi_{\rm OWS}$ with a granularity $[0.005 \text{m},0.005 \text{m/s}]$ and sample the control space $\mathcal{U}_{\rm OWS}$ with a $0.02\text{rad/s}$ granularity, resulting in a sampled finite control set $\widehat{\mathcal{U}}_{\rm OWS}=0.02\mathbb{Z}\cap[2,4]$. Let $\mathcal{A}_{\rm OWS}=\widehat{\mathcal{U}}_{\rm OWS}^{[0,\delta \zeta]}$ be a piece-wise trajectory with zero-order-hold values in $[0, \delta \zeta]$, where $\delta \zeta = 2$ ms is the time duration for each state in $\mathcal{T}_{\rm OWS}$. The system dynamics in (\ref{eq:PIPM-simplified}) are subject to additive disturbances bounded by $D_r=(0.05\text{m};0.1\text{m/s})$, i.e., position and velocity disturbances, respectively. 
Given the PIPM parameters above, we synthesize a reachability controller of $\mathcal{TS}_{\rm OWS, INT}$ and the computed winning sets 
are shown in Fig.~\ref{fig:reachabilityPIPM12}. As the result shows, the one-walking step reachability is realizable as long as the winning set overlaps (at least partially) the initial and final robustness margin sets. Five simulated trajectories under randomly-sampled bounded disturbances are shown as the black lines. Fig.~\ref{fig:winsets_uncertainlevels} evaluates the changing size of the winning set under different levels of the disturbance. The winning set shrinks as the disturbance set increases because the  synthesized controller needs to reach the goal robust set against a larger set of disturbances.

\begin{figure}[t]
    \centering
    \subfigure[Controlled trajectories of PIPM one-walking step.]
    {
    \includegraphics[scale=0.025]{walk_pipm2pipm2.pdf}
    \label{fig:reachabilityPIPM12}
    }
    \subfigure[Winning sets under different levels of disturbances.]
    {
    \includegraphics[scale=0.17]{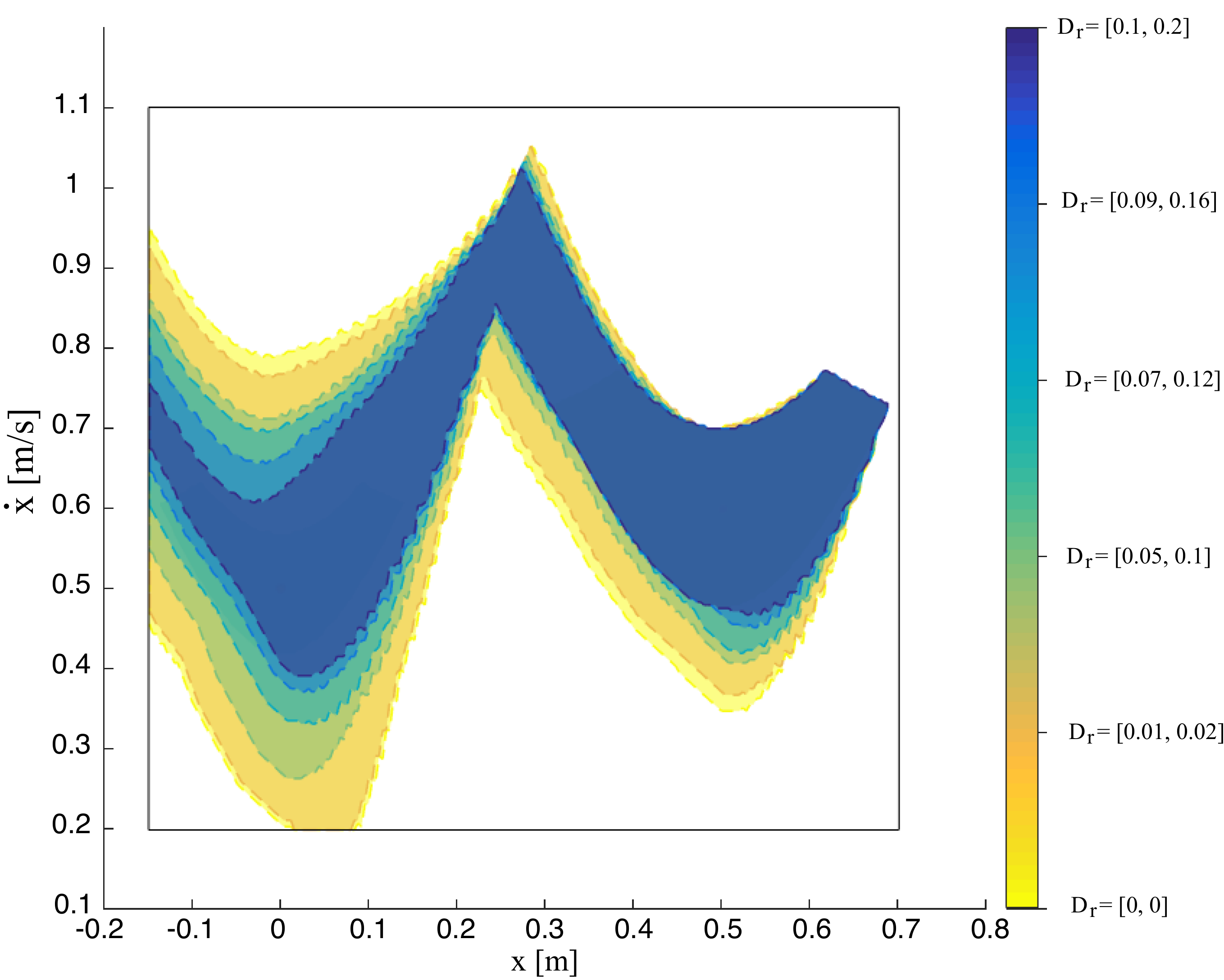}
    \label{fig:winsets_uncertainlevels}
    }
    \caption{\captionsize The additive disturbances to the dynamics are bounded by $D_r=(0.05\text{m}; 0.1 \text{m/s})$ in the subfigure (a). The shaded yellow region represents the winning set. The black trajectories are the five closed-loop trajectories simulated in five trials. The blue trajectory represents a trial suffering a large disturbance, i.e., a velocity jump in the phase-space. Since the disturbed state is still in the winning set, the CoM trajectory is guaranteed to reach the final robustness margin set. In subfigure (b), different levels of bounded disturbances are modeled in the computation of the winning sets. Naturally, a larger magnitude of the disturbance results in a smaller winning set.}
    \label{fig:phase-space-reachability}
\end{figure}

We design reachability controllers for all the combinations of the locomotion mode set $\mathcal{P}$. Consider another locomotion mode transition from the PIPM to the prismatic pendulum model (PPM). The PPM dynamics in Eq.~(\ref{eq:accel2}) are reformulated as follows
\begin{align}\label{eq:PPM-simplified}
\begin{pmatrix}
\dot{x}(\zeta)\\
\dot{v}_x(\zeta)
\end{pmatrix}
= 
 \begin{pmatrix}
   v_x(\zeta)\\[2mm]
  - \omega_{\rm PPM}^2 (x(\zeta) - x_{\rm hand})
 \end{pmatrix},
\end{align}
%
%
with the assumption of $\tau_x = \tau_y = 0$ and a predefined hand contact position $x_{\rm hand}$. Other parameters are defined in Table~\ref{table:PIPM_PPM_transition}.
To evaluate the robustness performance of the synthesized controller, we examine the success rate of reaching the goal robust set through 50 simulation tests under different granularities and bounded disturbances. In Fig.~\ref{fig:traj50_robust}, each trial is run for the one walking step with the PIPM-PPM mode pair. The exerted disturbance in the simulation is the same as the one used in the controller synthesis process, i.e., $D_r = (0.15 \text{m}, 0.3 \text{m/s})$. As shown in Fig.~\ref{fig:traj50_robust}, all the trials reach the final robustness margin successfully. This agrees with the correctness guarantee by the one-walking-step robust reachability property of Theorem~\ref{theorem:OWSReachability}.

\begin{table}[t]
 \caption{\normalsize Parameters of the PIPM-PPM mode transition.}
  \begin{center}\vspace{-5mm}
   \begin{tabular}{c||c|c|c} \hline\hline
    {\bf Parameters}        & {\bf Values}     & {\bf Parameters} &          {\bf Values}             \\     \hline \hline
    initial keyframe $\boldsymbol{q}_{\rm initial}$ & $(0 \text{m}, 0.5 \text{m/s})$
                & final keyframe $\boldsymbol{q}_{\rm final}$ & $(0.6 \text{m}, 1.7 \text{m/s})$\\ 
    \hline 
    initial tangent bound $\delta \sigma_{\rm bound, init}$ & $0.002$ 
                & initial cotangent bound $\delta \zeta_{\rm bound, init}$ & $0.05$ \\ 
    \hline 
    final tangent bound $\delta \sigma_{\rm bound, final}$ & $0.06$ 
                & final cotangent bound $\delta \zeta_{\rm bound, init}$ & $0.005$ \\            
    \hline 
    mode set  $\mathcal{P}_{\rm OWS}$    &  $\{{\rm PIPM,PPM}\}$
                & disturbance range $D_r$ & $(0.15 \text{m}, 0.3 \text{m/s})$\\ 
    \hline 
    OWS state space $\Xi_p$ & $[-0.1 \text{m},0.7 \text{m}]\times[0.1 \text{m/s},1.8 \text{m/s}]$
                & control space $\mathcal{U}_{\rm OWS}$ & $[2 \text{rad/s},4 \text{rad/s}]$\\ 
    \hline 
   \end{tabular}
  \end{center}
  \label{table:PIPM_PPM_transition}
\end{table}

\begin{figure}[t]
    \centering
    \subfigure[50 simulated trajectories using the synthesized controller.]{
    \includegraphics[scale=0.16]{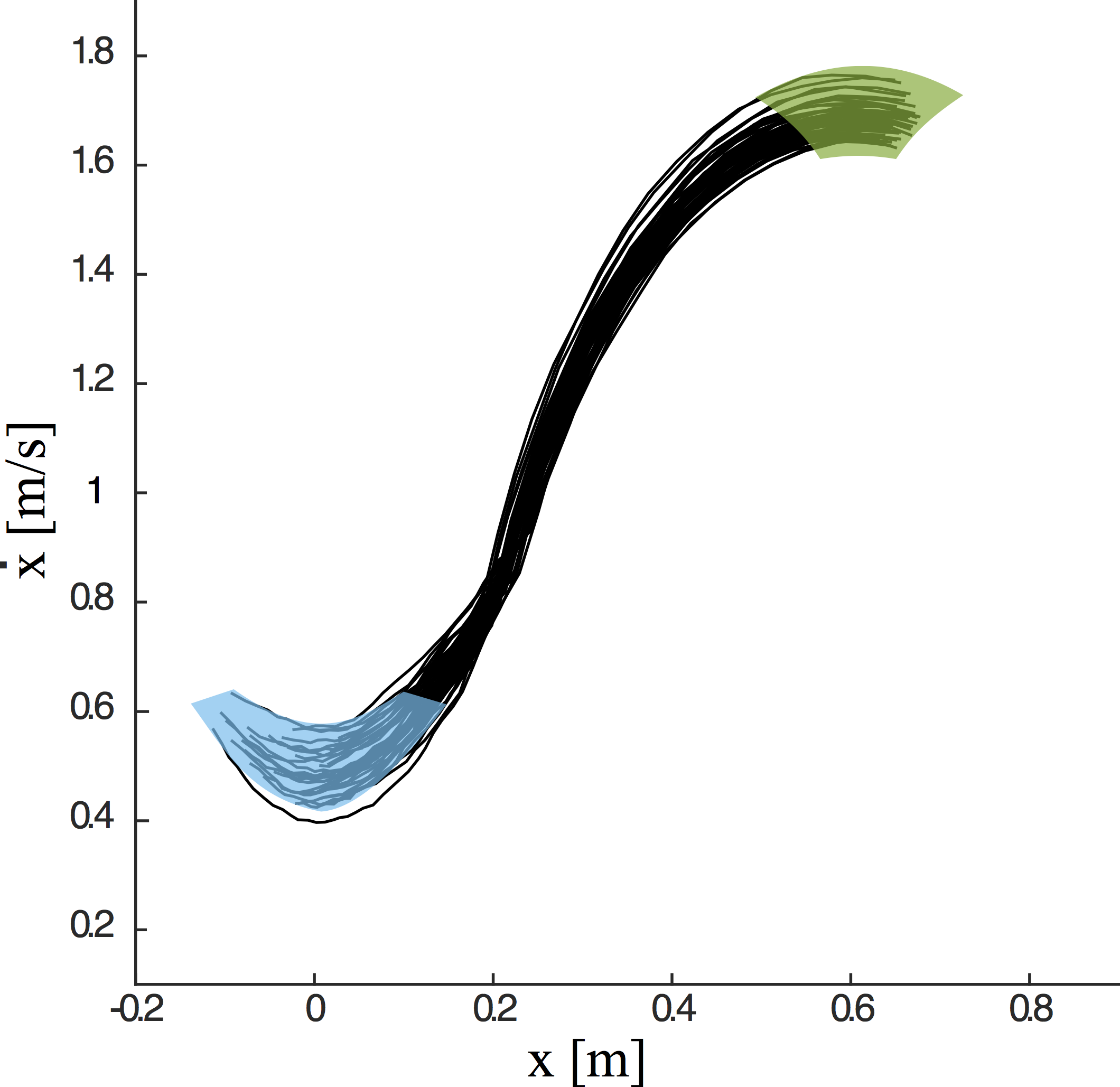}
    \label{fig:traj50_robust}
    }
    \subfigure[Simulation success rate under various granularities and disturbances.]{
    \includegraphics[scale=0.20]{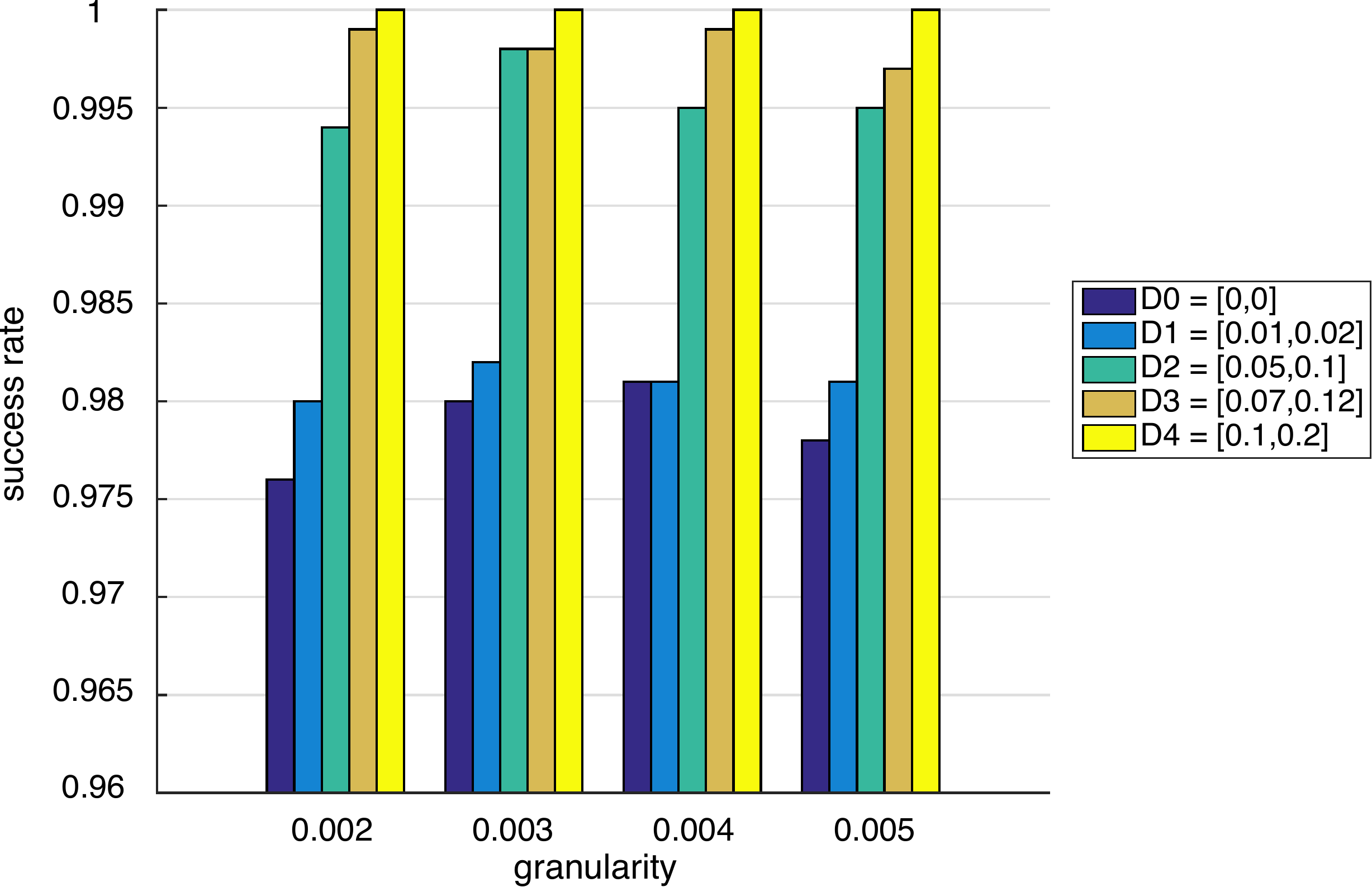}
    \label{fig:success_rate}
    }
    \caption{Success rate of the simulations under varying granularities and disturbances. In subfigure (a), the system is subjected to disturbances bounded by $D_r=(0.15\text{m}; 0.3\text{m/s})$. All the 50 simulation trails can reach the goal robustness margin set successfully. In subfigure (b), we run 1000 trials for each case with a specific granularity and a bounded disturbance. The disturbance exerted in the simulation remains the same, i.e., $D_r = (0.1\text{m}, 0.2\text{m/s})$.}
\end{figure}

We evaluate the effect of the discretization granularity and the magnitude of disturbances used in the controller synthesis process as shown in Fig.~\ref{fig:success_rate}. Given each controller synthesized using a specific granularity and for a specific disturbance bound, we simulate 1000 trials with the bounded disturbance $D_r = (0.1\text{m}, 0.2\text{m/s})$. Fig.~\ref{fig:success_rate} shows four sets of simulation results for different granularities ranging from 0.002 to 0.005. For each set of simulations, the success rate increases as the modeled disturbance in the controller synthesis increases, and it reaches 100\% when the modeled disturbance matches the actual disturbance $D_r$ used in the simulation. This is consistent with our expectation. Let us inspect the figure from another perspective. If we compare the results for different granularities with a specific disturbance set $D_i$ ($i=0,1,2,3,4$), the success rate almost remains the same. This is because when constructing the abstraction for the robust reachability analysis, we have taken into consideration the effects of approximation errors caused by different discretization granularities, by using non-deterministic transitions that over-approximate the dynamics of the system. In addition, we observe that the success rates for all the synthesized controllers are greater than $97\%$, even in the case no disturbance is considered in the controller synthesis. This can again be interpreted by the over-approximation used in the abstraction. Nonetheless, as shown in the simulations, to achieve 100\% correctness guarantee, the modeled disturbance has to be larger than (or at least match) the actual disturbance in the simulation. Moreover, under the same disturbance $D_r$, the nominal phase-space planner with a fixed open-loop control input only achieves a success rate of $29\%$. This huge discrepancy in success rate clearly shows the advantage of using an abstraction-based feedback controller over an open-loop phase-space planner.

\subsection{Case V: Integrated multi-step locomotion via the reachability control library}

This case evaluates an integrated multi-step locomotion example with the robust finite transition system $\mathcal{TS}_{\rm OWS}$, the inter-sampling finite abstraction $\mathcal{TS}_{\rm OWS, INT}$, and the replanning strategy. Assume that the decision of the task planner renders a locomotion mode sequence involving the PIPM, PPM, and MCM modes below,
\begin{align}\nonumber
{\rm PIPM\rightarrow PIPM \rightarrow PPM\rightarrow PIPM\rightarrow MCM\rightarrow PIPM \rightarrow PIPM}
\end{align}
To enable the initial and final keyframe robustness margin sets to cover a sufficiently larger phase space, we extend the default $3\times3$ keyframe grid to a $5\times5$ keyframe grid for each mode. This allows the reachability controllers to be applicable to a larger set of keyframe states. For each locomotion mode pair, we synthesize all the feasible controllers that reach the final keyframe robustness margin set under a bounded disturbance. We enumerate all the combinations of the allowable locomotion mode pairs and generate all the reachability control policies offline. These controllers are saved as a control library and are executed at runtime according to the high-level decision and measured states under bounded disturbances. 

Parameters of constructing the inter-sampling finite abstraction $\mathcal{TS}_{\rm OWS, INT}$ are defined as follows. The controller synthesis and execution process use the same disturbance bound $D_r= (0.05 \text{m}; 0.1 \text{m/s})$. The full discretized state space is $\Xi_{\rm full}=[-0.2\text{m},3.8\text{m}]\times [0.2\text{m/s},1.9\text{m/s}]$ with a granularity $(0.003 \text{m},0.003 \text{m/s})$. The local state space of each walking step is chosen so that it is sufficiently large to cover the space around the two keyframe states. A time step $\delta\zeta=0.02$s is used for the abstraction construction of each walking step. The control inputs for PIPM, PPM and MCM satisfy $\omega_{\rm PIPM} \in [2, 4]$, $\omega_{\rm PPM}\in [2, 4]$ and $\omega_{\rm MCM} \in [1, 3]$. 
We obtain the sets of sampled control values by a granularity of 0.02. The robustness margins of the phase space manifolds are $\delta\sigma_{\rm PIPM}=0.002$, $\delta \zeta_{\rm PIPM}=0.002$; $\delta\sigma_{\rm PPM}=0.04$, $\delta \zeta_{\rm PPM}=0.003$; $\delta\sigma_{\rm MCM}=0.15$, $\delta \zeta_{\rm MCM}=0.9\times 10^{-5}$.

The computational time for constructing abstractions is around $30$s on average, and $5$s to $15$s for synthesizing a reachability controller corresponding to each keyframe pair, depending on the number of states and transitions of the abstraction. Since we synthesize 625 (i.e., $25\times25$) reachability controllers for each walking step, the time to generate them is approximately $90$ mins. In our simulation of six consecutive walking steps, all the local reachability control strategies are patched together to cover the overall state space. The time for simulating a single closed-loop walking trajectory is around $2$s.
As the results show in Fig.~\ref{fig:integrated_control}, we simulate six different trials with different initial conditions, i.e., starting from different initial robustness margin sets. Each locomotion trajectory is guaranteed to reach one of the robustness margin sets at the next walking step via using the reachability controller from the control library. In particular, a trial is tested to evaluate the replanning strategy when the state is perturbed out of the winning set.

\begin{figure}[t]
 \centering
   \includegraphics[width=0.95\linewidth]{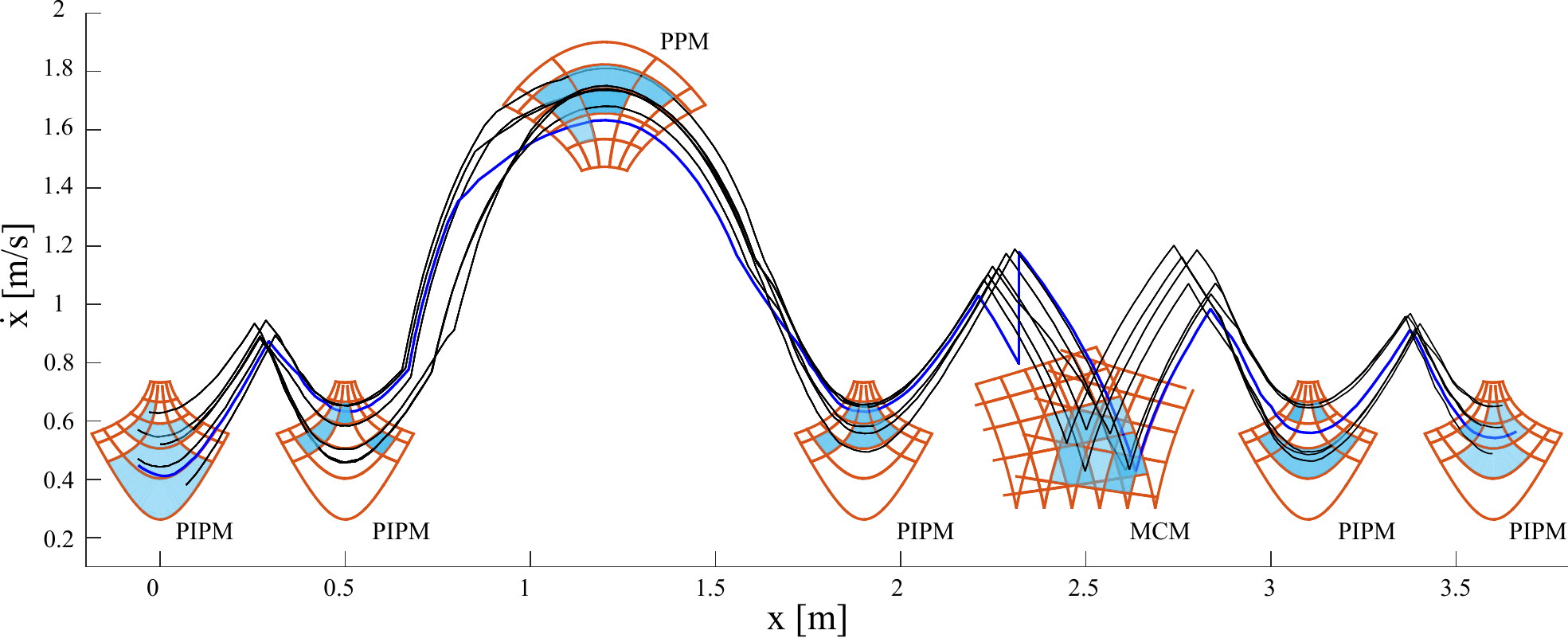}
 \caption{\captionsize Integrated phase-space trajectories of multi-walking step simulations under bounded disturbances. The replanning strategy is evaluated with a trial (see the blue trajectory) exerted with a velocity disturbance larger than the modeled disturbance in the MCM mode (around the position $x = 2.3 $ m). In this case, the state is perturbed out of the winning set of the currently used reachability controller. A replanning signal is triggered, and the planner searches within the control library for a new winning set (together with a new reachability controller) that covers the perturbed state. Then the perturbed state will use that new controller to reach a new robustness margin set for continuous locomotion maneuvering.}
\label{fig:integrated_control}
\end{figure}

 \section{Discussions and Future Work}
\label{section:discussions}
%
%
\subsection{Low-level uncertainties}
This paper proposes a hierarchical approach to the task and motion planning of dynamic locomotion in complex environments. We achieve robustness against a general, bounded disturbance by synthesizing a middle-layer robust reachability controller with robustness margins to accommodate low-level uncertainties.
Undoubtedly, a variety of low-level uncertainties can come from time delays, actuation limits, unmodeled dynamics, state estimation, and measurement error from the environment. These uncertainties severely deteriorate the execution success rate of the high-level planner, in particular when the robot performs highly agile motions in complex and unstructured environments. In addition, the abstraction methods can induce approximation errors between the high-level and low-level planners. Although not directly dealing with these low-level uncertainties and abstraction approximation errors, the keyframe-based robustness margin proposed in this paper can be viewed as an abstract representation of these uncertainties in the center-of-mass (CoM) state space. As long as a mapping can be established between these low-level uncertainties and the CoM phase-space deviations from the nominal trajectory, these uncertainties can be handled indirectly by the proposed reachability controller at run-time. Additionally, a replanning strategy is designed to handle large uncertainties that are not explicitly modeled in the reachability controller. 
In the future, abstraction refinement [\cite{nilsson2014incremental}] will be inspirational for designing a model abstraction with a proper granularity. More importantly, implementing the proposed high-level decision-making algorithms in the dynamic simulation and real hardware [\cite{kim2016stabilizing, luo2017locomotion}], and evaluating the robustness performance against low-level uncertainties will be our main upcoming task. 




\subsection{System and environment assumption relaxation}
%

To make the proposed hierarchical planning approach applicable to locomotion tasks in more complex and cluttered environments, it is important to relax the assumptions and approximations of the environment and model more realistic scenes. For instance, how to formally design recovery strategies for slippery terrains (i.e., with friction coefficient inaccuracies), large tilting angles, and swing foot obstacle collision is a practically meaningful topic.

Our current planner assumes that all limb contacts switch synchronously. To relax this conservative assumption, we will explore the asynchronous contact switching strategy in the future. This relaxation opens up the opportunity for designing more natural and diverse locomotion contact behaviors. From a more general perspective, contact actions and keyframe states may exhibit probabilistic features. Incorporating probabilistic models, such as Markov decision process (MDP) [\cite{platt2004manipulation, fu2014probably, feng2015controller}], into the high-level decisions will be a promising direction. Accordingly, studying probabilistic correctness and completeness will be of our interest.


\subsection{Generalization to complex tasks}
Generalizing the proposed planning framework to more realistic locomotion tasks is of practical importance, in particular when robots are unleashed into the real world. Some more practical locomotion tasks include walking while carrying a payload, walking alongside human teammates, dynamically interacting with a human during motion [\cite{alonso2018reactive}], and multi-robot coordination [\cite{da2016combined}]. To this end, how to design an automatic method for generating locomotion primitives of diverse tasks becomes important. Also, allocating computing resources efficiently among different planning layers is an essential topic. A mission planner will be needed to operate at a more abstract level to make decisions on task allocation, coordination, and synchronization. A key problem is how to properly design integrated, scalable, and reactive mission and motion planners [\cite{da2016combined}] for legged robots to accomplish collaborative tasks in dynamic and unstructured environments.

At the individual robot level, our motion planner is designed for the three-dimensional case, although the demonstrated locomotion tasks are primarily straight walking. 
In the future, we will incorporate steering models [\cite{zhao2017robust}] such that the locomotion behaviors are extendable to complex 3D motions with steering capabilities. An advantage of our planning framework is to use simplified models which allow us to efficiently compose multiple locomotion modes and achieve dynamic and complex locomotion behaviors in constrained environments. The high-level symbolic planner automates this sequential composition process and guarantees the formal correctness of the overall planning framework.

\textit{An application of the proposed whole-body dynamic locomotion methodology in the constrained environment is the following:} The US Defense Advanced Research Projects Agency (DARPA) created a Subterranean Challenge [\cite{DARPASubT}] aiming at augmenting underground operation capabilities.
\textit{``The Challenge aims to explore new approaches to rapidly map, navigate, and search underground environments ... and propose novel methods for tackling time-critical scenarios through unknown courses in mapping subsurface networks and unpredictable conditions, which are too hazardous for human first responders''}.
Our proposed hierarchical decision-making approach for whole-body dynamic locomotion in constrained environments raise the importance of decision-making algorithms with formal guarantees for robots as complex as humanoid robots, a research topic of increasing importance as these robots begin to move out of the laboratories and work outdoors. 

\subsection{Planning horizon}

Making planner decisions with a sufficiently long predictive horizon has great potential to enable intelligent and robust behaviors in complex and dynamically changing environments [\cite{egerstedt2018robot}]. Our task planner has a one-walking-step horizon and may sometimes result in myopic locomotion decisions. For instance, if the disturbance is so large that the robot can not recover within one walking step, our planner will execute a replanning process. However, a natural alternative is to design a recovery strategy over the next multiple walking steps, which is  commonly used in the recovery process of human locomotion. The downside of this strategy is the increase in computational complexity. Our planning process substantially relieves this computational burden by using the simplified locomotion models. In addition to this computational consideration, the design of the planning horizon should take into account the versatility of the locomotion behaviors. For instance, if the locomotion process is of high speed, being able to make predictions over a longer horizon will be advantageous. Overall, we should take into account the computational power and behavior versatility when designing the planning horizon. 
Algorithm~\ref{al:reachcontrol} is designed in a general form and should be extendable to the multiple-walking-step scenario. 


\section{Conclusions}
\label{sec:conclusion}
This paper employs temporal-logic-based formal methods to synthesize a high-level reactive task planner and designs a middle-level discrete control to achieve the one-walking-step robust reachability process for complex whole-body dynamic locomotion (WBDL) behaviors in constrained environments. A particular focus has been given to (i) the robustness of the keyframe state reachability with respect to bounded disturbances; (ii) the correctness of the top-down hierarchy from the high-level task planner to the low-level motion planner processes.

A diverse set of locomotion models are devised at the low-level to form a template library in response to various environmental events, including those adversarial ones such as cracked terrain, human appearance, and narrow passage. These adversarial events require specifically-designed locomotion modes to enable desired locomotion behaviors. Deviating from numerous existing studies primarily using a single simplified model for a specific locomotion task, our symbolic task planner focuses on integrating and unifying a variety of simplified models and achieves complex locomotion behaviors via sequential composition of trajectories. A key novelty of this task planner lies in solving the traditional contact planning problem via a two-player game. Contact decisions are determined according to the synthesized switching protocol in response to possibly-adversarial environment actions. 

As for the reachability control under disturbances, we propose a robust metric of the keyframe state and use it to design a robust finite transition system realized by the underlying reachability synthesis.
The proposed task and motion planner is validated through simulations of WBDL maneuvers in constrained environments. The performance of the reachability control is benchmarked via a series of synthesis and execution tests. We expect that this line of work acts as an entry point for the locomotion community to employ formal methods to verify and synthesize planners and controllers in legged and humanoid robots [\cite{kuindersma2016optimization, hereid20163d, ramezani2014performance}]. Evaluating the proposed framework on dynamic bipedal robots is one of our high-priority future works. 



\begin{funding}
This work was partially supported by the NSF Grant [\#1924978, \#1724360, \#1652113], Office of Naval Research (ONR) Grant [grant \#N000141512507], and partially supported by NSERC of Canada, the Canada Research Chairs program, and an Ontario MRIS Early Researcher Award.
\end{funding}

 
\begin{appendix_sec}
\appendix
\label{appen}
\section{Index to Multimedia Extensions}

\begin{tabular}{l*{6}{c}r}
Extension  & Type & Description \\
\hline
1 & Video & Reactive task and motion planning for whole-body dynamic locomotion  \\
\end{tabular}

\begin{figure}[t]
 \centering
   \includegraphics[width=0.85\linewidth]{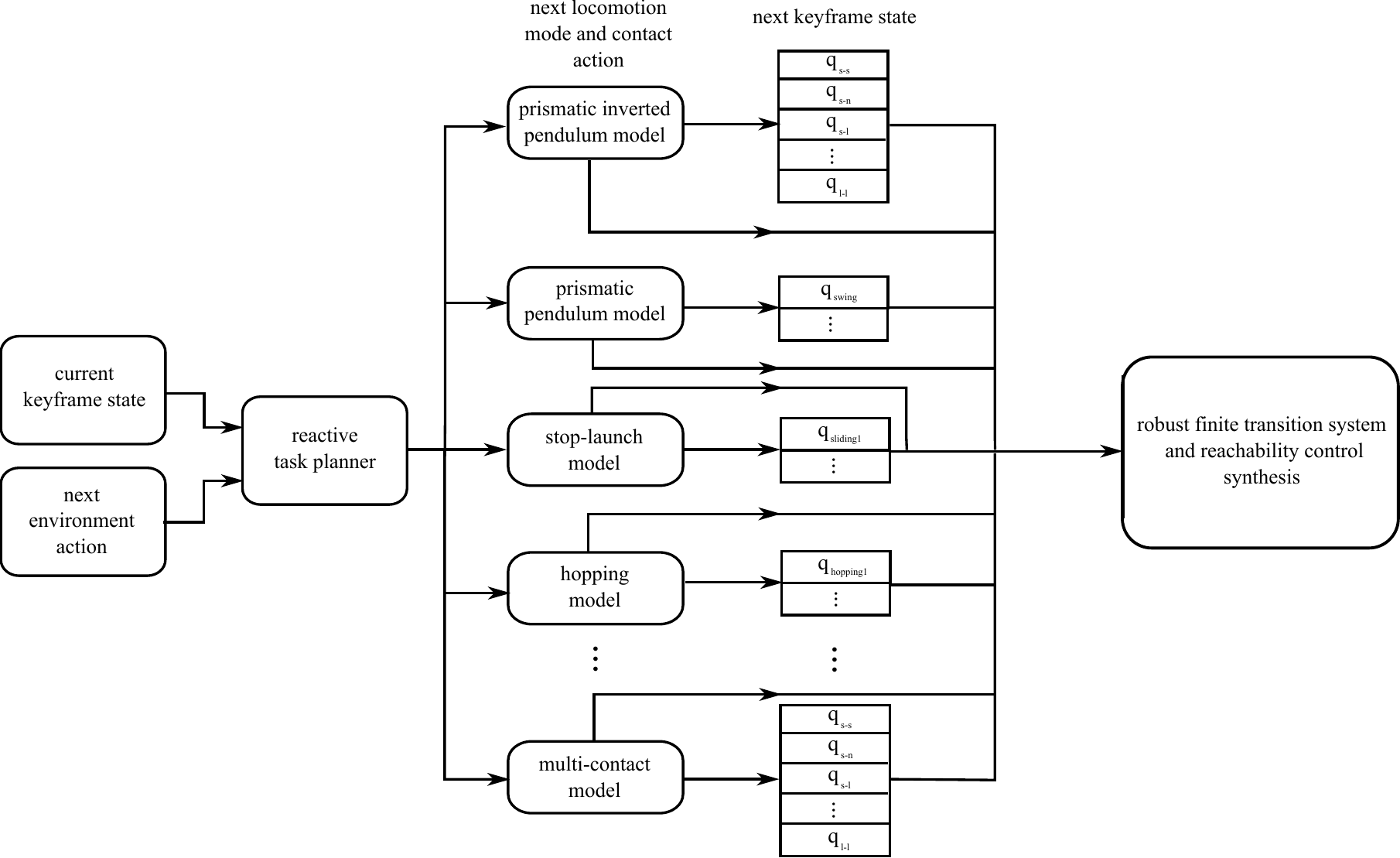}
 \caption{\captionsize An illustration of the top-down decision sequence of the high-level reactive task planner and middle-level reachability controller synthesis. It illustrates the relationship between the keyframe state, environment action, and system mode.
}
\label{fig:decision_sequence}
\end{figure}

\begin{figure}[t]
 \centering
   \includegraphics[width=0.65\linewidth]{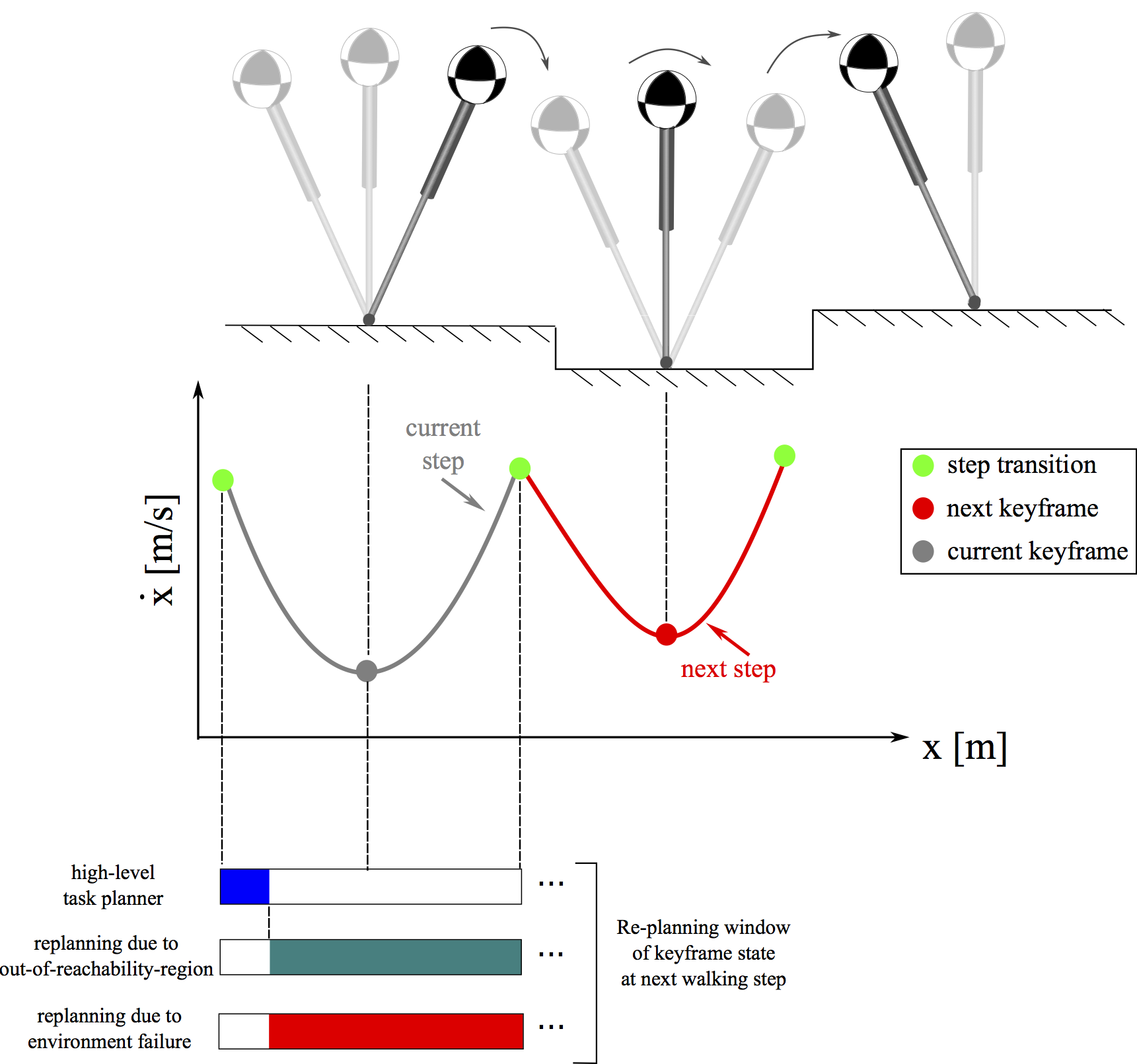}
 \caption{\captionsize Replanning timing for the next walking step. The high-level task planning for the next walking step is determined at the beginning of one walking step. Then during the remaining time of the current walking step (before switching to the next walking step), a replanning process can be triggered anytime if the state is out of the reachability region (i.e., the winning set) or the environment action change suddenly. This figure uses single-contact prismatic inverted pendulum model for illustration.
}
\label{fig:planning_sequence}
\end{figure}

\begin{algorithm}
\caption{Execution of the synthesized controller based on robust finite transition system $\mathcal{TS}_{\rm OWS}$}\label{al:execution_reachcontrol}
\begin{algorithmic}[1]
\STATE \textbf{procedure} ExecuteOWSReachabilityControl(keyframeState $\boldsymbol{q}_{\rm initial}$, $\boldsymbol{q}_{\rm final}$, locomotionMode $p_1$, $p_2$, contactConfiguration $s_{1}, s_{2}$, initialState $\boldsymbol{\xi}_{\rm init}$, transitionTime $\zeta_{\rm trans}$, environmentActionOfNextOWS $e_{\rm next\text{-}OWS}$)
\STATE choose winning sets of the first and second semi-steps $\mathcal{WIN}_{p_1}$ and $\mathcal{WIN}_{p_2}$ via ($p_{1}, p_{2}$), ($s_{1}, s_{2}$), and ($\boldsymbol{q}_{\rm initial}, \boldsymbol{q}_{\rm final}$).
\STATE set initial and final robustness margins ($\delta \zeta_{\rm bound, init}, \delta \sigma_{\rm bound, init}$) and ($\delta \zeta_{\rm bound, final}, \delta \sigma_{\rm bound, final}$)
\STATE set the intermediate robustness margin set $\mathcal{R}_{\rm inter}=\{\boldsymbol{\xi}':|\mathcal{Z}_{p_1,\sigma}(\boldsymbol{\xi}')| \leq \delta \sigma_{\rm bound, init}\wedge |\mathcal{Z}_{p_2,\sigma}(\boldsymbol{\xi}')| \leq \delta \sigma_{\rm bound, final}\}$.
\STATE initialize CoM state $\boldsymbol{\xi}_{\rm current} \leftarrow \boldsymbol{\xi}_{\rm init}$.
\STATE \textsf{\small isEnvironmentAbruptChange} $\leftarrow$ \textbf{false}
\WHILE {$\boldsymbol{\xi}_{\rm current} \notin \mathcal{R}_{\rm inter}$}
\STATE \textsf{\small isEnvironmentAbruptChange} $\leftarrow$ detectEnvironmentAbruptChange($e_{\rm next\text{-}OWS}$)
\IF{$\boldsymbol{\xi}_{\rm current} \in \mathcal{WIN}_{p_1}$ \textbf{and} \textsf{\small isEnvironmentAbruptChange} == \textbf{false}}
\STATE assign control input $u_{p_1}$ from $\mathcal{WIN}_{p_1}$ based on current state $\boldsymbol{\xi}_{\rm current}$
\ELSIF{$\boldsymbol{\xi}_{\rm current} \notin \mathcal{WIN}_{p_1}$ \textbf{and} \textsf{\small isEnvironmentAbruptChange} == \textbf{false}}
\STATE (\textsf{\small isAlternativeWinningSetFeasible}, $\mathcal{WIN}_{\rm alternative}$) $\leftarrow$ detectAlternativeWinningSet($p_1$, $\mathcal{TS}_{\rm OWS}$)
\IF{\textsf{\small isAlternativeWinningSetFeasible} == \textbf{true}}
\STATE $\mathcal{WIN}_{p_1} \leftarrow \mathcal{WIN}_{\rm alternative}$
\ELSE
\STATE \textbf{failure} occurs, \textbf{exit}, \textbf{replanning}, and \textbf{move} to Line 2 \hspace{1.6in} \COMMENT{no alternative winning set}
\ENDIF
\ELSIF {\textsf{\small isEnvironmentAbruptChange} == \textbf{true}}
\STATE \textbf{failure} occurs, \textbf{exit}, \textbf{replanning}, and \textbf{move} to Line 2 \hspace{1.6in} \COMMENT{environment abrupt change}
\ENDIF
\STATE send control command $u_{p_1}$ to robot for execution.  
\STATE measure the actual state $\boldsymbol{\xi}_{\rm next}$ at next time step
\STATE $\boldsymbol{\xi}_{\rm current} \leftarrow \boldsymbol{\xi}_{\rm next}, \zeta \leftarrow \zeta + \delta \zeta_{p_1}$ $\boldsymbol{\chi}_{\rm current} \gets$ generateLimbTraj($s_{p_1}, \boldsymbol{\xi}_{\rm current}$)
\ENDWHILE
\WHILE {$\boldsymbol{\xi}_{\rm current} \in \mathcal{R}_{\rm inter} $ and $\zeta < \zeta_{\rm trans}$}
\STATE send control command $u_{p_1}$ to robot/simulator for execution.  
\STATE measured actual state at next time step $\boldsymbol{\xi}_{\rm next}$
\STATE $\boldsymbol{\xi}_{\rm current} \leftarrow \boldsymbol{\xi}_{\rm next}, \zeta \leftarrow \zeta + \delta \zeta_{p_1}$, $\boldsymbol{\chi}_{\rm current} \gets$ generateLimbTraj($s_{p_1}, \boldsymbol{\xi}_{\rm current}$)
\ENDWHILE
\STATE $(\sigma, \zeta) \leftarrow \mathcal{Z}_{p_2}(\boldsymbol{\zeta}_{\rm current})$
\WHILE {$|\zeta - \zeta_{p_2}| > \delta \zeta_{\rm bound, final}$ \textbf{or} $|\sigma - \sigma_{p_2}| > \delta \sigma_{\rm bound, final}$}
\IF{$\boldsymbol{\xi}_{\rm current} \in \mathcal{WIN}_{p_2}$}
\STATE assign control input $u_{p_2}$ from $\mathcal{WIN}_{p_2}$ based on current state $\boldsymbol{\xi}_{\rm current}$
\ELSIF{$\boldsymbol{\xi}_{\rm current} \notin \mathcal{WIN}_{p_2}$}
\STATE (\textsf{\small isAlternativeWinningSetFeasible}, $\mathcal{WIN}_{\rm alternative}$) $\leftarrow$ detectAlternativeWinningSet($p_2$, $\mathcal{TS}_{\rm OWS}$)
\IF{\textsf{\small isAlternativeWinningSetFeasible} == \textbf{true}}
\STATE $\mathcal{WIN}_{p_2} \leftarrow \mathcal{WIN}_{\rm alternative}$
\ELSE
\STATE \textbf{failure} occurs, \textbf{exit}, \textbf{replanning}, \textbf{select} a new feasible goal set from the task planner, and \textbf{move} to Line 30 
\STATE \hspace{4.3in} \COMMENT{no alternative winning set}
\ENDIF
\ENDIF
\STATE send control command $u_{p_2}$ to the robot/simulator for execution.
\STATE measure actual state at next time step $\boldsymbol{\xi}_{\rm next}$
\STATE $\boldsymbol{\xi}_{\rm current} \leftarrow \boldsymbol{\xi}_{\rm next}, \zeta \leftarrow \zeta + \delta \zeta_{p_2}$, $\boldsymbol{\chi}_{\rm current} \gets$ generateLimbTraj($s_{p_2}, \boldsymbol{\xi}_{\rm current}$)
\STATE $(\sigma, \zeta) \leftarrow \mathcal{Z}_{p_2}(\boldsymbol{\zeta}_{\rm current})$.
\ENDWHILE
\RETURN a sequence of $(\boldsymbol{\xi}_{\rm current}, \boldsymbol{\chi}_{\rm current})$.
\end{algorithmic}
\end{algorithm}

\section{Linear temporal logic}
\label{subsec:LTL-preliminary}

Linear temporal logic is an extension of propositional logic that incorporate temporal operators. An LTL formula $\varphi$ is composed of atomic propositions $\pi \in AP$. The generic form of a LTL formula has the following grammar,
\begin{equation}\nonumber
\varphi ::= \pi \; \Big|  \; \neg \varphi \; \Big| \; \varphi_1 \wedge \varphi_2 \; \Big|\; \varphi_1 \vee \varphi_2 \; \Big| \; \bigcirc \varphi \;\Big|\; \varphi_1 \mathcal{U} \varphi_2,
\end{equation}
where the Boolean constants ${\rm true}$ and ${\rm false}$ are expressed by ${\rm false} = \neg {\rm true}$ and ${\rm true} = \varphi \vee \neg \varphi$, and we have the temporal operators $\bigcirc$ (``next"), $\mathcal{U}$ (``until"), $\neg$ (``negation") and $\wedge$ (``conjunction"). We can also define $\vee$ (``disjunction"), $\Rightarrow$ (``implication"), $\Leftrightarrow$ (``equivalence"). Another two key operators in LTL are $\Diamond$ (``Eventually") and $\Box$ (``Always"). We can interpret them $\Diamond \varphi := {\rm true}\; \mathcal{U} \varphi$ for ``Eventually" and $\Box \varphi := \neg \Diamond \neg \varphi$ for ``Always".

LTL formulae are interpreted over an infinite sequence of states $\boldsymbol{q}$. We define $\boldsymbol{q}_i = q_i q_{i+1} q_{i+2} \ldots$ as a run from $i^{\rm th}$ position. It is said that a LTL formula $\varphi$ holds at $i^{\rm th}$ position of $\boldsymbol{q}$, represented as $q_i \models \varphi$, if and only if $\varphi$ holds for the remaining sequence of $\boldsymbol{q}$ starting at $i^{\rm th}$ position.
The semantics of LTL are defined inductively as 
\begin{align}\nonumber
q_i  \models \neg \varphi \; &{\rm iff} \; q_i  \not\models \varphi\\\nonumber
q_i  \models \varphi_1 \wedge \varphi_2 \; &{\rm iff} \; q_i  \models \varphi_1 \wedge q_i \models \varphi_2\\\nonumber
q_i  \models \varphi_1 \vee \varphi_2 \; &{\rm iff} \; q_i \models \varphi_1 \vee q_i \models \varphi_2\\\nonumber
q_i  \models \bigcirc \varphi \; &{\rm iff} \; q_{i+1} \models \varphi\\\nonumber
q_i \models \varphi_1 \mathcal{U} \varphi_2\; &{\rm iff} \; \exists j \geq i, {\rm s.t.}  q_j \models \varphi_2 \\\nonumber
 & {\rm and} \;q_k \models \varphi_1, \forall i \leq k \leq j
\end{align}
In these definitions, the notation $\bigcirc \varphi$ represents that $\varphi$ is true at the next ``step" (i.e., next position in the sequence), $\Box \varphi$ represents $\varphi$ is always true (i.e., true at every position of the sequence), $\Diamond \varphi$ represents that $\varphi$ is eventually true at some position of the sequence, $\Box \Diamond \varphi$ represents that $\varphi$ is true infinitely often (i.e., eventually become true starting from any position), and $\Diamond \Box \varphi$ represents that $\varphi$ is eventually always true (i.e., always becomes true after some point in time in the sequence) [\cite{baier2008principles}].

\section{Phase-Space Manifold}
\label{appen:PSManifold}
Closed-form solutions of the phase-space manifolds are required to define the robustness margin sets in Def.~\ref{def:robustSetSimp}. Besides the ones proposed for the prismatic inverted pendulum model (PIPM) in Propositions~\ref{prop:PSMTangent} and~\ref{prop:PSMCotangent}, this subsection proposes the closed-form solutions of additional locomotion modes, including the prismatic pendulum model (PPM) and the multi-contact model (MCM) as defined in Section~\ref{subsec:low-level-imp}.
\begin{proposition}[\textbf{PPM phase-space tangent manifold}]\label{theorem:PSM}
Given the PPM of Eq.~(\ref{eq:PPM-simplified}) with initial conditions $(x_0, \dot{x}_0) = (x_{\rm foot}, \dot{x}_{\rm apex})$ and known arm placement $x_{\rm foot}$, the PPM phase-space tangent manifold is defined as
\begin{align}\label{eq:simplifiedPPMPSM}
\sigma(x, \dot{x}, \dot{x}_{\rm apex}, x_{\rm foot}) = \dfrac{\dot{x}^2_{\rm apex}}{-\omega_{\rm PPM}^2} \big(\dot{x}^2 - \dot{x}^2_{\rm apex} + \omega_{\rm PPM}^2(x - x_{\rm foot})^2\big),
\end{align}
\end{proposition}
Compared to the PIPM tangent manifold in Proposition~\ref{prop:PSMTangent}, the PPM tangent manifold has a negative asymptote slope square, i.e., $-\omega_{\rm PPM}^2$. Thus, the tangent manifold with $\sigma > 0$ locates beneath the nominal $\sigma = 0$ tangent manifold. This property is in contrast to that of the PIPM tangent manifold.

\begin{proposition}[\textbf{PPM phase-space cotangent manifold}]\label{prop:PSCoM}
Given the PPM of Eq.~(\ref{eq:PPM-simplified}), the PPM cotangent manifold is
\begin{align}\label{eq:zeta_manifold}
\zeta  = \zeta_0(\dfrac{\dot{x}}{\dot{x}_0})^{-\omega_{\rm PPM}^2} \dfrac{x - x_{\rm foot}}{x_0 - x_{\rm foot}},
\end{align}
\end{proposition}

\begin{proposition}[\textbf{MCM phase-space tangent manifold}]\label{proposition:HP-PSTM}
Given the multi-contact model with a constant acceleration $\omega_{\rm MCM}$ (i.e., the control input), an initial condition $(x_0, \dot{x}_0) = (x_{\rm foot}, \dot{x}_{\rm apex})$, and a known foot placement $x_{\rm foot}$, the MCM phase-space tangent manifold is
\begin{align}\label{eq:simplifiedPSM-2}
\sigma(x, \dot{x}, x_{\rm foot}, \dot{x}_{\rm apex}) = 2 \omega_{\rm MCM} (x - x_{\rm apex}) - (\dot{x}^2 - \dot{x}^2_{\rm apex}),
\end{align}
where $\sigma = 0$ represents the nominal phase-space tangent manifold.
\end{proposition}

\begin{proposition}[\textbf{MCM phase-space cotangent manifold}]\label{proposition:HP-PSCTM}
Given the multi-contact model with a constant acceleration and initial conditions $(x_0, \dot{x}_0) = (x_{\rm foot}, \dot{x}_{\rm apex})$ and known foot placement $x_{\rm foot}$, the phase-space cotangent manifold is
\begin{align}\label{eq:simplifiedPSM-3}
\zeta(x, \dot{x}, x_{\rm foot}, \dot{x}_{\rm apex}) = \omega_{\rm MCM} \cdot {\rm ln} (\dfrac{\dot{x}}{\dot{x}_{\rm apex}}) - (x - x_{\rm foot}),
\end{align}
\end{proposition}
Again, for all the manifolds above, $\sigma = 0$ represents the nominal phase-space tangent manifold. The phase-space manifolds of the hopping model are trivial since its tangent phase-space manifold is a horizontal line.
The stop-launch model and sliding model have similar phase-space manifolds (i.e., parabolic trajectories) as those of the multi-contact model since all of them has a constant sagittal acceleration. Their derivations are omitted for brevity.

\end{appendix_sec}

\bibliographystyle{chicago}
\bibliography{bib}

\end{document}